\documentclass[11pt]{article}

\usepackage[margin=1in]{geometry}
\usepackage[utf8]{inputenc}
\usepackage[T1]{fontenc}
\usepackage{amsmath,amssymb,amsthm,mathtools}
\usepackage{bm}
\usepackage{algorithm}
\usepackage{algpseudocode}
\usepackage{tikz}
\usepackage{graphicx}
\usepackage{booktabs}
\usepackage{nicefrac}
\usepackage{microtype}
\usepackage{xcolor}
\usepackage{longtable}
\usepackage{enumitem}
\usepackage[square,authoryear]{natbib}
\usepackage{url}
\usepackage{hyperref}
\usepackage{cleveref}
\hypersetup{
  colorlinks=true,
  linkcolor=blue,
  citecolor=blue,
  urlcolor=magenta,
  hypertexnames=false
}
\usetikzlibrary{arrows.meta,positioning,shapes.geometric}
\allowdisplaybreaks

\algrenewcommand\algorithmicrequire{\textbf{input:}}
\algrenewcommand\algorithmicensure{\textbf{output:}}
\algrenewcommand\algorithmiccomment[1]{\hfill\textnormal{// #1}}
\algnewcommand\algorithmicforeach{\textbf{for each}}
\algdef{SE}[FOREACH]{ForEach}{EndForEach}[1]{\algorithmicforeach\ #1\ \algorithmicdo}{\algorithmicend\ \algorithmicfor}

\bibpunct[, ]{[}{]}{,}{a}{}{,}

\newtheorem{theorem}{Theorem}
\newtheorem{lemma}{Lemma}[section]

\newtheorem{assumption}{Assumption}
\theoremstyle{definition}
\newtheorem{definition}[theorem]{Definition}
\theoremstyle{remark}
\newtheorem{remark}[theorem]{Remark}

\makeatletter
\renewenvironment{proof}[1]{\par\addvspace{6pt}\noindent\textit{#1}\ }{\par\addvspace{6pt}}
\makeatother
\newcommand{\Halmos}{\unskip\nobreak\hfill\ensuremath{\square}}

\newcommand{\R}{\mathbb{R}}
\renewcommand{\P}{\mathbb{P}}
\newcommand{\E}{\mathbb{E}}

\newcommand{\lwhalf}{\lambda_{W}^{1/2}}
\newcommand{\luhalf}{\lambda_{U}^{1/2}}
\newcommand{\lwhat}{\lambda_{\widehat{W}}}
\newcommand{\luhat}{\lambda_{\widehat{U}}}
\newcommand{\lwhhalf}{\lambda_{\widehat{W}}^{1/2}}
\newcommand{\luhhalf}{\lambda_{\widehat{U}}^{1/2}}

\newcommand{\ublwhat}{\lambda^{\mathrm{up}}_{\widehat{W}}}
\newcommand{\ubluhat}{\lambda^{\mathrm{up}}_{\widehat{U}}}
\newcommand{\ubgmhat}{\widehat{\gamma}^{\mathrm{up}}}
\newcommand{\lbgmhat}{\widehat{\gamma}^{\mathrm{low}}}
\newcommand{\ubkhatvh}{\widehat{K}_{Vh}^{\mathrm{up}}}
\newcommand{\ubmhat}{\widehat{M}^{\mathrm{up}}}
\newcommand{\Vhat}{\widehat{V}}

\newcommand{\Phat}{\widehat{P}}
\newcommand{\bp}{\bm{P}}
\newcommand{\bphat}{\widehat{\bm{P}}}
\newcommand{\rhohat}{\widehat{\rho}}
\newcommand{\one}{\mathbf{1}}
\newcommand{\pistar}{\pi^\star}
\newcommand{\piid}{\Pi_{\mathrm{ID}}}
\newcommand{\pihstar}{\widehat{\pi}^\star}
\newcommand{\pihid}{\widehat{\Pi}_{\mathrm{ID}}}
\newcommand{\pibst}{\bar{\pi}}

\newcommand{\pits}{\Pi_{\mathrm{TS}}}
\newcommand{\pihts}{\widehat{\Pi}_{\mathrm{TS}}}

\newcommand{\hhat}{\widehat{h}}
\newcommand{\pihat}{\widehat{\Pi}}
\newcommand{\Qhat}{\widehat{Q}}

\newcommand{\yhst}{\widehat{y}^\star}

\newcommand{\tmix}{\tau_{\mathrm{unif}}}
\newcommand{\Aideal}{A^{\mathrm{ideal}}}

\newcommand{\bms}{\bm{s}}
\newcommand{\bma}{\bm{a}}
\newcommand{\bmx}{\bm{x}}
\newcommand{\phihat}{\widehat{\Phi}}

\newcommand{\rhorel}{\rho^{\mathrm{rel}}}
\newcommand{\rhohrel}{\widehat{\rho}^{\mathrm{rel}}}
\newcommand{\yneu}{y_{\mathrm{neu}}}
\newcommand{\yneuhat}{\widehat{y}_{\mathrm{neu}}}

\global\long\def\twonorm#1{\left\lVert #1\right\rVert _{2}}

\global\long\def\infnorm#1{\left\Vert #1\right\Vert _{\infty}}
\global\long\def\infinfnorm#1{\left\lVert #1\right\rVert _{\infty \to \infty}}

\global\long\def\onenorm#1{\left\Vert #1\right\Vert _{1}}

\global\long\def\spannorm#1{\left\Vert #1\right\Vert _{\textnormal{span}}}

\newcommand{\States}{\mathcal{S}}
\newcommand{\Actions}{\mathcal{A}}

\newcommand{\zero}{\mathbf{0}}
\newcommand{\muhat}{\widehat{\mu}}
\newcommand{\uhat}{\widehat{U}}
\newcommand{\yhat}{\widehat{y}}
\newcommand{\xihat}{\widehat{\xi}}
\newcommand{\zetahat}{\widehat{\zeta}}
\newcommand{\lambdahat}{\widehat{\lambda}}
\newcommand{\enfe}{\epsilon_N^{\mathrm{fe}}}
\newcommand{\qtilde}{\widetilde{Q}}
\newcommand{\muhst}{\widehat{\mu}^\star}
\newcommand{\gmhat}{\widehat{\Gamma}}



\makeatletter
\newcommand{\vast}{\bBigg@{4}}
\newcommand{\Vast}{\bBigg@{5}}
\makeatother

\title{Lyapunov-Based Sample Complexity Analysis for Weakly-Coupled MDPs}
\author{
Tianhao Wu\textsuperscript{1}, Matthew Zurek\textsuperscript{2}, Weina Wang\textsuperscript{3}, and Qiaomin Xie\textsuperscript{1}\\[0.55em]
\small \textsuperscript{1}Department of Industrial and Systems Engineering, University of Wisconsin-Madison\\
\small \textsuperscript{2}Department of Computer Sciences, University of Wisconsin-Madison\\
\small \textsuperscript{3}Computer Science Department, Carnegie Mellon University\\[0.4em]
\small \texttt{tianhao.wu@wisc.edu}\quad
\texttt{matthew.zurek@wisc.edu}\quad
\texttt{weinaw@cs.cmu.edu}\quad
\texttt{qiaomin.xie@wisc.edu}
}
\date{}

\begin{document}
\maketitle

\begin{abstract}
We study the sample complexity of learning in average-reward weakly-coupled Markov decision processes (WCMDPs) and Restless Bandits (RBs) under a generative model. Naive reduction to a tabular MDP leads to high complexity bounds as the state-action space is exponentially large in the number of arms $N$. By exploiting the weakly coupled structure, we show that near-optimal policies can be learned with sample and computational complexities that are polynomial in $N$. Specifically, we analyze the plug-in approach, which applies an efficient planning algorithm to an empirical model estimated from data. For fully heterogeneous WCMDPs, we establish the first finite-sample PAC guarantee with polynomial complexity and an $O(1/\sqrt{N})$ optimality gap. For homogeneous RBs, we further prove that a smaller optimality gap is achievable under mild structural assumptions. A primary technical contribution is a novel Lyapunov-based analysis framework. Unlike classical approaches that rely on the difficult-to-control bias function, our framework uses an explicitly constructed Lyapunov function along with a drift transfer technique between the true and empirical models. A key step of independent interest in our framework  is a fine-grained perturbation analysis for the underlying linear programming (LP) relaxation, which provides a general tool for analyzing LP-based policies and weakly-coupled systems.
\end{abstract}

\noindent\textbf{Keywords:} Weakly-coupled MDPs, restless bandits, sample complexity, Lyapunov analysis.

\section{Introduction}\label{sec:Intro}
Large-scale decision-making problems with shared resources arise throughout modern engineering systems, including online advertising~\citep{boutilier2016budget, zhou2023rl}, machine maintenance~\citep{glazebrook2005index}, job scheduling~\citep{yu2018deadline}, surveillance~\citep{villar2016indexability}, and healthcare operations~\citep{biswas2021learning}. A common modeling paradigm for these complex systems is a \emph{weakly-coupled Markov decision process} (WCMDP)~\citep{hawkins2003langrangian}, which comprises $N$ arms/subsystems, each of which itself is an MDP. In a heterogeneous WCMDP, each arm evolves according to its own dynamics, yet their actions are coupled through several global constraints. The goal is to find a policy that maximizes the long-term average reward over an infinite time horizon. A prominent special case is the \emph{(homogeneous) restless bandit} (RB) model~\citep{whittle1988restless}, where all arms share the same MDP model with a binary action space (active/passive). In RBs, there is only one constraint that limits the total number of active arms at each time step. 

Learning for WCMDPs and RBs with unknown system dynamics is challenging: when viewed as a single MDP, the state-action space $\boldsymbol{\mathcal{S}} \times \boldsymbol{\mathcal{A}}$ of a WCMDP/RB is the product of the state-action spaces of the arms and thus has cardinality exponentially large in the number of arms $N$. Recent work has essentially resolved the sample complexity of \emph{tabular average-reward} MDPs and established the necessity of a $\Theta(|\boldsymbol{\mathcal{S}} \times \boldsymbol{\mathcal{A}}|)$ sample size \citep{wang_near_2022, wang_optimal_2024, zurek_span-based_2025, zurek2025span}. 
Naively applying these results to WCMDPs and RBs, which ignores the weakly coupling structure, yields a sample complexity exponential in $N$. Even the planning problem---finding the \emph{exact} optimal policy when the transition dynamics are known---is computationally intractable due to large spaces. Nonetheless, recent work has made notable progress on the planning problem and developed near-optimal policies that can be computed efficiently~\citep{zhang2025projection,hong2024achieving,hong2024unichain,yan2024achieving,Gast_2024}. 

This paper considers the learning setting and addresses the following key question:
\begin{quote}
\emph{How can one learn a near-optimal policy in average-reward WCMDPs (or RBs) without incurring exponential dependence on the number of arms $N$?}
\end{quote}
We answer this question for both {fully heterogeneous} WCMDPs  and {homogeneous} RBs with access to a generative model, showing that \emph{polynomial-in-$N$} complexities are possible by exploiting the weakly coupled structure. In particular, we develop a \emph{plug-in} approach, which applies an efficient planning algorithm---specifically the ID policy~\citep{zhang2025projection} for WCMDPs and two-set policy~\citep{hong2024achieving} for RBs---to the empirical model estimated from data. For fully heterogeneous WCMDPs, we establish the first \emph{finite-sample (PAC)} optimality-gap guarantee, where both sample and computational complexities scale polynomially in $N$ for obtaining a near-optimal policy with an $O(1/\sqrt{N})$ optimality gap. 
For RBs, we further show that an optimality gap of $O(1/N^\beta)$ is achievable for any $\beta>0$ under mild structural assumptions.

A main technical contribution of the paper is a novel \emph{Lyapunov-based} analysis framework. Classical sample complexity analysis for the plug-in approach uses the standard simulation lemma, which relies on the bias function (i.e., relative value function) of the MDP~\citep{zurek_plug-approach_2024,li_breaking_2020,agarwal_model-based_2020}. The bias function is typically implicit and difficult to control, especially for the $N$-armed system under the sophisticated planning algorithms we adopt. Our approach replaces the role of the bias function with an explicitly constructed \emph{Lyapunov function}. Specifically, our framework first performs a Lyapunov drift analysis of the adopted planning algorithm in the true system, and then employs a drift transfer technique that allows us to analyze the empirical system by explicitly bounding the norm of the Lyapunov function. This framework applies whenever the planning algorithm admits a Lyapunov drift analysis, thus providing a powerful tool for analyzing the plug-in approach. We believe it is broadly applicable to the analysis of other stochastic systems and learning algorithms. We present the details of this framework in  Section~\ref{sec:Lyapunov-Framework}.

As a key step in implementing our Lyapunov framework, we develop a fine-grained perturbation analysis for the linear program (LP) relaxation associated with the RB. The analysis is applicable whenever the LP is solved with perturbed inputs, including but not limited to sampling-based estimates. Beyond serving as a technical tool in our proofs, this perturbation result reveals an intrinsic robustness property of policies constructed using this LP: under appropriate conditions, such a policy is stable against perturbed input, in the sense that the actions remain unchanged in all but one state, where the action probabilities are adjusted to satisfy the constraints under perturbation. Since many control schemes for WCMDPs and RBs rely on solving variants of this LP relaxation, our perturbation results can be used as a robustness module in analyzing other algorithms, including various index and LP-priority policies that are extensively studied in the WCMDP and RB literature~\citep{weber1990index,verloop2016asymptotically,Gast_2024,hodge2015asymptotic}.
\section{Related Work}
We review related work along the following three dimensions.

\emph{Planning for average-reward RBs/WCMDPs.}
For RBs, the Whittle index policy is a classical approach, but it requires indexability and is not optimal in general
\citep{whittle1988restless,weber1990index}.
To move beyond indexability, \citet{verloop2016asymptotically} introduce a class of LP-priority policies that apply to both indexable and nonindexable RBs, and the Whittle index policy belongs to this class when the model is indexable.
Subsequent work has weakened the conditions required for asymptotic optimality~\citep{hong2023restless,hong2024unichain,yan2024achieving,gast25a} or improved the rates of the optimality gaps~\citep{gast2023exponential,Gast_2024,hong2024achieving,gast25a}. A major milestone in this literature is the establishment of exponentially small optimality gaps of order $O(\exp(-CN))$, going beyond the $O(1/\sqrt N)$ rate commonly obtained owing to the Central Limit Theorem.
In particular, \citet{hong2024achieving} develop a two-set policy that achieves an $O(\exp(-CN))$ optimality gap under easy-to-verify conditions.

Moving beyond RBs to general WCMDPs, \citet{goldsztajn2026asymptotically} consider homogeneous WCMDPs and construct asymptotically optimal policies.
\citet{zhang2025projection} establish an $O(1/\sqrt N)$ optimality gap for fully heterogeneous WCMDPs using a Lyapunov-based analysis.
More recently, \citet{hong2026achieving} obtain an $O(1/N)$ optimality gap for homogeneous WCMDPs under additional non-degeneracy and local-stability conditions.
These results concern the planning setting with known dynamics and do not directly provide finite-sample guarantees when the arm dynamics are unknown.

\emph{Learning for average-reward RBs/WCMDPs.}
One line of work studies online learning for average-reward RBs, focusing on regret guarantees. Early papers consider weak regret against a simple benchmark policy~\citep{dai2011non,liu2011logarithmic,tekin2012online}. Other results obtain sublinear regret against stronger benchmarks but either incur exponential computational complexity~\citep{ortner2012regret}, consider special RB classes such as birth-death arms~\citep{wang2020restless} or compete with a particular index policy~\citep{xiong2022learning}. Thompson-sampling based algorithms have also been studied~\citep{jung2019thompson}. However, these regret results do not imply the PAC/sample-complexity guarantees studied here, since there is no general regret-to-PAC conversion for average-reward MDPs~\citep{tuynman_finding_2024}. Another line of work studies convergence of learning algorithms for RBs~\citep{avrachenkov2022whittle,avrachenkov2024lagrangian,killian2021q} and WCMDPs~\citep{el2023weakly}, but they do not provide finite-sample optimality-gap bounds. Some papers consider Q-learning methods with function approximation to learn Whittle index policies~\citep{xiong2023finite,xiong2024whittle}, but their finite-time convergence
results do not  directly translate to sample-complexity bounds. 
For RBs with finite-horizon or discounted objectives, methods based on Lagrangian relaxation, Q-learning, and deep reinforcement learning (RL) have also been proposed~\citep{killian2021q,robledo2022qwi,killian2022restless,nakhleh2021neurwin,robledo2024deep}, but their guarantees are largely asymptotic or tailored to different performance criteria.

\emph{Sample complexity of tabular average-reward MDPs.}
Recent work has established sharp sample-complexity bounds for tabular
average-reward MDPs. For uniformly mixing or uniformly ergodic MDPs with
mixing-time bound \(t_{\mathrm{mix}}\), \citet{wang_optimal_2024}
obtain the minimax-optimal bound
\(\widetilde O(SA t_{\mathrm{mix}}/\varepsilon^2)\); see also
\citet{jin_towards_2021}. For weakly communicating MDPs, the minimax-optimal
sample complexity is
\(\widetilde O(SAH/\varepsilon^2)\), where
\(H=\mathrm{sp}(h^\star)\) is the span of the optimal bias function
~\citep{wang_near_2022,zurek_span-based_2025}. The latter result also extends
to general multichain MDPs with an additional transient-time parameter. Parameter-free and adaptive methods have also been studied
~\citep{tuynman_finding_2024,neu_dealing_2024,zurek_plug-approach_2024,lee2025near,zurek2025span}.
However, applying these results directly to an \(N\)-armed WCMDP/RB gives
exponential dependence on \(N\), since the joint state-action space is
exponentially large.

\section{Problem Setup}
\paragraph{General notation.}
Let \(\mathbb R\), \(\mathbb N\), and \(\mathbb N_+\) denote the sets of real
numbers, nonnegative integers, and positive integers, respectively. For
\(n\le n'\) with \(n,n'\in\mathbb N_+\), define
\([n]\triangleq\{1,2,\ldots,n\}\) and
\([n:n']\triangleq\{n,n+1,\ldots,n'\}\). We define the grid
\([0,1]_n\triangleq\{i/n\mid i\in\mathbb N,\ 0\le i/n\le 1\}\).
For vectors \(u,v\in\mathbb R^d\), we use the inner product
\(\langle u,v\rangle\triangleq\sum_{i=1}^du_iv_i\).
We use standard Bachmann--Landau notation
\(O(\cdot)\), \(\Omega(\cdot)\), \(\Theta(\cdot)\), \(o(\cdot)\), and
\(\omega(\cdot)\). We also use
\(\widetilde O(\cdot)\), \(\widetilde\Omega(\cdot)\), and
\(\widetilde\Theta(\cdot)\) to hide logarithmic factors. We use lowercase $\pi$ for single-arm policies and uppercase
$\Pi$ for policies of the $N$-arm system.

\subsection{Weakly-Coupled MDPs}\label{sec:WCMDP}
A weakly-coupled MDP consists of \(N\) arms. Each arm \(i \in [N]\triangleq \{1,2,\dots,N\}\) itself is associated with an MDP \(\mathcal{M}_i = (\States, \Actions, P_i, r_i, (c_{k,i})_{k \in [K]})\).
Here \(\States\) and \(\Actions\) are the state and action spaces, respectively, which are finite sets with $|\States|=S, |\Actions|=A$; 
\(P_i(s'_i \mid s_i,a_i)\) is the transition probability from state \(s_i\) to state \(s'_i\) under action \(a_i\);
$r_i$ is the reward function; $c_{k,i}$'s are cost functions ($K$ types of costs in total). When arm~$i$ is in state $s_i$ and takes action $a_i$, it generates reward $r_i(s_i,a_i)$ and cost $c_{k,i}(s_i,a_i)$ of each type $k$.
We assume that the reward functions take values in $[0,r_{\max}]$ for $r_{\max}>0$, and the cost functions take values in $[0,c_{\max}]$ for $c_{\max}>0$. 

For the overall $N$-armed MDP, its state space is $\mathcal{S}^N$ and action space is $\mathcal{A}^N$, where the state is the joint state of all arms, denoted as $\bms=(s_1,\ldots,s_N)$,  the action is the joint action of all arms, denoted as $\bma=(a_1,\ldots,a_N)$, and the reward is the sum of rewards from all arms.
Given actions, the arms make state transitions independently, i.e., $\bp(\bms'\mid\bms,\bma) = \prod\nolimits_{i\in[N]} P_i(s'_i\mid s_i,a_i).$
The arms are coupled by the following cost constraints on their joint action:
\[
\sum\nolimits_{i\in[N]} c_{k,i}(s_i,a_i) \le \alpha_k N, \quad \forall k\in[K],
\]
where \(\alpha_k N\) is referred to as the budget for the cost of type \(k\).
We assume that for every arm, the action~$0$, paired with any state, does not incur any cost of any type, and thus there always exists a feasible action $\bma=(0,0,\dots,0)$.

In this paper, we vary the number of arms $N$ but keep the following constant: the number of cost types $K$, the budgets $\alpha_k$'s, the reward upper bound $r_{\max}$, and the cost upper bound $c_{\max}$.

\paragraph{Performance criterion.}
We consider the infinite-horizon average-reward criterion.
For the $N$-armed system, let \(\bm{S}_t = (S_{i,t})_{i\in[N]}\) denote the state at time $t$, where $S_{i,t}$ is the state of arm~$i$; let \(\bm{A}_t = (A_{i,t})_{i\in[N]}\) denote the action at time $t$, where $A_{i,t}$ is arm~$i$'s action.
For any stationary policy $\Pi$, the long-run average reward, also referred to as the \emph{gain}, from an initial state $\bms_0$ is defined as
\[
\rho^\Pi(\bms_0) \triangleq \lim_{T \to \infty} \frac{1}{T} \sum_{t=0}^{T-1} \frac{1}{N} \sum_{i \in [N]} \mathbb{E}^{\Pi} \left[ r_i(S_{i,t}, A_{i,t}) \right],
\]
where we have used the superscript $^\Pi$ to indicate that the law is induced by the policy $\Pi$.
Note that this limit exists due to the finite state and action spaces.

Classical MDP theory shows that it suffices to seek optimal policies within stationary policies (see, e.g., \citealt[Theorem 9.1.8]{puterman_markov_1994}).
Therefore, the WCMDP aims to solve
\begin{align}
\label{eq:opt_problem}
\max_{\Pi\colon \text{stationary}} \;\rho^\Pi(\bms_0) \quad\; \text{s.t.} \;\sum\nolimits_{i \in [N]} c_{k,i}(S_{i,t}, A_{i,t}) \leq \alpha_{k} N, \;\forall k \in [K], \forall t \geq 0, \text{ under }\Pi.
\end{align}
Denote the optimal value by $\rho^\star(\bms_0)$.

In this paper, besides stationary policies, we also consider policies with an augmented, but still finite, state space.
Such policies maintain some extra state variables in addition to the system state $\bm{S}_t$, and they are stationary with respect to the augmented state. The Markov chain induced by such a policy still has a finite state space and thus its long-run average reward is still well-defined.

\subsection{LP Relaxation}\label{subsec:lp-relax}

As we explain momentarily, the linear program (LP) below is a relaxation of the $N$-armed WCMDP. This LP plays an important role in the policy design in existing work \citep{zhang2025projection}.
\begin{subequations}\label{eq:lprelax}
\begin{align}
\max_{(y_i(s,a))_{i\in[N],\,s\in\mathcal{S},\,a\in\mathcal{A}}}
  &\frac{1}{N} \sum\nolimits_{i\in[N]} \sum\nolimits_{s\in\mathcal{S},\,a\in\mathcal{A}} y_i(s,a) r_i(s,a)\label{eq:relax:obj} \\
\text{subject to} \qquad
&\frac{1}{N} \sum\nolimits_{i\in[N]} \sum\nolimits_{s\in\mathcal{S},\,a\in\mathcal{A}} 
    y_i(s,a) c_{k,i}(s,a) \leq \alpha_k, \quad \forall k\in[K] \label{eq:relax:budget} \\
&\sum\nolimits_{s'\in\mathcal{S},\,a'\in\mathcal{A}} P_i(s \mid s',a') y_i(s',a') 
   = \sum\nolimits_{a\in\mathcal{A}} y_i(s,a), \;\; \forall s\in\mathcal{S}, i\in[N] \label{eq:relax:flow} \\
&\sum\nolimits_{s'\in\mathcal{S},\,a'\in\mathcal{A}} y_i(s',a') = 1, \quad 
   y_i(s,a)\geq 0, \quad \forall s\in\mathcal{S},\,  a\in\mathcal{A},\, i\in[N].
 \label{eq:relax:nonneg}
\end{align}
\end{subequations}
Here the optimization variables $y_i(s, a)$'s can be thought of as the long-run fraction of time arm $i$ spends in state $s$ taking action $a$.
Then the constraint \eqref{eq:relax:budget} is a constraint on the long-run average costs, which relaxes the hard constraint on the costs for each time step.

Let $(y_i^\star(s,a))_{i\in[N],s\in\States,a\in\Actions}$ be an optimal solution to this LP, and let $\rhorel$ denote the optimal value.
\citet{zhang2025projection} show that $\rho^\star(\bms_0)\le \rhorel$ for all $N$-armed WCMDPs and all initial state $\bms_0$.

\subsection{Restless Bandits}
\label{subsec:homogeneous-rb}

An extensively studied subclass of WCMDPs is the restless bandits (RBs). The additional structure in RBs has enabled the development of planning policies with smaller optimality gap. We are interested in whether the similar refinement is possible for learning.

We use the same notation for RBs as that for WCMDPs unless otherwise specified. It should be clear what we refer to from the context.
An RB problem is a WCMDP where each arm has a binary action space
$\mathcal A = \{0,1\}$, where
$a=1$ is referred to as the \emph{pulling/active} action and $a=0$ the \emph{non-pulling/passive} action.
The cost constraint in an RB has a simple form:
\[
\sum\nolimits_{i\in[N]} a_i = \alpha N,
\]
where $0<\alpha < 1$ and $\alpha N$ is an integer.
Compared to the standard WCMDP form, this constraint is written as an equality rather than an inequality.
This discrepancy is largely conventional and does not change the problem in any fundamental way.
We focus on the commonly studied \emph{homogeneous} setting, where each arm's MDP has the same transition kernel $P$ and reward function $r$.

As for WCMDPs, we relax the per-period activation constraint
to a long-run average constraint, following
\citet{whittle1988restless,weber1990index}.
For homogeneous RBs, this gives the following LP
relaxation~\citep{Gast_2024,hong2024achieving}:
\begin{subequations}\label{eq:lprelax-rb}
\begin{align}
\max_{(y(s,a))_{s\in\mathcal{S},\,a\in\mathcal{A}}}\mspace{15mu}&\sum\nolimits_{s\in\mathcal{S},a\in\mathcal{A}}  y(s, a) r(s, a)\label{eq:lp-rb-objective}\\
    \text{subject to}\mspace{21mu}
    &\mspace{0mu}\sum\nolimits_{s\in\mathcal{S}} \quad y(s, 1) = \alpha, \label{eq:expect-budget-constraint}\\ &\sum\nolimits_{s'\in\mathcal{S},\,a'\in\mathcal{A}} P(s \mid s',a') y(s',a') = \sum\nolimits_{a\in\mathcal{A}} y(s,a), \quad \forall s\in\mathcal{S} \label{eq:flow-balance-equation}\\
    &\mspace{0mu}\sum\nolimits_{s'\in\mathcal{S}, a'\in\mathcal{A}} y(s',a') = 1,
    \quad
     y(s,a) \geq 0, \;\; \forall s\in\mathcal{S}, a\in\mathcal{A}.  \label{eq:non-negative-constraint}
\end{align}
\end{subequations}
Let $(y^\star(s,a))_{s\in\States,a\in\Actions}$ be an optimal solution to this LP, and let $\rhorel$ denote the optimal value. It has been shown that $\rho^\star(\bms_0)\le \rhorel$ for all $N$-armed RBs and all initial state $\bms_0$~\citep{hong2024unichain}.

\section{Learning via a Plug-In Approach with Reference Policies}

We adopt a plug-in approach for learning. 
This approach first constructs an empirical transition probability matrix for each arm using data, and then plug the empirical MDP into a planning algorithm. While the exact optimal policy is intractable for WCMDPs/RBs, we make use of existing efficient planning algorithms for computing a near-optimal policy, which we call a \emph{reference policy}.

\subsection{Constructing Empirical MDPs}
We assume access to a generative model, also known as a simulator~\citep{kearns_finite-sample_1998}, for the transition matrix of each arm's MDP, which is unknown. As is standard in the generative setting, we assume the reward and cost functions are known.

In an $N$-armed WCMDP, for each arm~$i$, we collect $n$ i.i.d.\ samples $S_{s,a}^{i,1}, \dots, S_{s,a}^{i,n}$ from $P_i(\cdot\mid s,a)$ for each state-action pair $(s,a)\in\States\times\Actions$.
We construct the empirical transition kernel $\widehat{P}_i$ by $
    \widehat{P}_i(s' \mid s, a) \triangleq\frac{1}{n} \sum\nolimits_{j=1}^{n} \mathbb{I}\{S_{s,a}^{i,j} = s'\}, \forall s' \in \States.$
The empirical $N$-armed transition kernel is then given by the product
    $\bphat(\bms'\mid \bms,\bma)=\prod\nolimits_{i=1}^N \widehat{P}_i(s'_i \mid s_i, a_i).$

For an $N$-armed RB, since we focus on the homogeneous setting where all the arms share the same transition kernel $P$, we can pick any arm and construct the empirical transition kernel $\widehat P$ using $n$ i.i.d.\ samples per state-action pair.
Then we define the empirical $N$-armed transition kernel $\bphat$ in the same way as WCMDP by setting $\widehat P_i=\widehat P$ for all $i\in[N]$.

\subsection{Plug-In Approach for WCMDPs with ID Policy}

For WCMDPs, we consider the ID policy developed in~\citet{zhang2025projection} as our reference policy, denoted as $\piid$.
The ID policy is efficiently computable (with $\text{poly}(N)$ complexity), and (under proper conditions) it achieves an $O(1/\sqrt{N})$ optimality gap, i.e., $\rho^\star(\bms_0)-\rho^{\piid}(\bms_0)=O(1/\sqrt{N})$.

A key component of the ID policy is the collection of \emph{optimal single-armed policies}.
Specifically, recall that $(y_i^\star(s,a))_{i\in[N],s\in\States,a\in\Actions}$ is an optimal solution to the LP \eqref{eq:lprelax}.
Then for each arm~$i$, viewed as an MDP on its own, we can define the following policy:
\begin{align}\label{def-pibst}
    \pi_i^\star(a \mid s) =
\begin{cases}
\frac{y_i^\star(s,a)}{\sum_{a \in \mathcal{A}} y_i^\star(s,a)}, & \text{if } \sum_{a \in \mathcal{A}} y_i^\star(s,a) > 0, \\
\frac{1}{|\mathcal{A}|}, & \text{if } \sum_{a \in \mathcal{A}} y_i^\star(s,a) = 0.
\end{cases}
\end{align}
This policy is called the optimal single-armed policy for arm~$i$.
The policies $\{\pi_i^\star\}_{i\in[N]}$ are designed such that if each arm~$i$ follows $\pi_i^\star$ and evolves completely independently, then the arms together satisfy the cost constraint in a time-average fashion and achieve a total reward of  $\rhorel$, the optimal value of the LP, which is an upper bound on the true optimal value of the $N$-armed WCMDP.

We defer a formal description of the ID policy to Algorithm~\ref{alg:IDpolicy} in Appendix~\ref{sec:preliminary-thm-hete}.
The high-level idea of the ID policy is to let most arms follow their optimal single-armed policies, while conforming to the hard cost constraints.
An insight behind this approach is that if the arms are given a consistent ordering, e.g., based on the arm IDs, then the set of arms that can consistently follow their optimal single-armed policies expands over time and ultimately covers all but $O(\sqrt{N})$ arms on average.

With the ID policy as our reference policy, the plug-in approach proceeds as follows.
We substitute the empirical transition kernels $\widehat P_i$'s into the LP \eqref{eq:lprelax}.
Solving this empirical LP gives an optimal solution $(\yhst_i(s,a))_{i\in[N],s\in\States,a\in\Actions}$ and the resulting optimal single-armed policies $\{\pihstar_i\}_{i\in[N]}$.
We then use $(\yhst_i(s,a))_{i\in[N],s\in\States,a\in\Actions}$ and $\{\pihstar_i\}_{i\in[N]}$ to construct the ID policy, denoted as $\pihid$.

\subsection{Plug-In Approach for RBs with Two-Set Policy}

For homogeneous RBs, we consider the two-set policy from \citet{hong2024achieving} as our reference policy, denoted as $\pits$.
The two-set policy is also efficiently computable, and under proper conditions it achieves an optimality gap $\rho^\star(\bms_0)-\rho^{\pits}(\bms_0)=O(\exp(-cN))$ for a constant $c>0$.
Note that such an \emph{exponentially} small optimality gap has been known to be achievable only for the RB subclass of WCMDPs, in the homogeneous setting under a set of conditions.

The two-set policy is also built on an optimal single-armed policy.
Recall that $(y^\star(s,a))_{s\in\mathcal S,a\in\mathcal A}$ is an optimal solution to the LP \eqref{eq:lprelax-rb}.
We define the corresponding optimal single-armed policy as:
\begin{equation}\label{eq:single-opt-RB}
\pi^\star(a\mid s)=
\begin{cases}
\frac{y^\star(s,a)}{y^\star(s,0)+y^\star(s,1)}, &\text{if } y^\star(s,0)+y^\star(s,1)>0,\\
\frac{1}{2}, & \text{if }y^\star(s,0)+y^\star(s,1)=0.
\end{cases}
\end{equation}
The long-run average reward of \( \pistar \) equals \( \rho^{\mathrm{rel}} \), and the long-run average budget usage under \( \pistar \) equals \( \alpha \). Since the arms are homogeneous, they share the same single-armed policy $\pi^\star$.

A formal description of the two-set policy is given in Algorithm~\ref{alg:two-set-policy} in Appendix~\ref{subsec:two-set-policy}.
Here we briefly describe its structure to facilitate the discussion of our learning approach.
The two-set policy $\pits$ maintains two disjoint subsets of arms, $D_t^{\mathrm{OL}}$ and $D_t^{\pistar}$, at each time $t$.
It applies two different decision rules to these two subsets, and both rules are based on the optimal single-armed policy $\pi^\star$.
The subsets $D_t^{\mathrm{OL}}$ and $D_t^{\pistar}$ are updated using $D_{t-1}^{\mathrm{OL}}$, $D_{t-1}^{\pistar}$, and the current arm states.
As a result, $\pits$ operates on an augmented state that includes both the arm states and the two subsets.

The plug-in approach with the two-set policy as the reference policy proceeds similarly to the WCMDP case.
We first substitute the empirical kernel $\widehat P$ into the LP in~\eqref{eq:lprelax-rb}.
If the empirical LP has multiple optimal solutions, we select one using a fixed deterministic tie-breaking rule.
Let $(\yhst(s,a))_{s\in\States,a\in\Actions}$ denote the selected solution and let $\pihstar$ be the corresponding single-armed policy.

To ensure that the learned policy is well-defined for every realization of $\widehat P$, we proceed as follows.
If the empirical quantities satisfy the structural conditions required to construct the two-set policy in Algorithm~\ref{alg:two-set-policy}, we use the resulting empirical two-set policy.
Otherwise, we use a fixed feasible fallback policy $\Pi^{\mathrm{fb}}$, e.g., one that activates a fixed set of $\alpha N$ arms at every time step.
We denote the resulting learned policy by $\pihts$.
Under the model accuracy condition of Theorem~\ref{thm:LP_perturbation_main_complete}, the fallback policy is not invoked and $\pihts$ coincides with the two-set policy constructed from the empirical LP.

\section{Main Results}
We present our main results, which give upper bounds on the optimality gaps of the learned policies: $\pihid$ for WCMDPs and $\pihts$ for RBs.
These bounds directly imply sample complexity bounds.

\subsection{Main Results for WCMDPs}
We begin by stating our assumptions.
In average-reward MDP literature, it is common for sample complexity results to require a uniform mixing time bound under all policies (e.g., \citealt{jin_efficiently_2020, jin_towards_2021, wang_optimal_2024, li_stochastic_2024}), or alternatively depend on the MDP diameter or bias span of an optimal policy (e.g., \citealt{zurek_plug-approach_2024, zurek_span-based_2025, zurek2025single, tuynman_finding_2024}).
However, in our setting, the weakly coupled structure allows us to avoid imposing uniform mixing (or other) assumptions on the overall $N$-armed MDP, whose state/action spaces grow exponentially with $N$.
Instead, it suffices to make uniform mixing assumptions on the individual arm MDPs.

Specifically, for each arm~$i$, we consider the associated single-armed MDP $(\States, \Actions, P_i, r_i)$, ignoring the cost functions and constraints.
Given a policy $\pi$ for this MDP, let $P^\pi_i$ denote the transition probability matrix of the Markov chain induced by $\pi$, i.e., $P^\pi_i(s,s') \triangleq \sum_{a\in\mathcal{A}} P_i(s'\mid s,a)\pi(a\mid s).$
If this Markov chain has a unique stationary distribution $\mu^\pi_i$, we define its mixing time as
$\tau^\pi_i\triangleq\min\left\{ t\in\mathbb{N}\colon
\max_{s\in\mathcal{S}}\bigl\| (P^\pi_i)^{t}(s,\cdot)-\mu^\pi_i(\cdot)\bigr\|_{\mathrm{TV}} \le 1/4 \right\},$
where $(P^\pi_i)^{t}$ is the $t$-step transition matrix.
We assume that each arm has a finite uniform mixing time bound under all policies,
and further that these bounds across arms have a finite uniform upper bound, as summarized in Assumption~\ref{ass:unichain-unifmixing} below.
\begin{assumption}[Single-arm unichain and uniform mixing]\label{ass:unichain-unifmixing}
We assume that for each arm $i$, its associated single-armed MDP $(\States, \Actions, P_i, r_i)$ is unichain; i.e., all policies induce a Markov chain with a single recurrent class (and possibly some transient states).
Further, we assume that there exists a finite \(\tmix\) such that $\sup_{N\in\mathbb{N}_+,i\in[N]}\sup_{\pi}\tau^\pi_i\le \tmix$.
\end{assumption}

We remark that our learning approach does not use prior knowledge of the mixing time upper bound \(\tmix\) in Assumption~\ref{ass:unichain-unifmixing}.

\begin{theorem}\label{thm-heteWCMDP}
Consider an $N$-armed WCMDP that satisfies Assumption~\ref{ass:unichain-unifmixing}.
Let $\pihid$ be the learned policy using the plug-in approach with the ID policy as the reference policy, and let $n$ be the number of samples per state-action pair per arm.
Then for any initial state $\bms_0$, with probability at least $1-\eta$, 
\[
\rho^{\star}(\bms_0) - \rho^{\pihid}(\bms_0) 
    \le \min\left\{C_1 \sqrt{S + \log (SAN/\eta)}
    \frac{N}{\sqrt{n}},r_{\max}\right\}+\frac{C_2}{\sqrt{N}},
\]
where $C_1$ and $C_2$ are constants independent of $N$, $n$, $\eta$, and $\bms_0$.
\end{theorem}

We provide the proof of Theorem~\ref{thm-heteWCMDP} in Section~\ref{sec:proof_outline_thm_WCMDP}, and give the detailed proofs of the key lemmas in Appendix~\ref{app:full-proof-thm-heteWCMDP}.
The optimality gap bound for $\pihid$ in Theorem~\ref{thm-heteWCMDP} consists of two terms.
The first term is a finite-sample error term that decreases at rate $1/{\sqrt{n}}$ as the sample size $n$ increases.
The second term, $C_2/{\sqrt{N}}$, is inherited from the optimality gap of the ID policy and is not affected by learning or the sample size $n$.
However, it vanishes at rate $1/{\sqrt{N}}$ as the system size $N$ increases.
This scaling is the asymptotic regime of interest for WCMDPs, where the primary interest is in large systems.

The bound in Theorem~\ref{thm-heteWCMDP} directly yields a sample complexity bound for achieving an optimality gap at most $\epsilon$.
Note that we can only guarantee a gap level $\epsilon=\Omega(1/\sqrt{N})$ due to the second term in the bound.
To achieve any $\epsilon=\Omega(1/\sqrt{N})$, the sample complexity is $n=O\left(\frac{(S + \log (SAN/\eta))N^2}{\epsilon^2}\right)$. 
Notably, the total number of samples needed, $NSAn$, grows only polynomially in $N$ rather than exponentially in $N$, although the state and action space sizes of the overall $N$-armed MDP are exponential in $N$.
This contrasts with the sample complexity results for general average-reward MDPs, and demonstrates the benefit of exploiting the weakly coupled structure in learning.

\subsection{Main Results for RBs}
Recall that for RBs, our plug-in learning approach uses the two-set policy as the reference policy.
We first state a set of assumptions slightly adapted from the conditions that guarantee the $O(\exp(-cN))$ optimality gap for the two-set policy \citep{hong2024achieving}.

Recall that $(y^\star(s,a))_{s\in\mathcal S,a\in\mathcal A}$ is an optimal solution to the LP \eqref{eq:lprelax-rb}, based on which we defined the optimal single-armed policy $\pi^\star$ in \eqref{eq:single-opt-RB}. Define 
\begin{equation}\label{eq:def-mustar}
    \mu^\star(s) \triangleq y^\star(s, 0) + y^\star(s, 1), \quad \forall s
    \in \States.
\end{equation}
It is easy to verify that $\mu^\star$ is a stationary distribution for the single-armed MDP under policy $\pi^\star$.
With these quantities, we define an $S\times S$ matrix $\Phi$ as follows
\begin{equation}\label{def:phi}
\Phi \triangleq P^{\pistar}-\one \mu^\star
-\bigl(c^{\pistar}-\alpha\mathbf 1\bigr)\bigl(P_{\tilde{s}1}-P_{\tilde{s}0}\bigr),
\end{equation}
where $P^{\pistar}$ is the transition matrix of the Markov chain induced by $\pi^\star$, $\one$ is the all-ones column vector,
$c^{\pistar}\triangleq\bigl(\pi^\star(1\mid s)\bigr)_{s\in\mathcal S}$ is a column vector, 
and $P_{sa}\triangleq \bigl(P(s'\mid s,a)\bigr)_{s'\in\mathcal{S}}$ are
row vectors. Here $\tilde{s}$ is the neutral state  defined below in Assumption~\ref{ass:RB}.

We state the assumptions below, which are all independent of $N$.
\begin{assumption}\label{ass:RB}
For an RB with single-armed MDP $(\mathcal{S},\mathcal{A},P,r)$ and cost budget $\alpha$, we assume:
\begin{itemize}[leftmargin=1.5em,itemsep=-3pt,topsep=1pt]
    \item \textbf{(Ergodicity)} For the single-armed MDP, the Markov chain induced by the optimal single-armed policy $\pi^\star$ is aperiodic and irreducible and has a mixing time upper bounded by $\tau$.
    \item \textbf{(Non-degeneracy)} There exists a unique state $\tilde s\in\mathcal S$ such that $y^\star(\tilde s,1)>0$ and $y^\star(\tilde s,0)>0$. This state is referred to as the neutral state.
    \item \textbf{(Local stability)}
    The matrix $\Phi$ satisfies $\rho(\Phi)<1$.
\end{itemize}
\end{assumption}
Here $\rho(\cdot)$ denotes the spectral radius, the largest
modulus among the eigenvalues of a matrix.
We remark that the ergodicity assumption is slightly stronger than that in \citet{hong2024achieving}.
The ergodicity assumption requires irreducibility while that paper requires unichain;
the local-stability and non-degeneracy conditions are the same as in \citet{hong2024achieving}.

We make a further assumption on the dual LP of the LP \eqref{eq:lprelax-rb}, where the dual LP can be written as
\begin{subequations}
\label{eq:RB_dual_LP}
\begin{align}
\min_{\zeta,\lambda,(h(s))_{s\in\mathcal{S}}}\quad & \zeta - \alpha \lambda \\
\text{subject to}\quad 
& r(s,a)+\lambda\mathbb{I}\{a\!=\!1\} + \sum\nolimits_{s'\in\mathcal S} \! P(s'\mid s,a)\, h(s')\le\zeta + h(s),\;\; \forall s\in\mathcal S,a\in\mathcal A.
\end{align}
\end{subequations}

\begin{assumption}[Strict action-value gaps]
\label{ass:strict_action_gaps}
There exists an optimal solution $(\zeta^\star, \lambda^\star, h^\star)$ to the dual LP~\eqref{eq:RB_dual_LP} which satisfies that for all $s \in \States, a\in \Actions$ such that $y^\star(s,a)=0$, letting $\bar{a} = 1-a$, we have
\begin{align*}
    r(s, \bar{a}) + \lambda^\star \mathbb{I}(\bar{a}=1) + P_{s \bar{a}}h^\star > r(s, a) + \lambda^\star \mathbb{I}(a=1) + P_{s a}h^\star.
\end{align*}
\end{assumption}
Assumption \ref{ass:strict_action_gaps} is a dual non-degeneracy condition, which in general is a sufficient condition for a unique primal solution, and in our setting is actually an equivalent condition.

We remark that we do not require uniform-mixing-style assumptions.
Instead, we only require ergodicity for the single-armed Markov chain induced by \emph{one policy}, $\pi^\star$, and thus our analysis differs from established techniques based on uniform mixing.
Interestingly, the assumptions that allow for an exponentially small optimality gap, after being slightly strengthened into our Assumptions~\ref{ass:RB} and \ref{ass:strict_action_gaps}, also enable us to prove a \emph{perturbation analysis theorem}, Theorem~\ref{thm:LP_perturbation_main_complete}.
This theorem shows that the empirical LP solution and the derived single-armed policy $\pihstar$ remain close to the true versions and retain key structural properties, including the identity of the neutral state and the local stability property under small transition-kernel perturbations.
This establishes a form of robustness to \emph{model perturbations} induced by the structural assumptions which were previously linked to an exponential optimality gap.
This perturbation result may be of independent interest. The full statement of Theorem~\ref{thm:LP_perturbation_main_complete} is given in Section~\ref{subsec:thm-complete-perturbation}.
Based on Theorem~\ref{thm:LP_perturbation_main_complete}, we obtain the following theorem for RBs. The proof is provided in Section~\ref{sec:proof_outline_thm_RB}, and the detailed proofs of the key lemmas are given in Appendix~\ref{app:proof-thm-homoRB}.
\begin{theorem}\label{thm-homoRB}
Consider an $N$-armed RB that satisfies Assumptions~\ref{ass:RB} and \ref{ass:strict_action_gaps}.
Let $\pihts$ be the learned policy using the plug-in approach with the two-set policy as the reference policy, and let $n$ be the number of samples per state-action pair. There exists an $N_{min}>0$ (defined in \eqref{eq:def-Nmax}) such that for all $N\geq N_{\min}$, and any initial state $\bms_0$, we have
\[
\rho^{\star}(\bms_0) - \rho^{\pihts}(\bms_0) \le \min\left\{C_3\sqrt{S + \log (SA/\eta)}\frac{N}{\sqrt{n}}, \, r_{\max}\right\}+C_4\exp(-C_5N)
\]
with probability at least $1 - \eta$,
where $C_3,C_4$ and $C_5$ are constants independent of $N$, $n$, $\eta$, and $\bms_0$.
\end{theorem}

Similarly to WCMDP, the optimality gap bound for $\pihts$ in Theorem~\ref{thm-homoRB} also consists of two terms.
The first term is the finite-sample error that decreases at a rate $1/{\sqrt{n}}$ with the sample size $n$.
The second term, $C_4\exp(-C_5N)$, is inherited from the optimality gap of the two-set policy.

The bound in Theorem~\ref{thm-homoRB} also directly yields a sample complexity bound for achieving an optimality gap at most $\epsilon$.
However, since now the reference policy achieves a much smaller optimality gap, we can guarantee a gap level $\epsilon=\Omega(\exp(-C_5N))$.
To achieve any $\epsilon=\Omega(\exp(-C_5N))$, the sample complexity is $n=O\left(\frac{(S + \log (SA/\eta))N^2}{\epsilon^2}\right)$. 
Now as long as we target an $\epsilon=\Theta(1/\textrm{poly}(N))$, the total number of samples, $SAn$, grows polynomially in $N$.

\paragraph{Blocking scheme in the fixed-sample regime.} Building on Theorem~\ref{thm-homoRB}, we can improve the subotpimality guarantee in the fixed-$n$ regime by a simple blocking scheme.
Let $B$ satisfy $\bar N\triangleq N/B\in \mathbb Z$ and $\alpha N/B\in\mathbb Z$. We partition the arms into $B$ equal-sized blocks and on each block run the learned two-set policy, using the same empirical LP solution with $\widehat P$ and block budget $\alpha\bar N$.
Denote the resulting policy by $\pihts^{\mathrm{blk}}$. For block $b$, let $\rho_b^{\pihts}(\bms_{0,b})$ denote its average reward per arm.
If $\bar N\ge N_{\min}$, the proof of Theorem~\ref{thm-homoRB} shows that, with probability at least $1-\eta,$
$\rhorel-\rho_b^{\pihts}(\bms_{0,b})
\le
\min\{r_{\max},C_3\ell\bar N/\sqrt n\}
+C_4e^{-C_5\bar N}$,
where $\ell=\sqrt{S+\log(SA/\eta)}$.
Since the blocks have equal size,
$\rho^{\pihts^{\mathrm{blk}}}(\bms_0)
=
B^{-1}\sum_{b=1}^B\rho_b^{\pihts}(\bms_{0,b})$.
Together with $\rho^\star(\bms_0)\le\rhorel$, we have
\begin{align}\label{eq:blocking-corrected-bound}
&\rho^\star(\bms_0)-\rho^{\pihts^{\mathrm{blk}}}(\bms_0)
\le
\min\big\{
r_{\max},
{C_3\ell}\bar N/{\sqrt n}+C_4e^{-C_5\bar N}
\big\}.
\end{align}

The bound in~\eqref{eq:blocking-corrected-bound} exhibits a trade-off in the block size $\bar N$: the finite-sample term $C_3\ell\bar N/\sqrt n$ grows linearly in $\bar N$, whereas the inherited gap term $C_4e^{-C_5\bar N}$ decays exponentially. Consequently, a suitable choice of $\bar N$ achieves a vanishing suboptimality gap, whereas without blocking the bound in~\eqref{eq:blocking-corrected-bound} does not vanish as $N$ grows. To see this, setting aside the integrality constraint and the cap at $r_{\max}$, the right-hand side is minimized at
$\bar N^*
=
C_5^{-1}\log\!\big(C_4C_5\sqrt n/(C_3\ell)\big)$. Consider the regime  $n=\Theta(N^2)$, we have $\bar N^*=\Theta(\log N)$. Thus~\eqref{eq:blocking-corrected-bound} gives
$\rho^\star(\bms_0)-\rho^{\pihts^{\mathrm{blk}}}(\bms_0)
=\widetilde O(N^{-1})$.
By contrast, the unblocked choice $\bar N=N$ leaves a finite-sample term $C_3\ell N/\sqrt n=\Theta(1)$, which does not improve as $N$ grows. Blocking thus sharpens a non-vanishing guarantee to $\widetilde O(N^{-1})$ using blocks of logarithmic length.

\subsection{LP Perturbation Analysis}
\label{subsec:thm-complete-perturbation}
In this subsection, we state the LP perturbation analysis theorem. We first introduce vectorized versions of the primal and dual LPs for the true and perturbed systems (corresponding to $P$ and $\Phat$).

\paragraph{LPs for true system with $P$.} 
We treat $y$ as a $SA$-dimensional row vector and define $\Tilde{e} \in \R^{SA}$ by $\Tilde{e}(s,a) = \mathbb{I}(a = 1)$. We let $e_s \in \R^S$ denote the vector such that $e_s(s') = \mathbb{I}(s'=s)$, and $e_{sa} \in\R^{SA}$ be $e_{sa}(s',a') = \mathbb{I}(s'=s,a'=a).$ Let $J \in \R^{SA \times S}$ be the matrix such that for each $(s,a),$ $e_{sa}^\top J = e_s^\top$. Then we can rewrite the primal and associated dual LPs in a vectorized form as follows:
\noindent
\begin{minipage}[t]{0.48\textwidth}
\setlength{\abovedisplayskip}{0pt}
\setlength{\abovedisplayshortskip}{0pt}
\vspace{0pt}
\makebox[\linewidth][c]{\textbf{Primal LP}}
\vspace{-1em}
\begin{subequations}
\label{eq:primal_LP_true}
\begin{align}
\text{maximize} \quad
    & yr \\
\text{subject to} \quad
    & y \Tilde{e} = \alpha, \quad y P = y J, \label{eq:true_primal_arm_frac_const}\\
    &  y \one = 1,\quad y \in \R^{SA}_{\geq 0}\nonumber.
\end{align}
\end{subequations}
\end{minipage}
\hfill
\begin{minipage}[t]{0.48\textwidth}
\setlength{\abovedisplayskip}{0pt}
\setlength{\abovedisplayshortskip}{0pt}
\vspace{0pt}\makebox[\linewidth][c]{\textbf{Dual LP}}
\vspace{-1em}
\begin{subequations}
\label{eq:dual_LP_true}
\begin{align}
\text{minimize} \quad
    & \zeta - \alpha \lambda \\
\text{subject to} \quad
    & r + \lambda \Tilde{e} + P h
      \leq \zeta \one + J h, \label{eq:dual_LP_true_Bellman_ineq}\\
    & \zeta, \lambda \in \R,\quad h \in \R^{\States}\nonumber.
\end{align}
\end{subequations}
\end{minipage}
\par\vspace{1em}
To assist the reader in interpreting the matrix $J$, we note that $Jh \in \R^{SA}$ and $(Jh)(sa) = h(s)$ for all $(s,a)$. In addition, $yJ = \mu$ where $\mu(s) = y(s,0) + y(s,1)$.
Lastly we present the complementary slackness property of solutions to the above LPs: 
if $y$ is primal feasible and $\zeta, \lambda, h$ are dual feasible, then they are said to satisfy the complementary slackness property if
\begin{align}
    y \circ \left( r + \lambda \Tilde{e} + Ph - \zeta \one - J h \right) = \zero. \label{eq:LP_CS_true}
\end{align}

\paragraph{LPs for perturbed system  with $\Phat$.}
The perturbed primal and dual LPs are
given as follows: 
\noindent
\begin{minipage}[t]{0.48\textwidth}
\setlength{\abovedisplayskip}{0pt}
\setlength{\abovedisplayshortskip}{0pt}
\vspace{0pt}
\makebox[\linewidth][c]{\textbf{Perturbed primal LP}}
\vspace{-1em}
\begin{subequations}
\label{eq:primal_LP_pert}
\begin{align}
\text{maximize} \quad
    & yr \\
\text{subject to} \quad
    & y \Tilde{e} = \alpha,\quad y \Phat = y J, \\
    & y\one = 1,\quad y \in \R^{SA}_{\geq 0}\nonumber.
\end{align}
\end{subequations}
\end{minipage}
\hfill
\begin{minipage}[t]{0.48\textwidth}
\setlength{\abovedisplayskip}{0pt}
\setlength{\abovedisplayshortskip}{0pt}
\vspace{0pt}
\makebox[\linewidth][c]{\textbf{Perturbed dual LP}}
\vspace{-1em}
\begin{subequations}
\label{eq:dual_LP_pert}
\begin{align}
\text{minimize} \quad
    & \zeta - \alpha \lambda \\
\text{subject to} \quad
    & r + \lambda \Tilde{e} + \Phat h
      \leq \zeta \one + J h, \label{eq:pert_dual_ineq} \\
    & \zeta, \lambda \in \R,\quad h \in \R^{\States}\nonumber.
\end{align}
\end{subequations}
\end{minipage}
\par\vspace{1em}
For the perturbed LPs, the complementary slackness condition for primal feasible $y$ and dual feasible $\zeta, \lambda, h$ is 
\begin{align}
    y \circ \left( r + \lambda \Tilde{e} + \Phat h - \zeta \one - J h \right) = \zero. \label{eq:LP_CS_pert}
\end{align}

The theorem involves the span seminorm $\spannorm{\cdot}$, defined as $\spannorm{x} = \max_{i} x_i - \min_i x_i$ for a vector $x \in \R^S$. It also involves the matrix
\(
    H_{\Gamma}=(I-\Phi)^{-1}-\one\mu^\star,
\)
which frequently plays a condition-number-like role. A more detailed explanation of \(H_{\Gamma}\) is given in Appendix~\ref{sec:proof-PA}. Let $\mu_{\min}^\star=\min_{s\in\States}\mu^\star(s)$, where $\mu^\star$ is defined in \eqref{eq:def-mustar}. By the ergodicity condition in Assumption~\ref{ass:RB}, we have \(\mu_{\min}^\star>0\). We define $U\triangleq\sum_{k=0}^{\infty}\Phi^k(\Phi^\top)^k$ and $\lambda_U\triangleq\lambda_{\max}(U)$, where 
$\lambda_{\max}(\cdot)$ denotes the largest eigenvalue
of a symmetric matrix. For symmetric matrices $A$ and $B$, we use $A\preceq B$ to denote that $B-A$ is positive semidefinite. 
We also use the $\ell_\infty$ operator norm for a matrix $A \in \mathbb{R}^{S \times S}$: 
$
    \infinfnorm{A}
    \triangleq \sup_{f \neq 0} \frac{\|Af\|_\infty}{\|f\|_\infty}
    = \max_{s \in \States} \sum_{s' \in \States} |A(s,s')|.
$
We next state the theorem.

\begin{theorem}
\label{thm:LP_perturbation_main_complete}
    Let $y^\star$ be an optimal solution to the true primal LP~\eqref{eq:primal_LP_true}, and suppose that Assumption~\ref{ass:RB} holds for $y^\star$.
Fix an optimal solution $(\zeta^\star,\lambda^\star,h^\star)$ to the true dual LP~\eqref{eq:dual_LP_true} that satisfies Assumption~\ref{ass:strict_action_gaps}.
Suppose $\infinfnorm{\Phat-P}\le\delta$ and
        \begin{multline}\label{eq:def-delta-min}
            \delta \leq \delta_{\min}\triangleq \min \vast\{\frac{1}{6\sqrt S(1+\infinfnorm{H_{\Gamma}})\lambda_U(2\twonorm{\Phi}+1)},  \frac{\min_{sa : y^\star(s,a) > 0} y^\star(s,a)}{6\infinfnorm{H_{\Gamma}} } , \\
             \qquad \frac{\min_{sa:y^\star(s,a)=0} \left(\zeta^\star +  h^\star(s) -r(s, a) - \lambda^\star \mathbb{I}(a=1) - P_{s a}h^\star\right)}{(8 + 36\infinfnorm{H_{\Gamma}})  \spannorm{h^\star}}, \\
             \frac{\mu^\star_{\min}}{144S\tau(5+\log_2S)\infinfnorm{H_{\Gamma}}},
             \frac{1}{8S\tau(5+\log_2S)}  \vast\}.
        \end{multline}
    Then the perturbed primal LP~\eqref{eq:primal_LP_pert} has a unique optimal solution $\yhat^\star$ and the perturbed dual LP~\eqref{eq:dual_LP_pert} has an optimal solution $(\zetahat^\star, \lambdahat^\star, \hhat^\star)$ that is the unique solution satisfying $\muhat^\star \hhat^\star = 0$, where $\muhat^\star(s) \triangleq \yhat^\star(s,0) + \yhat^\star(s,1)$. Furthermore, these solutions have the following properties:
    \begin{enumerate}
        \item $\yhat^\star(s,a)=0$ if and only if $y^\star(s,a) = 0$. In particular, $\muhat^\star(s) > 0$ for all $s \in \States$, and $\tilde{s}$ is the unique state such that both $\yhat^\star(\tilde{s},0) >0$ and $\yhat^\star(\tilde{s},1) >0$.
        \item For all $s,a$ such that $y^\star(s,a)>0$, we have $\yhat^\star(s,a) \geq y^\star(s,a)/2$.
        \item The solutions satisfy the following perturbation bounds:      
\[
\begin{alignedat}{2}
\infnorm{\yhat^\star-y^\star}
&\le
3\delta\infinfnorm{H_{\Gamma}},
&
\spannorm{\hhat^\star-h^\star}
&\le
6\delta\infinfnorm{H_{\Gamma}}\spannorm{h^\star},
\\
\left|\lambdahat^\star-\lambda^\star\right|
&\le
\delta\spannorm{h^\star}
\left(1+6\infinfnorm{H_{\Gamma}}\right),
&
\left|\zetahat^\star-\zeta^\star\right|
&\le
\delta\spannorm{h^\star}
\left(\frac{5}{2}+6\infinfnorm{H_{\Gamma}}\right).
\end{alignedat}
\]
        \item For all $s,a$ such that $y^\star(s,a)=0$, we have
        \begin{align*}
            \zetahat^\star +  \hhat^\star(s) -r(s, a) - \lambdahat^\star \mathbb{I}(a=1) - \Phat_{s a}\hhat^\star 
            &\geq \frac{\zeta^\star +  h^\star(s) -r(s, a) - \lambda^\star \mathbb{I}(a=1) - P_{s a}h^\star}{2}.
        \end{align*}
        \item Defining $\pihstar(a \mid s) = \frac{\yhat^\star(s,a)}{\muhat^\star(s)}$, and letting $\phihat = \Phat^{\pihstar}-\one\muhat^\star -(c^{\pihstar}-\alpha\one) \bigl(\Phat_{\tilde s 1} -\Phat_{\tilde s 0}\bigr)$, we have
        \begin{equation}
        \label{eq:perturbation-inverse-bound}
        \begin{gathered}
            \rho(\phihat)<1,\;\; I\preceq\uhat\triangleq\sum_{k=0}^{\infty}\phihat^k(\phihat^\top)^k\preceq2U,\;\;
        \infinfnorm{(I-\phihat)^{-1}}
        \le2\sqrt S\,\lambda_{\uhat}\le4\sqrt S\,\lambda_U.
        \end{gathered}
        \end{equation}
        \item The Markov chain induced by $\pihstar$ under $\Phat$ is aperiodic and irreducible, with unique stationary distribution $\muhat^\star$ and mixing time at most $(2+\log_2S)\tau$.
    \end{enumerate}
\end{theorem}

\paragraph{Interpretation of Theorem~\ref{thm:LP_perturbation_main_complete}.}
Theorem~\ref{thm:LP_perturbation_main_complete} states how the LP
solution changes when the transition kernel \(P\) is replaced by 
\(\widehat P\). Under Assumptions~\ref{ass:RB} and
\ref{ass:strict_action_gaps}, the perturbed LP preserves the main structural
properties of the true LP, as long as $\widehat P$ is close to $P.$

Parts 1 and 2 show that the support structure of the optimal LP solution is robust to perturbation. In particular, \(\widehat y^\star(s,a)=0 \Leftrightarrow y^\star(s,a)=0. \) Therefore, by the construction in~\eqref{eq:single-opt-RB}, the empirical optimal single-arm policy \(\widehat{\pi}^\star\) can differ from \({\pi}^\star\) only at the neutral state. The neutral state is also preserved: \(\tilde s\) remains the unique state where both actions have positive mass. This matters because the two-set policy uses the neutral state to adjust the number of active arms and meet the budget constraint. Part 2 further shows that the positive entries of \(y^\star\) stay bounded away from zero after perturbation. 

Parts 3 and 4 describe the changes in the primal and dual LP
solutions. Part 3 gives explicit \(O(\delta)\) bounds on the
changes in \(\widehat y^\star\), \(\widehat h^\star\),
\(\widehat\lambda^\star\), and \(\widehat\zeta^\star\).
Together with Assumption~\ref{ass:strict_action_gaps},
these bounds yield Part 4: for every \((s,a)\) with
\(y^\star(s,a)=0\), the perturbed dual slack is at least half
the true dual slack and is therefore strictly positive.
This inequality, together with equality of the
dual constraints on the support of \(y^\star\), establishes
dual feasibility of the candidate dual solutions.

Parts 5 and 6 provide the bounds needed to analyze the OLC
and UOC subroutines (see \ref{subsec:two-set-policy}) of the empirical two-set policy.
For OLC, Part 5 establishes $\rho(\phihat)<1$, which
ensures that the series $\sum_{k=0}^{\infty}\phihat^k(\phihat^\top)^k$ converges.
The bound $\uhat\preceq2U$ implies
$\lambda_{\uhat}\le2\lambda_U$.
This allows us to bound the terms involving $\lambda_{\uhat}$
in the analysis uniformly over all allowed estimates.
The bounds on $\uhat$ and $(I-\phihat)^{-1}$ also
control the coefficients of the empirical Lyapunov function.
For UOC, where each arm follows $\widehat\pi^\star$,
Part 6 bounds the mixing time of the Markov chain induced by $\pihstar$ under $\Phat$. This controls how quickly the state distribution
under UOC approaches $\widehat\mu^\star$.
Together, these results are used to construct the empirical
Lyapunov function and bound the matrices and coefficients
appearing in it using quantities that depend only on the
true model.

\section{A Lyapunov Framework for Sample Complexity Analysis}\label{sec:Lyapunov-Framework}

In this section, we give an overview of our technical framework, which we believe is broadly applicable beyond WCMDPs and RBs.
Our approach can be viewed as generalizing the classical simulation lemma.
A key innovation is that we use \emph{Lyapunov drift bounds}, rather than the Bellman equation, to derive performance difference between the true and empirical systems.

Let $\Pi$ be a reference policy and $\pihat$ be its empirical version computed from the empirical transition kernels.
We use $\,\widehat{\cdot}\,$ to denote the empirical version of a quantity in general, and we ignore the initial state for simplicity.
Our goal is to bound the optimality gap of $\pihat$, which can be decomposed as 
\begin{equation}
\label{eq:decomp}
\begin{aligned}
    \rho^{\star}  - \rho^{\pihat} 
    \le \rho^{\mathrm{rel}} - \widehat{\rho}^{\Pi}  + \widehat{\rho}^{\Pi}  - \rho^{\pihat} 
    \le \bigl(\rho^{\mathrm{rel}} - \widehat{\rho}^{\Pi}\bigr)  + \bigl(\widehat{\rho}^{\mathrm{rel}} - \rho^{\pihat}\bigr).
\end{aligned}
\end{equation}
In the remainder of this section, we demonstrate our framework on deriving an upper bound for $\rho^{\mathrm{rel}} - \widehat{\rho}^{\Pi}$; the term $\widehat{\rho}^{\mathrm{rel}} - \rho^{\pihat} $ can be bounded in a similar manner. 

If we were to use the classical simulation lemma, we would note $\rhorel-\rhohat^{\Pi} =\bigl(\rhorel-\rho^{\Pi}\bigr) + \bigl(\rho^{\Pi} -\rhohat^{\Pi}\bigr)$, and then bound the second term as
\[\rho^{\Pi} -\rhohat^{\Pi}= \widehat{\E}^{\Pi} \big[[(\bp^{\Pi}-\bphat^{\Pi})h^{\Pi}](Z_\infty) \big]
\le \infinfnorm{\bp^\Pi-\bphat^\Pi}\infnorm{h^\Pi}.\] 
Here $h^\Pi$ is the bias function of $\Pi$, and
$Z_\infty$ has the stationary distribution of the system
under $\Pi$ and the empirical model.
The matrices $\bp^\Pi$ and $\bphat^\Pi$ in this bound
are the transition matrices under the true and empirical
models, respectively, with entries
\(
    \bp^\Pi(\bms,\bms')
    =
    \sum_{\bma\in\Actions^N}
    \Pi(\bma\mid\bms)\bp(\bms'\mid\bms,\bma)
\)
and
\(
    \bphat^\Pi(\bms,\bms')
    =
    \sum_{\bma\in\Actions^N}
    \Pi(\bma\mid\bms)\bphat(\bms'\mid\bms,\bma).
\)
However, a challenge to further proceed is that it is difficult to bound $\infnorm{h^{\Pi}(\cdot)}$, as the bias function $h^{\Pi}$ can be complicated.

In our approach, we construct a Lyapunov function $V$ explicitly, which essentially plays the role of the bias function $h^{\Pi}$.
However, because $V$ is a design choice, we can choose it to ensure $\infnorm{V(\cdot)}$ has a desirable bound.
In particular, our framework has two steps, described below.

\paragraph{Step 1: Lyapunov analysis}
This step follows the classical Lyapunov analysis, which has a long history in the study of stochastic systems.
The goal is to construct a Lyapunov function $V$ that satisfies the following two conditions.
\begin{enumerate}[label=(\textbf{C\arabic*}),leftmargin=3.5em, topsep=0em,itemsep=0em]
    \item\label{cond:drift} \textbf{Drift bound:}
    For each state $z$,
    \begin{equation}\label{eq:matrix-drift}
    [(\bp^{\Pi}-I)V](z)\le -C_0 F(z)+G_1(N),
    \end{equation}
    where $F$ is a function of the state, $C_0>0$ is a constant independent of $N$, and $G_1(N)$ is a term that may depend on $N$.
    \item\label{cond:dominance} \textbf{Gap dominance:} For each state $z$,
    \begin{equation}\label{eq:gap-dominance}
        \rho^{\mathrm{rel}}-r^\Pi(z)\le C_0'F(z)+G_2(N),
    \end{equation}
    where \(r^{\Pi}(z)\triangleq\frac{1}{N}\E_{\bma\sim\Pi(\cdot\mid z)}[\sum_{i\in[N]} r(z_i,a_i)]\), $C_0'>0$ is a constant independent of $N$, and $G_2(N)$ is a term that may depend on $N$.
\end{enumerate}

With these conditions, one can bound the optimality gap of the policy $\Pi$.
Specifically, we take expectations on both sides of \eqref{eq:matrix-drift} and \eqref{eq:gap-dominance} with respect to $z\sim Z_{\infty}$, where $Z_{\infty}$ follows the state stationary distribution of $\Pi$.
This makes the left-hand-sides of \eqref{eq:matrix-drift} and \eqref{eq:gap-dominance} equal to $0$ and $\rhorel-\rho^{\Pi}$, respectively.
Combining the two inequalities then yields $\rhorel-\rho^{\Pi} \le \frac{C_0'}{C_0}G_1(N)+G_2(N).$
This approach has been used in ~\citet{zhang2025projection} and \citet{hong2024achieving} to derive optimality gap bounds for the ID policy and the two-set policy, respectively.

\paragraph{Step 2: Drift transfer}
We next transfer the drift bound \eqref{eq:matrix-drift} in the previous step for the true system to the empirical system, which has statistical errors.
In particular, noting that $[(\bphat^{\Pi}-I)V](z) = [(\bp^{\Pi}-I)V](z) + [(\bphat^{\Pi}-\bp^{\Pi})V](z)$, we have:

\begin{enumerate}[label=(\textbf{\^{C\arabic*}}),leftmargin=3.5em]
    \item\label{cond:drift-hat} \textbf{Drift bound:} For each state $z$,
    \begin{align}
    [(\bphat^{\Pi}-I)V](z) \leq - C_0 F(z) + G_1(N) + [(\bphat^{\Pi}-\bp^{\Pi})V](z). \label{eq:overview_drift_empirical}
    \end{align}
\end{enumerate}

\paragraph{Completing the empirical performance bound}
With Steps~1 and 2 in place, we are ready to bound the gap, $\rho^{\mathrm{rel}} - \widehat{\rho}^{\Pi}$, using \ref{cond:drift-hat} and \ref{cond:dominance}.
This follows the same procedure as in deriving an optimality gap bound using \ref{cond:drift} and \ref{cond:dominance}.
We take expectations on both sides of \eqref{eq:overview_drift_empirical} and \eqref{eq:gap-dominance} with respect to $z\sim Z_{\infty}$, where $Z_{\infty}$ follows the state stationary distribution of $\Pi$ in the empirical system.
Again the left-hand-side of \eqref{eq:overview_drift_empirical} becomes $0$.
This gives
\begin{align*}
    \rho^{\mathrm{rel}}-\widehat{\rho}^{\Pi}
    &\le \frac{C_0'}{C_0}G_1(N)+G_2(N)
    +\frac{C_0'}{C_0}\widehat{\mathbb{E}}^{\Pi}
    \left[
    [(\bphat^{\Pi}-\bp^{\Pi})V](Z_\infty)
    \right]\\
    &\le \frac{C_0'}{C_0}G_1(N)+G_2(N)
    +\frac{C_0'}{C_0}
    \infinfnorm{\bp^{\Pi}-\bphat^{\Pi}}
    \infnorm{V}.
\end{align*}
From this bound we can see that it suffices to bound $\infnorm{V(\cdot)}$.
This is typically more tractable than bounding $\infnorm{h^{\Pi}(\cdot)}$ since our construction of $V$ is explicit.

These steps suggest that our framework may be applicable to a much broader set of problems, given that the reference policy admits a Lyapunov analysis with a well-behaved Lyapunov function.

\section{Proof of Theorem~\ref{thm-heteWCMDP}}\label{sec:proof_outline_thm_WCMDP}

We employ the framework introduced in Section~\ref{sec:Lyapunov-Framework} to prove
Theorem~\ref{thm-heteWCMDP}. According to the decomposition~\eqref{eq:decomp},
it suffices to bound
\(\rho^{\mathrm{rel}}-\widehat\rho^{\piid}(\bms_0)\) and
\(\widehat\rho^{\mathrm{rel}}-\rho^{\pihid}(\bms_0)\) separately.

\paragraph{Proof roadmap.} The proof has two main parts, as summarized in Figure~\ref{fig:flow-pf-thmhete}. First, we bound
\(\rho^{\mathrm{rel}}-\widehat\rho^{\piid}(\bms_0)\), where the ID policy $\piid$
is computed from the true model and evaluated under the empirical model.
This step uses the Lyapunov drift condition of ID policy from
Lemma~\ref{lem:lemma5-zhang2025} and the gap dominance bound from
Lemma~\ref{lem:simu-lemma-hete}, which correspond to
Conditions~\ref{cond:drift} and~\ref{cond:dominance} in our framework.
Together with the drift transfer technique, they yield a bound on \(\rho^{\mathrm{rel}}-\widehat\rho^{\piid}(\bms_0)\) (
Lemma~\ref{lem:rhorel-rhohpiid-hete}).

Next we bound
\(\widehat\rho^{\mathrm{rel}}-\rho^{\pihid}(\bms_0)\), where the ID policy $\pihid$
is computed from the empirical model and evaluated under the true model.
The proof follows a similar line of argument as the first part, but we first verify that the empirical
single-arm chains are well behaved (Lemma~\ref{lem:empirical-single-arm-chain}). This result allows us to establish the empirical drift bound (Lemma~\ref{lem:DC-Vhat-hete}) and dominance
bound (Lemma~\ref{lem:simu-lemma-hat-hete}), which give a bound on \(\widehat\rho^{\mathrm{rel}}-\rho^{\pihid}(\bms_0)\) (Lemma~\ref{lem:rhohrel-rhopihid-hete}).

To complete the proof of Theorem \ref{thm-heteWCMDP} using Lemma~\ref{lem:rhorel-rhohpiid-hete} and Lemma~\ref{lem:rhohrel-rhopihid-hete}, we also need to apply model-accuracy lemmas that connect the sample size $n$ to the empirical model accuracy. 

\begin{figure}[htbp]
\centering
\begin{tikzpicture}[
  >=Stealth,
  thick,
  lemma/.style={
    draw,
    rounded corners,
    minimum width=2.35cm,
    minimum height=0.75cm,
    align=center,
    font=\scriptsize,
    inner sep=2pt
  },
  theorem/.style={
    draw,
    ellipse,
    minimum width=2.70cm,
    minimum height=0.80cm,
    align=center,
    font=\scriptsize,
    inner sep=2pt
  }
]

\node[lemma] (l7)  at (0,0)        {Lemma~\ref{lem:lemma5-zhang2025}\\Drift for $\piid$};
\node[lemma] (l8)  at (3.2,0)      {Lemma~\ref{lem:simu-lemma-hete}\\Gap dominance};
\node[lemma] (l1)  at (1.6,-1.20)  {Lemma~\ref{lem:rhorel-rhohpiid-hete}};

\node[lemma] (l11) at (6.9,0)      {Lemma~\ref{lem:DC-Vhat-hete}\\Drift for $\pihid$};
\node[lemma] (l12) at (10.1,0)     {Lemma~\ref{lem:simu-lemma-hat-hete}\\Gap dominance};
\node[lemma] (l2)  at (8.5,-1.20)  {Lemma~\ref{lem:rhohrel-rhopihid-hete}};

\node[theorem] (thm1) at (5.05,-2.15) {Theorem~\ref{thm-heteWCMDP}};

\draw[->] (l7) -- (l1);
\draw[->] (l8) -- (l1);
\draw[->] (l11) -- (l2);
\draw[->] (l12) -- (l2);
\draw[->] (l1) -- (thm1);
\draw[->] (l2) -- (thm1);

\end{tikzpicture}
\caption{Flowchart for the proof of Theorem~\ref{thm-heteWCMDP}. We omit auxiliary model-accuracy lemmas.}
\label{fig:flow-pf-thmhete}
\end{figure}

\paragraph{Model accuracy.}  We begin with a standard concentration bound for the empirical estimate of each single-armed transition kernel. Define $d_S(\widehat P_i, P_i)\triangleq\max_{s\in\States,a\in\Actions}
    \onenorm{P_i(\cdot\mid s,a)-\widehat P_i(\cdot\mid s,a)}$, $d_N(\bphat,\bp)\triangleq\max_{\bms\in\States^N,\bma\in\Actions^N}
    \onenorm{\bp(\cdot\mid\bms,\bma)-\bphat(\cdot\mid\bms,\bma)}$.
The following lemma gives a uniform $\ell_1$ bound over all arms $i\in[N]$ and all state-action pairs $(s,a)\in\States\times\Actions$.
\begin{lemma}\label{lem:samplesize-singlearm_accu}
    Fix $\eta \in (0,1)$. With probability at least $1 - \eta$, we have
    \begin{align*}
        \max_{i\in[N]}d_S(\widehat P_i, P_i) \leq \sqrt{\frac{2S \log(2) + 2\log (SAN/\eta)}{n}}.
    \end{align*}
\end{lemma}

Next we lift the single-armed model accuracy bound to the \(N\)-armed
system. The error accumulates across arms, which leads to an additional
factor of \(N\). The same bound also holds for the transition matrices
induced by any fixed policy.

\begin{lemma}
\label{lem:lift-single-to-Narm}
Suppose that we have the following uniform single-armed model accuracy bound: 
\(
    \max_{i\in[N]}d_S(\widehat P_i, P_i)
    \le
    \delta .
\)
Then the $N$-armed model accuracy satisfies
\(
    d_N(\bphat,\bp)
    \le
    N\delta .
\)
Moreover, for any policy \(\Pi\), 
\(
    \infinfnorm{\bp^\Pi-\bphat^\Pi}
    \le
    N\delta .
\)
\end{lemma}
The proofs of Lemmas~\ref{lem:samplesize-singlearm_accu} and \ref{lem:lift-single-to-Narm} are deferred to Appendix~\ref{subsec:model-accuracy}. 

We next state Lemma~\ref{lem:rhorel-rhohpiid-hete} and Lemma~\ref{lem:rhohrel-rhopihid-hete}. Throughout this subsection, we consider an $N$-armed WCMDP satisfying Assumption~\ref{ass:unichain-unifmixing}.

\begin{lemma}\label{lem:rhorel-rhohpiid-hete}
Suppose 
$
\max_{i\in[N]}d_S(\widehat P_i, P_i)
    \le
    \delta.
$
Then for the ID policy $\piid$, 
    \[\rhorel - \rhohat^{\piid}(\bms_0)\le \frac{3(2r_{\max}+L_h)}{\rho_V}N\delta+\left[\frac{(2r_{\max}+L_h)K_V}{L_h\rho_V}+2 r_{\max}K_{\mathrm{conf}}\right]\frac{1}{\sqrt{N}},\]
where $L_h$ is a constant defined in~\eqref{eq:def-Lh-Ctau-tau} and $\rho_V, K_{\mathrm{conf}}$ are constants defined in~\eqref{eq:cons-hetero-drift}.   
\end{lemma}
\begin{lemma}\label{lem:rhohrel-rhopihid-hete}
Suppose
$
\max_{i\in[N]}d_S(\widehat P_i, P_i)
    \le
    \delta\le\frac{1}{4S\tmix(5+\log_2S)}.
$
Then for the empirical ID policy $\pihid$,
    \[\rhohrel - \rho^{\pihid}(\bms_0)\le \frac{3(2r_{\max}+L_h)}{\rho_V}N\delta+\left[\frac{(2r_{\max}+L_h)K_V}{L_h\rho_V}+2 r_{\max}K_{\mathrm{conf}}\right]\frac{1}{\sqrt{N}},\]
where $\tmix$ is defined in Assumption~\ref{ass:unichain-unifmixing}, $L_h$ is defined in~\eqref{eq:def-Lh-Ctau-tau} and $\rho_V, K_{\mathrm{conf}}$ are defined in~\eqref{eq:cons-hetero-drift}.
\end{lemma}

We provide the proof of Lemma~\ref{lem:rhorel-rhohpiid-hete} in Section \ref{sec:proof-lem-rhorel-rhohpiid-hete} to illustrate how our Lyapunov framework is applied; the proof of Lemma~\ref{lem:rhohrel-rhopihid-hete} is deferred to the Appendix~\ref{subsec:pr-C2-hete}.

\paragraph{Proof of Theorem~\ref{thm-heteWCMDP}}\label{subsec:final-pr-thm1}
To combine Lemma~\ref{lem:rhorel-rhohpiid-hete} and Lemma~\ref{lem:rhohrel-rhopihid-hete} to obtain Theorem~\ref{thm-heteWCMDP}, we only need to connect the single-armed model accuracy $\delta$ to the number of samples $n$. 
First, we notice that when
$
n\le n_0\triangleq\left(2S\log2+2\log(SAN/\eta)\right)16\tmix^2S^2(5+\log_2S)^2,
$
we can choose $C_1=8r_{\max}\tmix S(5+\log_2S)$ so that the trivial bound $\rho^\star(\bms_0)-\rho^{\pihid}(\bms_0)\le r_{\max}$ holds.
Therefore, we only need to consider the case where
$n\ge n_0.$
In this case, according to Lemma \ref{lem:samplesize-singlearm_accu}, we have the single-armed model accuracy:
\[
\max_{i\in[N]}d_S(\widehat P_i, P_i) \le \delta\triangleq \sqrt{\frac{2S \log(2) + 2\log (SAN/\eta)}{n}}\le\frac{1}{4\tmix S(5+\log_2S)}.
\]
with probability at least $1-\eta$. In this case, the conditions of Lemma~\ref{lem:rhorel-rhohpiid-hete} and Lemma~\ref{lem:rhohrel-rhopihid-hete} are satisfied. Combining the results of these two lemmas, together the fact that the reward from each arm is upper bounded by $r_{\max}$, we complete the proof of Theorem~\ref{thm-heteWCMDP}.

\subsection{Proof of Lemma~\ref{lem:rhorel-rhohpiid-hete}}\label{sec:proof-lem-rhorel-rhohpiid-hete}

\proof{Proof of Lemma~\ref{lem:rhorel-rhohpiid-hete}.}
We leverage the Lyapunov function \(V\) constructed in
\citet{zhang2025projection} for analyzing the optimality gap of the
ID policy in WCMDPs. In particular, \(V\), defined in~\eqref{def:lya-true},
is a function of the one-hot representation
\(\bm{X}_t=(X_{i,t})_{i\in[N]}\in\{0,1\}^{N\times S}\) of
\(\bm{X}_t\), where
\(X_{i,t}(s)=\mathbb{I}\{S_{i,t}=s\}\).
We use \(\bmx\) to denote a generic realization of
\(\bm{X}_t\).

\paragraph{Step 1: Lyapunov analysis.} 
We make slight modifications to the existing results to show that the Lyapunov function $V$ satisfies \ref{cond:drift} and \ref{cond:dominance} in the framework, stated formally as Lemma~\ref{lem:lemma5-zhang2025} and Lemma~\ref{lem:simu-lemma-hete}, respectively. 

\begin{lemma}[{\citealt[Lemma 5]{zhang2025projection}}]\label{lem:lemma5-zhang2025}
    Consider an $N$-armed WCMDP that satisfies Assumption~\ref{ass:unichain-unifmixing}. The Lyapunov function $V$ defined in (\ref{def:lya-true}) satisfies
\begin{equation}\label{drift_cond_heterogenous}
    \mathbb{E}^{\piid}[V(\bm{X}_{t+1})-V(\bm{X}_t)|\bm{X}_t=\bmx]\le -\rho_VV(\bmx)+K_V/\sqrt{N},
\end{equation}
where 
\begin{equation}\label{eq:cons-hetero-drift}
    \begin{gathered}
        \rho_V = \frac{1-\gamma}{1+\frac{L_h}{\eta_c}},\;
        \eta_c =\frac{\alpha_{\min}}{4d_c},\; d_c=\max\left\{1,\left\lceil
                    \frac{(4c_{\max}-\alpha_{\min})K}{\alpha_{\min}}
                \right\rceil\right\},\;
        \alpha_{\min}=\min_{k\in[K]}\alpha_k>0,\\
        K_V = 2C_h + C_h K_{\mathrm{conf}} + 2L_h K_{\mathrm{mono}}
        + \rho_V L_h K_{\mathrm{cov}}, \quad
        C_h=2C_\tau(Kc_{\max}+r_{\max}),\\
        K_{\mathrm{conf}}=\frac{2Kc_{\max}+M_c}{\eta_c}, \;
        K_{\mathrm{mono}}=\frac{(2+K_{\mathrm{conf}})C_h+M_c}{\eta_c},\;
        K_{\mathrm{cov}}=\frac{\eta_c+M_c+L_h}{\eta_c}, \;
        M_c = \frac{\alpha_{\min}}{2}.
    \end{gathered}
\end{equation}
Here $\gamma$ is defined in~\eqref{def:gamma-tau-g} and $L_h, C_\tau$ are defined in~\eqref{eq:def-Lh-Ctau-tau}.
\end{lemma}
All constants above depend only on the true WCMDP parameters .

\begin{lemma}\label{lem:simu-lemma-hete}
    For each system state $\bms$, denote $A_i^{\piid}$ as the action applied to arm $i$ under policy $\piid$. Define
    $
    r^{\piid}(\bms)=\frac{1}{N}\sum_{i=1}^N\E[r_i(s_i,A_i^{\piid})].
    $
    Then
    \(
    \rhorel-r^{\piid}(\bms)\le \frac{2r_{\max}+L_h}{L_h}V(\bmx)+\frac{2r_{\max}K_{\mathrm{conf}}}{\sqrt{N}}.
    \)
\end{lemma}
The proof of Lemma~\ref{lem:simu-lemma-hete} follows the same argument as \citet[Lemma~4]{zhang2025projection}; we omit it here.

Therefore, our Lyapunov function $V$ satisfies \ref{cond:drift} and \ref{cond:dominance} with $F(\bmx)=V(\bmx),$ $G_1(N)=K_V/\sqrt{N}$ and $G_2(N)={2r_{\max}K_{\mathrm{conf}}}/{\sqrt{N}}$.

\paragraph{Step 2: Drift transfer.} 
The following lemma formalizes the drift-transfer step. In the lemma, \(Q\) denotes the transition matrix under which the
Lyapunov drift inequality is available, while \(P\) denotes the transition
matrix under which the long-run performance is evaluated. The proof is a
finite-horizon telescoping argument and 
is given in
Appendix~\ref{sec:common-technical-lemmas}.
\begin{lemma}
\label{lem:generic-drift-transfer}
Let \(P\) and \(Q\) be two transition matrices on the same finite state
space \(\mathcal Z\). Let \(\bar r:\mathcal Z\to\mathbb R\) be a bounded reward function. For any initial state \(z_0\), define
long-run average reward under \(P\) by
\(
    \rho_P(z_0)
    \triangleq
    \lim_{T\to\infty}
    \frac1T
    \sum_{t=0}^{T-1}
    \mathbb E_P[\bar r(Z_t)\mid Z_0=z_0].
\)
Here $\mathbb E_P$ denote the expectation with respect to the randomness under $P$. Suppose there exist functions
\(V:\mathcal Z\to\mathbb R\) and \(F:\mathcal Z\to\mathbb R\),
constants \(C_0>0\) and \(C_0'\ge0\), two nonnegative terms
\(G_1(N),G_2(N)\), and a scalar benchmark \(R\), such that, 
\[
    [(Q-I)V](z)
    \le
    -C_0F(z)+G_1(N),
    \quad
    R-\bar r(z)
    \le
    C_0'F(z)+G_2(N),\quad \forall z\in\mathcal Z.
\]
Then, for every initial state \(z_0\),
\[
    R-\rho_P(z_0)
    \le
    \frac{C_0'}{C_0}G_1(N)
    +
    G_2(N)
    +
    \frac{C_0'}{C_0}
    \sup_{z\in\mathcal Z}
    \left|[(P-Q)V](z)\right|.
\]
\end{lemma}
We note that since \(\mathcal Z\) is finite, the Ces\`aro average defining
\(\rho_P(z_0)\) exists for every initial state \(z_0\), although
\(P^t(z_0,\cdot)\) may not converge. 
We now apply Lemma~\ref{lem:generic-drift-transfer} with
\(
    Q=\bp^{\piid},
    P=\bphat^{\piid},
    R=\rhorel,
    F=V,
    \bar r=r^{\piid}.
\)
Lemmas~\ref{lem:lemma5-zhang2025} and~\ref{lem:simu-lemma-hete} give the
required drift and dominance inequalities with 
\(
    C_0=\rho_V,
    C_0'=\frac{2r_{\max}+L_h}{L_h},
    G_1(N)=\frac{K_V}{\sqrt N},
    G_2(N)=\frac{2r_{\max}K_{\mathrm{conf}}}{\sqrt N}.
\)
Hence,
\begin{equation}\label{eq:rhorel-rhohat-hete-intermedia}
    \rhorel-\widehat\rho^{\piid}(\bms_0)
    \le
    \frac{2r_{\max}+L_h}{L_h\rho_V}
    \sup_{\bmx}
    \left|
    (\bphat^{\piid}-\bp^{\piid})V(\bmx)
    \right|  +
    \left[
    \frac{(2r_{\max}+L_h)K_V}{L_h\rho_V}
    +
    2r_{\max}K_{\mathrm{conf}}
    \right]\frac1{\sqrt N}. 
\end{equation}
By Lemma~\ref{lem:lift-single-to-Narm},
\(
    \infinfnorm{\bphat^{\piid}-\bp^{\piid}}
    \le
    N\delta .
\)
Moreover, \(\|V\|_\infty\le L_h\) according to \eqref{eq:bound-V-hete}. Therefore,
\(
    \sup_{\bmx}
    \left|
    (\bphat^{\piid}-\bp^{\piid})V(\bmx)
    \right|
    \le
    \infinfnorm{\bphat^{\piid}-\bp^{\piid}}\|V\|_\infty 
    \le
    L_hN\delta.
\)
Substituting the above bound to the right-hand-side of Eq.~\eqref{eq:rhorel-rhohat-hete-intermedia} gives the desired result of Lemma \ref{lem:rhorel-rhohpiid-hete}. \Halmos
\endproof
\section{Proof of Theorem~\ref{thm-homoRB}}
\label{sec:proof_outline_thm_RB}

We also apply the framework presented in Section~\ref{sec:Lyapunov-Framework} to prove
Theorem~\ref{thm-homoRB}. The proof follows the same decomposition
as~\eqref{eq:decomp} for WCMDP. It is sufficient to bound
\(\rho^{\mathrm{rel}}-\widehat\rho^{\pits}(\bms_0)\) and
\(\widehat\rho^{\mathrm{rel}}-\rho^{\pihts}(\bms_0)\).

\paragraph{Proof roadmap.} Figure~\ref{fig:flow-pf-thmhomo} summarizes the main steps. First, we bound
\(\rho^{\mathrm{rel}}-\widehat\rho^{\pits}(\bms_0)\). Here the two-set
policy $\pits$ is computed from the true model but evaluated
in the empirical model. We use the Lyapunov drift bound for the two-set
policy (Lemma~\ref{lem:DC-RB}) and the gap-dominance bound
(Lemma~\ref{lem:C2-RB}), which correspond to
Conditions~\ref{cond:drift} and~\ref{cond:dominance} in our Lyapunov
framework. Together with the drift-transfer argument, they imply a bound on \(\rho^{\mathrm{rel}}-\widehat\rho^{\pits}(\bms_0)\)
(Lemma~\ref{lem:rhorel-rhohpi-homo}).

Next we bound
\(\widehat\rho^{\mathrm{rel}}-\rho^{\pihts}(\bms_0)\). Here the two-set
policy $\pihts$ is computed from the empirical model and evaluated under the true
model. We employ the same Lyapunov argument, but we also need to use
Theorem~\ref{thm:LP_perturbation_main_complete} to justify the empirical Lyapunov analysis. By establishing the empirical drift condition
(Lemma~\ref{lem:DC-Vhat-homo}) and the empirical gap dominance bound
(Lemma~\ref{lem:C2-hat-RB}), we obtain the bound on \(\widehat\rho^{\mathrm{rel}}-\rho^{\pihts}(\bms_0)\) (Lemma~\ref{lem:rhohrel-rhopihat-homo}).

Similar to the proof of Theorem~\ref{thm-heteWCMDP}, we also need to combine model-accuracy lemma with Lemmas~\ref{lem:DC-Vhat-homo} and \ref{lem:C2-hat-RB} to complete the proof of Theorem~\ref{thm-homoRB}. 

\begin{figure}[htbp]
\centering
\begin{tikzpicture}[
  >=Stealth,
  thick,
  lemma/.style={
    draw,
    rounded corners,
    minimum width=2.35cm,
    minimum height=0.68cm,
    align=center,
    font=\scriptsize,
    inner sep=2pt
  },
  theorem/.style={
    draw,
    ellipse,
    minimum width=2.70cm,
    minimum height=0.72cm,
    align=center,
    font=\scriptsize,
    inner sep=2pt
  },
  perturbation/.style={
    draw,
    ellipse,
    minimum width=2.15cm,
    minimum height=0.82cm,
    align=center,
    font=\scriptsize,
    inner sep=2pt
  },
  auxarrow/.style={
    ->,
    dashed,
    thick
  }
]

\node[lemma] (dc) at (0.0,0)
{Lemma~\ref{lem:DC-RB}\\
Drift for $\pits$};

\node[lemma] (c2) at (3.2,0)
{Lemma~\ref{lem:C2-RB}\\
Gap dominance};

\node[lemma] (dchat) at (6.8,0)
{Lemma~\ref{lem:DC-Vhat-homo}\\
Drift for $\pihts$};

\node[perturbation] (pert) at (9.7,0)
{Theorem~\ref{thm:LP_perturbation_main_complete}:\\
LP perturbation};

\node[lemma] (c2hat) at (12.6,0)
{Lemma~\ref{lem:C2-hat-RB}\\
Gap dominance};

\node[lemma] (l30) at (1.6,-1.20)
{Lemma~\ref{lem:rhorel-rhohpi-homo}};

\node[lemma] (l33) at (9.7,-1.20)
{Lemma~\ref{lem:rhohrel-rhopihat-homo}};

\node[theorem] (thm) at (5.65,-2.10)
{Theorem~\ref{thm-homoRB}};

\draw[->]
    ([xshift=5mm]dc.south)
    --
    ([xshift=-5mm]l30.north);

\draw[->]
    ([xshift=-5mm]c2.south)
    --
    ([xshift=5mm]l30.north);

\draw[->]
    (pert.west)
    --
    (dchat.east);

\draw[->]
    (pert.east)
    --
    (c2hat.west);

\draw[auxarrow]
    (pert.south)
    --
    (l33.north);

\draw[->]
    ([xshift=5mm]dchat.south)
    --
    ([xshift=-5mm]l33.north);

\draw[->]
    ([xshift=-5mm]c2hat.south)
    --
    ([xshift=5mm]l33.north);

\draw[->]
    (l30.south east)
    --
    (thm.north west);

\draw[->]
    (l33.south west)
    --
    (thm.north east);

\end{tikzpicture}
\caption{Flowchart for the proof of
Theorem~\ref{thm-homoRB}. Auxiliary model-accuracy lemmas are omitted.}
\label{fig:flow-pf-thmhomo}
\end{figure}
\paragraph{Model accuracy.} The model-accuracy lemma for the homogeneous RB setting is a special
case of the corresponding lemmas for the heterogeneous WCMDP setting. We state it below.
\begin{lemma}\label{lem:samplesize-singlearm_accu_rb}
Fix $\eta\in(0,1)$. With probability at least $1-\eta$,
\(
    d_S(\widehat P, P)
    \le
    \sqrt{\frac{2S\log(2)+2\log(SA/\eta)}{n}}.
\)
Moreover, if $d_S(\widehat P, P)\le\delta$, then 
\(
    d_N(\bphat,\bp)
    \le N\delta.
\)
And for any policy \(\Pi\), $\infinfnorm{\bp^\Pi-\bphat^\Pi}\le N\delta.$

\end{lemma}
We now state Lemma~\ref{lem:rhorel-rhohpi-homo} and Lemma~\ref{lem:rhohrel-rhopihat-homo}. Throughout this subsection, we consider an $N$-armed RB problem with initial state $\bms_0$ satisfying Assumptions~\ref{ass:RB} and~\ref{ass:strict_action_gaps}.
\begin{lemma}\label{lem:rhorel-rhohpi-homo}
Suppose that
\(
N\ge \max\left\{\frac{4}{\min\{y^\star(\tilde{s},0),y^\star(\tilde{s},1)\}},\left(\frac{4M}{\gamma}\right)^2\right\}
\)
and 
\(
d_S(\widehat P, P)\le \delta,
\)
where $\gamma$ is defined in~\eqref{eq:def-beta-Q-b-bareta-gamma} and $M$ is defined in~\eqref{def:g-C-M}. Then
\[
\rho^{\mathrm{rel}}-\widehat\rho^{\pits}(\bms_0)
\le C_6N\delta+C_6^{\mathrm{exp}}\exp(-C_6^{\mathrm{rate}}N),
\]
where $C_6,C_6^{\mathrm{exp}},C_6^{\mathrm{rate}}$ are defined in~\eqref{eq:def-C6} and depend only on the true single-armed RB parameters.
\end{lemma}

\begin{lemma}\label{lem:rhohrel-rhopihat-homo}
Suppose
\(
N\ge \max\left\{\frac{8}{\min\{y^\star(\tilde{s},0),y^\star(\tilde{s},1)\}},\left(\frac{4\widehat M^{\mathrm{up}}}{\widehat\gamma^{\mathrm{low}}}\right)^2\right\}
\)
and 
\(
d_S(\widehat P, P)
\le \delta\le \delta_{\min},
\)
where $\widehat M^{\mathrm{up}}$ is defined in~\eqref{def:ubmhat}, $\widehat\gamma^{\mathrm{low}}$ is defined in~\eqref{eq:def-lbg-upg-ubKvh}, and $\delta_{\min}$ is defined in~\eqref{eq:def-delta-min}. Then
\[
\widehat\rho^{\mathrm{rel}}-\rho^{\pihts}(\bms_0)
\le C_7N\delta+C_7^{\mathrm{exp}}\exp(-C_7^{\mathrm{rate}} N),
\]
where $C_7,C_7^{\mathrm{exp}},C_7^{\mathrm{rate}}$ are defined in~\eqref{eq:def-C7} and depend only on the true single-armed RB parameters.
\end{lemma}

The proof of Lemma~\ref{lem:rhorel-rhohpi-homo} follows a similar line of argument as that of Lemma~\ref{lem:rhorel-rhohpiid-hete} for WCMDP and is provided in Appendix~\ref{subsec:pr-C1-homo}. In Section~\ref{sec:proof-lem-rhohrel-rhopihat-homo}, we provide a proof outline of Lemma~\ref{lem:rhohrel-rhopihat-homo},
which requires using the perturbation analysis (Theorem~\ref{thm:LP_perturbation_main_complete}) to justify the empirical Lyapunov analysis. The complete proof is deferred to
Appendix~\ref{subsec:pr-C2-homo}.

\paragraph{Proof of Theorem~\ref{thm-homoRB}.}
Combining Lemma~\ref{lem:rhorel-rhohpi-homo} and Lemma~\ref{lem:rhohrel-rhopihat-homo} yields Theorem~\ref{thm-homoRB} once we relate the single-armed model accuracy parameter $\delta$ to the sample size $n$.
First, we notice that when
$n\le n_0'\triangleq\left(2S\log2+2\log(SA/\eta)\right)/\delta_{\min}^2,
$
where $\delta_{\min}$ is defined in~\eqref{eq:def-delta-min}, we can choose $C_3=\frac{2r_{\max}}{\delta_{\min}}$, then the trivial bound $\rho^\star(\bms_0)-\rho^{\pihts}(\bms_0)\le r_{\max}$ is covered.
Therefore, we only need to consider the case when $n\ge n_0'$.
In this case, according to Lemma \ref{lem:samplesize-singlearm_accu_rb}, we have the single-armed model accuracy
\(
    d_S(\widehat P, P) \le \delta\triangleq \sqrt{\frac{2S \log(2) + 2\log (SA/\eta)}{n}}\le\delta_{\min},
\)
with probability at least $1-\eta$. We consider the case
\begin{equation}\label{eq:def-Nmax}
    N\ge N_{\min}\triangleq\left\lceil\max\left\{\frac{8}{\min\{y^\star(\tilde{s},0),y^\star(\tilde{s},1)\}},\left(\frac{4M}{\gamma}\right)^2,\left(\frac{4\ubmhat}{\lbgmhat}\right)^2\right\}\right\rceil,
\end{equation}
where $M$ is defined in \eqref{def:g-C-M}, $\gamma$ in~\eqref{eq:def-beta-Q-b-bareta-gamma}, $\ubmhat$ in \eqref{def:ubmhat} and $\lbgmhat$ in \eqref{eq:def-lbg-upg-ubKvh}. Under these conditions, Lemmas~\ref{lem:rhorel-rhohpi-homo} and~\ref{lem:rhohrel-rhopihat-homo} hold. Therefore choose
\(
    C_3=\max\left\{\frac{2r_{\max}}{\delta_{\min}},2(C_6+C_7)\right\},
    C_4=C_6^{\mathrm{exp}}+C_7^{\mathrm{exp}},
    C_5=\min\{C_6^{\mathrm{rate}},C_7^{\mathrm{rate}}\}.
\)
These constants are independent of $N,n,\eta$, and $\bms_0$, which completes the proof of Theorem~\ref{thm-homoRB}.

\subsection{Proof Outline of Lemma~\ref{lem:rhohrel-rhopihat-homo}} \label{sec:proof-lem-rhohrel-rhopihat-homo}

 Unlike
Lemma~\ref{lem:rhorel-rhohpi-homo}, which analyzes the true two-set policy,
Lemma~\ref{lem:rhohrel-rhopihat-homo} analyzes the learned two-set policy
\(\pihts\) computed from the empirical model. Therefore, before applying the Lyapunov
framework, we must first verify that the empirical LP solution preserves
the structural properties required by the two-set policy and its Lyapunov
analysis. Theorem~\ref{thm:LP_perturbation_main_complete} exactly serves this purpose.

\paragraph{Step 1: Structural robustness from perturbation analysis.} By Theorem~\ref{thm:LP_perturbation_main_complete}, the empirical LP
preserves key structures of the true LP, including the unique neutral state and local stability. We then leverage these structures and perturbation bounds to show that the empirical Lyapunov function and associated quantities are well-defined; see details in Appendix~\ref{sssec:defs-hat-homo} and~\ref{subsec:wd-pro}.  
\paragraph{Step 2: Empirical Lyapunov analysis.}
We apply the Lyapunov analysis of the two-set policy to the empirical
system. For each subset \(D\subseteq[N]\), define the scaled
state-count vector
\(
    \bm X_t(D)=\left(\bm X_t(D,s)\right)_{s\in\States}
    \triangleq
    \left(
        \frac{1}{N}
        \sum_{i\in D}\mathbb I\{S_t(i)=s\}
    \right)_{s\in\States}.
\)
We use \(\bm X_t\) to denote the map
\(D\mapsto \bm X_t(D)\). Let
\(
    \Sigma_t
    =
    \bigl(
        \bm X_t,
        D_t^{\mathrm{OL}},
        D_t^{\pihstar}
    \bigr)
\)
denote the augmented state of the learned two-set policy. We write
\(
    \sigma
    =
    \bigl(
        \bmx,
        d^{\mathrm{OL}},
        d^{\pihstar}
    \bigr)
\)
for a generic augmented state. Let \(\widehat V\) be the empirical Lyapunov function defined
in~\eqref{equ:defi-Vhat}, which is constructed from
\(\widehat y^\star\), \(\widehat\mu^\star\), and the perturbed local
stability matrix \(\phihat\). The empirical two-set policy admits drift and
gap-dominance bounds, stated formally as
Lemmas~\ref{lem:DC-Vhat-homo} and~\ref{lem:C2-hat-RB},
respectively. Informally, for every augmented state \(\sigma\),
\(
    (\bphat^{\pihts}-I)\widehat V(\sigma)
    \le
    -\frac{\widehat\gamma}
    {4(\widehat\beta+r_{\max})}
    \widehat g(\sigma)
    +
    6\exp(8)S\widehat{\widetilde V}_{\max}
    \exp(-\widehat C N),
\)
and
\(
    \widehat\rho^{\mathrm{rel}}-r^{\pihts}(\sigma)
    \le
    \widehat g(\sigma).
\)
Thus, the empirical Lyapunov
function \(\widehat V\) satisfies Conditions~\ref{cond:drift}
and~\ref{cond:dominance} of our Lyapunov framework. 

Importantly, the perturbation analysis (Theorem~\ref{thm:LP_perturbation_main_complete}) ensures that the hatted objects
appearing in these bounds are well-defined and that the empirical
constants are uniformly controlled.

\paragraph{Step 3: Drift transfer.}
The drift bound in Step 2 is for the empirical transition kernel
\(\bphat^{\pihts}\), while Lemma~\ref{lem:rhohrel-rhopihat-homo} concerns
the performance of \(\pihts\) under the true model
\(\bp^{\pihts}\). We therefore apply
Lemma~\ref{lem:generic-drift-transfer} with
\(
    Q=\bphat^{\pihts},
    P=\bp^{\pihts},
    R=\widehat\rho^{\mathrm{rel}},
    F=\widehat g,
    V=\widehat V,
    \bar r=r^{\pihts}.
\)
We obtain
\[
\begin{aligned}
    \widehat\rho^{\mathrm{rel}}
    -
    \rho^{\pihts}(\bms_0)
    \le&
    \frac{4(\widehat\beta+r_{\max})}{\widehat\gamma}
    \sup_{\sigma}
    \left|
    (\bp^{\pihts}-\bphat^{\pihts})
    \widehat V(\sigma)
    \right|  +
    \frac{4(\widehat\beta+r_{\max})}{\widehat\gamma}
    6\exp(8)S\widehat{\widetilde V}_{\max}
    \exp(-\widehat C N).
\end{aligned}
\]
By Lemma~\ref{lem:p-phat-vhat}, the model-error term is upper bounded as
\(
    \sup_{\sigma}
    \left|
    (\bp^{\pihts}-\bphat^{\pihts})
    \widehat V(\sigma)
    \right|
    \le
    \widehat V_{\max}N\delta.
\)

\paragraph{Step 4: Replacing empirical constants by deterministic bounds.}
Finally, Theorem~\ref{thm:LP_perturbation_main_complete} and the auxiliary
well-definedness lemmas give
\(
    \widehat\gamma\ge \widehat\gamma^{\mathrm{low}},
    \widehat\beta\le \widehat\beta^{\mathrm{up}},
    \|\widehat V\|_\infty\le \widehat V_{\max},
    \widehat{\widetilde V}_{\max}<\infty,
\)
where all bounds depend only on the true single-armed RB parameters. Hence,
\(
\widehat\rho^{\mathrm{rel}}-\rho^{\pihts}(\bms_0)
\le C_7N\delta+C_7^{\mathrm{exp}}\exp(-C_7^{\mathrm{rate}}N),
\)
for constants \(C_7,C_7^{\mathrm{exp}},C_7^{\mathrm{rate}}\) depending only on
the true single-armed RB parameters, proving
Lemma~\ref{lem:rhohrel-rhopihat-homo}.

\section{Proof of Theorem~\ref{thm:LP_perturbation_main_complete}}\label{sec:proof-lp-pa}

In this section we provide the proof of Theorem \ref{thm:LP_perturbation_main_complete}, which will be preceded
by four steps in subsections
\ref{sec:pert_subsec_1}, \ref{sec:pert_subsec_2}, \ref{sec:pert_subsec_3} and \ref{subsec:final-pr-perturbation-full}. 
We first give a proof outline in Section~\ref{sec:pr-outline-pert} to show the high-level ideas.

\subsection{Proof Outline}\label{sec:pr-outline-pert}

The proof uses a constructive primal--dual perturbation argument. We show that a sufficiently small transition kernel perturbation preserves both the support of
the optimal primal solution and the set of tight  dual constraints. We thus construct candidate perturbed primal and dual solutions using the same support and equality structure as
in the true LP. We then verify feasibility,
complementary slackness, and uniqueness of the perturbed solutions.

Directly defining these candidates through the active constraints is complicated, due to  the redundant flow-balance constraint in
the primal LP~\eqref{eq:primal_LP_true} and the non-uniqueness of the dual solution. 
For unconstrained MDPs, a natural candidate for the perturbed primal solution is the stationary distribution $\muhat^\pi$ of the perturbed Markov chain $\Phat_\pi$ under the optimal policy $\pi$. We then have the standard \emph{perturbation identity} $\muhat^\pi - \mu^\pi = \muhat^\pi (\Phat^\pi - P^\pi) H_{P^\pi},$ where $H_{P^\pi}\triangleq\bigl(I-(P^\pi-\one\mu^\pi)\bigr)^{-1}-\one\mu^\pi$ is the group inverse of $I-P^\pi$, or deviation matrix \citep{meyer_jr_role_1975, meyer_jr_condition_1980, puterman_markov_1994}.  The matrix acts as a condition number measuring how much one-step transition errors affect stationary distributions.

However, the above argument cannot be applied directly to constrained MDPs. Our key idea is to eliminate the budget constraint $y \Tilde{e} = \alpha$ and reparameterize the problem using state-occupancy measures. This method allows us to first construct the perturbed stationary distribution $\muhat$ satisfying the perturbation identity (Step 1, Section \ref{sec:pert_subsec_1}) as the unconstrained MDPs. 
Specifically, given a state-occupancy vector $\nu$ near $\mu^\star\triangleq\mu^{\pistar} $, we can construct a state-action occupancy $y_\nu$ by following $\pistar$ in all non-neutral states and splitting the $\nu$-mass at the neutral state $\Tilde{s}$ between its two actions so that the used budget is exactly $\alpha$. This construction is represented by a linear map $W.$ Let \(c^\pi(s)\triangleq\pi(1\mid s)\) and \(\xi\triangleq P_{\tilde s1}-P_{\tilde s0}\).
If all arms follow $\pistar$ then the fraction of active arms would be $\nu c^{\pistar}$. Adjusting the active fraction within neutral state by $\alpha - \nu c^{\pistar}$, then the next occupancy is changed by $- (\nu c^{\pistar} - \alpha ) \xi.$
Consequently, the flow equation becomes $\nu=\nu \Gamma $ with $ \Gamma\triangleq P^{\pistar}-\bigl(c^{\pistar}-\alpha\one\bigr)\xi$. Thus, $\Gamma$ describes the local 
``budget-corrected'' occupancy dynamics. 

For the perturbed kernel, we define \(\xihat\triangleq
\Phat_{\tilde s1}-\Phat_{\tilde s0}\) and
\(\gmhat\triangleq\Phat^{\pistar}-\bigl(c^{\pistar}-\alpha\one\bigr)\xihat\). Although \(\Gamma\) and \(\gmhat\) are not transition matrices, their
perturbation theory resembles that of finite Markov chains. In particular, define \(H_{\Gamma}\triangleq(I-(\Gamma-\one\mu^\star))^{-1}-\one\mu^\star=(I-\Phi)^{-1}-\one\mu^\star.\)
We then define the candidate perturbed occupancy vector by
\(\muhat\triangleq\mu^\star\bigl(I-(\gmhat-\Gamma)H_{\Gamma}\bigr)^{-1}\). Equivalently, we have \(\muhat-\mu^\star=\muhat(\gmhat-\Gamma)H_{\Gamma}\), which is precisely analogous to the standard Markov-chain perturbation
identity. 
A central technical part of the proof establishes that the required algebraic properties remain valid even though \(\Gamma\)
and \(\gmhat\) are not stochastic matrices; see \ref{subsec:properties}.

We now remark on why it is no accident that the local stability condition, which appeared originally within the \emph{planning} analysis of the two-set policy, also plays a central role in this perturbation analysis.
The above neutral-state action adjustment is also employed \emph{algorithmically} within the two-set policy's \emph{Optimal Local Control}  (OLC) routine
\citep{hong2024achieving}, and more broadly by the LP-priority policy~\citep{verloop2016asymptotically}. More precisely, the OLC routine can stabilize its state occupancy around the equilibrium $\mu^{\star}$ even under stochastic dynamics, as long as the mean dynamic $\Gamma$ is well-conditioned. Since $\nu \Gamma - \mu^\star = (\nu - \mu^\star)\Phi$, the condition \(\rho(\Phi)<1\) ensures contraction toward $\mu^\star.$ 
In short, both the OLC stability and the LP sensitivity are closely linked to the neutral-state-adjusted system, and its conditioning is captured by $\Phi$.

We next construct the primal candidate solutions using $\muhat$ (Step 2, Section \ref{sec:pert_subsec_2}). We again apply the neutral state adjustment, represented by the matrix $W$, to construct the candidate $\yhat \triangleq\muhat W$. Note that the budget constraint is unchanged by the construction. 
It remains to verify the feasibility of $\yhat$ to the perturbed primal LP, as shown in Lemma~\ref{lem:main-perturbed-primal} by bounding $\|\yhat-y^\star\|.$ 

We then construct the dual candidate solutions, also based on $\muhat$ (Step 3, Section \ref{sec:pert_subsec_3}). Our idea is to exploit a key observation that the true dual solution $h^\star$ satisfies $h^\star=H_{\Gamma} Wr$ up to an additive constant (Lemma \ref{lem:true_dual_solution_form}). We are thus inspired to define the candidate dual $\hhat\triangleq H_{\gmhat}Wr$ where
\(H_{\gmhat}\triangleq
(I-(\gmhat-\one\muhat))^{-1}-\one\muhat\). 
We then construct the candidate dual $\lambdahat$ by using the complementary slackness, and the candidate dual $\zetahat$ by using the equality of the primal and dual optimal objectives. 
Lemma \ref{lem:main-perturbed-dual} bounds the distance between these candidates and the true dual solutions. Combined with the strict true action-value gap in Assumption \ref{ass:strict_action_gaps}, we ensure that all previously inactive dual inequalities remain strict. 

The constructed primal and dual candidate are therefore feasible and satisfy complementary slackness, so they are both optimal for the perturbed LP. The strict inequalities and error bounds show the preservation of the primal support and the tight dual constraints, establishing the main perturbation claim.
In particular, the perturbed optimal policy remains close to the true one. 
Standard Markov-chain perturbation bounds yield the
aperiodicity, irreducibility and mixing-time conclusions. We complete the proof of Theorem \ref{thm:LP_perturbation_main_complete} in Section \ref{subsec:final-pr-perturbation-full} (Step 4).

\subsection{Step 1: Candidate state occupancy construction}
\label{sec:pert_subsec_1}

With the notation introduced above, throughout the whole proof, fix any policy \(\pibst\) such
that
\(\pibst(\cdot\mid s)=\pistar(\cdot\mid s)\) for every
\(s\neq\tilde s\). Thus, \(\pibst\) and \(\pistar\) may differ only
at the neutral state. It is straightforward to check that
$P^{\pistar} - (c^{\pistar} -\alpha \one)\xi = P^{\pibst} - (c^{\pibst} -\alpha \one)\xi$. Similarly, \(\Phat^{\pistar}-\bigl(c^{\pistar}-\alpha\one\bigr)\xihat=\Phat^{\pibst}-\bigl(c^{\pibst}-\alpha\one\bigr)\xihat.\) We then define the local linearization matrix $\Gamma = P^{\pibst} - (c^{\pibst} -\alpha \one)\xi$ and its perturbed analogue
\(\gmhat = \Phat^{\pibst} - (c^{\pibst} -\alpha \one)\widehat \xi.\) We remark that \(\Gamma\) and \(\gmhat\) do not depend on the choice of
\(\pibst\) at the neutral state. We also define $H_{\Gamma} = (I - (\Gamma - \one \mu^\star))^{-1} - \one \mu^\star.$ Note that $\Gamma-\one \mu^\star=\Phi$, so the inverse $(I - (\Gamma - \one \mu^\star))^{-1}$ exists due to $\rho(\Phi)<1$.  Finally, we define the candidate state-occupancy vector 
\begin{equation}
    \muhat
    \triangleq
    \mu^\star
    \bigl(I-(\gmhat-\Gamma)H_{\Gamma}\bigr)^{-1}.
    \label{eq:muhat-definition}
\end{equation}
The invertibility of \(I-(\gmhat-\Gamma)H_{\Gamma}\) follows directly from the following lemma.
\begin{lemma}\label{lem:two-condtions}
    Under the conditions of Theorem~\ref{thm:LP_perturbation_main_complete}, the following condition is satisfied.
    \begin{equation}
    \infinfnorm{\gmhat-\Gamma}\infinfnorm{H_{\Gamma}}
    \le\frac{1}{2}\min_{i:y_i^\star>0}y_i^\star\le\frac14.
    \label{eq:well-define-Huhat}
\end{equation}
\end{lemma}
The proof of Lemma~\ref{lem:two-condtions} is given in~\ref{subsec:proof-two-conditions}. 
The motivation for the construction of $\muhat$ in \eqref{eq:muhat-definition} is
discussed in the proof outline. We next state a key identity
satisfied by \(\muhat\):
\begin{equation}\label{eq:equaldef-muhat}
    \muhat - \mu^\star = \muhat (\gmhat - \Gamma)H_{\Gamma}.
\end{equation}

Right-multiplying both sides of~\eqref{eq:muhat-definition} by $(I - (\gmhat-\Gamma)H_{\Gamma})$, we have that $\mu^\star = \muhat (I - (\gmhat-\Gamma)H_{\Gamma}) = \muhat - \muhat (\gmhat - \Gamma)H_{\Gamma},$ which can be directly rearranged to obtain \eqref{eq:equaldef-muhat}.
The basic properties of \(\Gamma\), \(\gmhat\), \(\mu^\star\), \(\muhat\), \(H_{\Gamma}\) are collected in~\ref{subsec:properties}.
Given the definition of $\muhat$ in \eqref{eq:muhat-definition} and \eqref{eq:equaldef-muhat}, we can now also define the perturbed version of $H_{\Gamma}$ as $H_{\gmhat} \triangleq (I - (\gmhat - \one \muhat))^{-1} - \one \muhat.$
This definition requires the invertibility of $(I - (\gmhat - \one \muhat))$, which is provided by the following lemma.
\begin{lemma}
\label{lem:well-define-Huhat}
Suppose that $\rho(\Phi)<1$ and \eqref{eq:well-define-Huhat} holds, the matrix $I-(\gmhat-\one\muhat)$
is invertible and $\muhat$ in~\eqref{eq:muhat-definition} is a probability vector; that is,
\(\muhat\one=1\) and \(\muhat(s)\ge 0\) for every \(s\in\States\).
\end{lemma}

Lemma~\ref{lem:well-define-Huhat} is proved in \ref{sec:pr-well-define-Huhat}, and \ref{subsec:properties-Huhat} shows useful properties of \(H_{\gmhat}\). 

We next construct candidate solutions for the primal and dual perturbed
LPs~\eqref{eq:primal_LP_pert} and~\eqref{eq:dual_LP_pert}, and then verify
their feasibility and optimality, as detailed in  \ref{sec:pert_subsec_2} and \ref{sec:pert_subsec_3}.

\subsection{Step 2: Perturbed primal LP solution construction}
\label{sec:pert_subsec_2}

We first construct the candidate solution for the perturbed primal LP. By
the constraint~\eqref{eq:true_primal_arm_frac_const} of the true primal LP,
we have \(y^\star\tilde e=\alpha\). Moreover, by the nondegeneracy
assumption, for all \(s\neq\tilde{s}\),
\(y^\star(s,a)=\mu^\star(s)\mathbb{I}\{\pibst(a\mid s)=1\}\). Therefore,
we can reconstruct \(y^\star\) from \(\mu^\star\) and
\(\pibst\), and then use \(y^\star\tilde e=\alpha\) to determine
\(y^\star(\tilde{s},0)\) and \(y^\star(\tilde{s},1)\). Hence \(y^\star\)
satisfies
\begin{align*}
    y^\star(s,a) &= \mu^\star(s) \pibst(a \mid s),  \quad \forall  s \neq \Tilde{s},  \\
    y^\star(\Tilde{s},1) &= \alpha -\sum_{s \neq \Tilde{s}} \pibst(1 \mid s) \mu^\star(s) =\alpha\left(\mu^\star \one \right) -\sum_{s \neq \Tilde{s}} \pibst(1 \mid s) \mu^\star(s),  \\
    y^\star(\Tilde{s},0) &= (1 - \alpha) - \sum_{s \neq \Tilde{s}} \pibst(0 \mid s) \mu^\star(s)  = (1 - \alpha)\left(\mu^\star \one \right) - \sum_{s \neq \Tilde{s}} \pibst(0 \mid s) \mu^\star(s).
\end{align*}
We write the last two equations using \(1=\mu^\star\one\) so that the
above display defines a linear system with matrix \(W\) satisfying
\(y^\star=\mu^\star W\). Using the same linear map, we define the candidate perturbed
solution \(\yhat\triangleq\muhat W\). This construction ensures
that \(\yhat(s,a)=0\) for all \((s,a)\) such that \(y^\star(s,a)=0\).
 The matrix $W$ will play a key role in all of our subsequent analysis and the properties of \(W\)  are collected
in Lemma~\ref{lem:W_properties}. The following lemma guarantees the feasibility of $\yhat$ and is proved in \ref{subsec:pr-feasibility-yhat}.
\begin{lemma}
\label{lem:main-perturbed-primal}
\label{lem:yhat_pert_and_feasibility}
Suppose that $\rho(\Phi)<1$,
$\infinfnorm{\Phat-P}\le\delta$, and
\eqref{eq:well-define-Huhat}
holds. Then \(\yhat=\muhat W\) satisfies
\begin{equation}
    \infnorm{\yhat-y^\star}
    \le
    \infinfnorm{\gmhat-\Gamma}\infinfnorm{H_{\Gamma}}\le 3\delta  \infinfnorm{H_{\Gamma}}.
    \label{eq:yhat_pert_bound}
\end{equation}
Furthermore, \(\yhat\ge y^\star/2\) elementwise and \(\yhat\) is feasible for the
perturbed primal LP~\eqref{eq:primal_LP_pert}. In particular, the support
of \(\yhat\) agrees with the support of \(y^\star\).
\end{lemma}

\subsection{Step 3: Perturbed dual LP solution construction}
\label{sec:pert_subsec_3}

Next we construct the candidate dual solution $(\zetahat, \lambdahat, \hhat)$. We note that the complementary slackness property~\eqref{eq:LP_CS_true} of the true LP, combined with the form of $y^\star$, implies that for all $i=(s,a)$ such that $y_i^\star > 0$ we have
\begin{align}
    r(i) + \lambda \tilde{e}_i + P_{i}h = \zeta + h(s), \label{eq:true_CS_equality_consequence}
\end{align}
which will be used to guide the construction of candidate dual solution. The following lemma gives an explicit form of the solution $h^\star$ to the true dual LP, which inspires us to construct the candidate perturbed dual solution $\widehat{h}$ below.
\begin{lemma}
\label{lem:true_dual_solution_form}
    There exists some $\kappa \in \R$ such that
    $
        h^\star = H_{\Gamma} Wr + \kappa \one.
    $
\end{lemma}
Lemma~\ref{lem:true_dual_solution_form} is proved in \ref{subsec:pr-h-construction}. We now construct the candidate dual solution. Motivated by
Lemma~\ref{lem:true_dual_solution_form}, we define
\(
    \hhat \triangleq H_{\gmhat}Wr.
\)
It remains to define \(\lambdahat\) and \(\zetahat\). Since~\eqref{eq:true_CS_equality_consequence} holds for both $(\tilde{s}, 0)$ and $(\tilde{s}, 1)$, we have the ``equal Q-function'' property for the neutral state that
$
    r(\tilde{s},0) + P_{\Tilde{s}0} h^\star = r(\tilde{s},1) + \lambda^\star + P_{\Tilde{s}1} h^\star.
$ Therefore, \(
\lambda^\star
=
r(\tilde s,0)-r(\tilde s,1)-\xi h^\star.
\)
We choose \(\lambdahat\) so that the same indifference holds for the
perturbed dual candidate:
\(
    r(\tilde{s},0)+\Phat_{\tilde{s}0}\hhat
    =
    r(\tilde{s},1)+\lambdahat+\Phat_{\tilde{s}1}\hhat .
\)
Equivalently, we set
\begin{align}
    \lambdahat
    \triangleq
    r(\tilde{s},0)-r(\tilde{s},1)-\xihat\hhat .
    \label{eq:lambdahat_defn}
\end{align}
Finally, as noted in the proof of Lemma \ref{lem:true_dual_solution_form}, the equality of the optimal objective values for the primal and dual LPs implies that $\zeta^\star = \mu^\star Wr + \lambda^\star \alpha$. We thus define $\zetahat$ as
\begin{align}
    \zetahat \triangleq \muhat W r  + \lambdahat \alpha,   
    \label{eq:zetahat_defn} 
\end{align}
so that the perturbed dual objective matches the objective value of the
candidate primal solution.

Now having defined $(\zetahat, \lambdahat, \hhat)$, we find conditions for their dual feasibility. The only constraint which we need to check is~\eqref{eq:pert_dual_ineq}. We check each row of~\eqref{eq:pert_dual_ineq}, dividing them into two cases based on whether the constraint will hold with equality or inequality. First we handle the equality case.
\begin{lemma}
\label{lem:pert_dual_LP_equality_rows}
The candidate dual variables $(\hhat,\lambdahat,\zetahat)$ satisfy
\(
    r^{\pibst}
    +\Phat^{\pibst}\hhat
    +\lambdahat c^{\pibst}
    =
    \hhat+\zetahat\one,
\)
and
\(
    r(\tilde s,0)+\Phat_{\tilde s0}\hhat
    =
    r(\tilde s,1)+\lambdahat+\Phat_{\tilde s1}\hhat
    =
    \hhat(\tilde s)+\zetahat.
\)
Consequently, for every \((s,a)\) such that \(y^\star(s,a)>0\),
\(
    r(s,a)
    +\Phat_{sa}\hhat
    +\lambdahat\mathbf 1\{a=1\}
    =
    \hhat(s)+\zetahat.
\)
\end{lemma}

Lemma \ref{lem:pert_dual_LP_equality_rows} shows that the dual constraints are active on the support of $y^\star.$ However, this is not sufficient for dual feasibility, because we must also verify
the dual inequalities for the inactive pairs \((s,a)\) with
\(y^\star(s,a)=0\). The next lemma shows that the perturbed dual variables
remain close to their true counterparts, so dual constraints still have positive slacks after perturbation.
\begin{lemma}
\label{lem:main-perturbed-dual}
\label{lem:main-dual-gap-preservation}
\label{lem:pert_dual_LP_ineq_consts}
Suppose
\(\infinfnorm{\Phat-P}\le\delta\),
\(\rho(\Phi)<1\), and
\eqref{eq:well-define-Huhat} holds. Then
\begin{subequations}\label{eq:dual-perturbation-bounds}
\begin{gather}
    \spannorm{\hhat-h^\star}
    \le
    \frac{
        \infinfnorm{H_{\Gamma}}
        \infinfnorm{\gmhat-\Gamma}
        \spannorm{h^\star}
    }{
        1-\infinfnorm{H_{\Gamma}}\infinfnorm{\gmhat-\Gamma}
    }\leq 6 \delta \infinfnorm{H_{\Gamma}} \spannorm{h^\star},
    \label{eq:dual-perturbation-h}\\
    |\lambdahat-\lambda^\star|
    \le
    \frac12
    \infinfnorm{\xi-\xihat}
    \spannorm{h^\star}
    +
    \spannorm{\hhat-h^\star}\le \delta\spannorm{h^\star}
\left(1+6\infinfnorm{H_{\Gamma}}\right),
    \label{eq:dual-perturbation-lambda}\\
    |\zetahat-\zeta^\star|
    \le
    \frac12
    \infinfnorm{\gmhat-\Gamma}
    \spannorm{h^\star}
    +
    |\lambdahat-\lambda^\star|\le \delta\spannorm{h^\star}
\left(\frac{5}{2}+6\infinfnorm{H_{\Gamma}}\right).
    \label{eq:dual-perturbation-zeta}
\end{gather}
\end{subequations}
Furthermore, suppose that
\(
    \delta
    \le
    \frac{
        \min_{(s,a):\,y^\star(s,a)=0}
        \left[
            \zeta^\star+h^\star(s)-r(s,a)
            -\lambda^\star\mathbb I\{a=1\}
            -P_{sa}h^\star
        \right]
    }{
        \left(8+36\infinfnorm{H_{\Gamma}}\right)
        \spannorm{h^\star}
    },
\)
then for every \((s,a)\) such that \(y^\star(s,a)=0\),
\[
    \zetahat+\hhat(s)-r(s,a)
    -\lambdahat\mathbb I\{a=1\}
    -\Phat_{sa}\hhat
    \ge
    \frac12
    \left[
        \zeta^\star+h^\star(s)-r(s,a)
        -\lambda^\star\mathbb I\{a=1\}
        -P_{sa}h^\star
    \right].
\]
\end{lemma}
The first part of the lemma shows how much the dual variables change. Together with Assumption~\ref{ass:strict_action_gaps},
the last inequality shows that every inactive pair \((s,a)\) still has
positive dual slack after perturbation.
Combining this with Lemma~\ref{lem:pert_dual_LP_equality_rows}, the perturbed dual constraints are tight on
the support of \(y^\star\) and strictly slack outside the support.
Hence the candidate dual solution is feasible. 
Lemma~\ref{lem:pert_dual_LP_equality_rows} and Lemma~\ref{lem:pert_dual_LP_ineq_consts} are proved in \ref{subsec:pr-feasibility-dual}.

\subsection{Step 4: Completing the proof of Theorem \ref{thm:LP_perturbation_main_complete}}\label{subsec:final-pr-perturbation-full} 
First we verify the optimality of the candidate solutions $\yhat$ and $(\zetahat, \lambdahat, \hhat)$ for their respective LPs. 
Lemma \ref{lem:yhat_pert_and_feasibility} guarantees primal feasibility of $\yhat$, and Lemmas \ref{lem:pert_dual_LP_equality_rows} and \ref{lem:pert_dual_LP_ineq_consts} guarantee dual feasibility of $(\zetahat, \lambdahat, \hhat)$. Since by construction these primal and dual solutions have the same objective value of $\muhat W r = \zetahat - \alpha \lambdahat$, they are both optimal.
The next lemma shows the uniqueness of both the primal and dual solutions and is proved in \ref{subsec:pr-uniqueness}.

\begin{lemma}
\label{lem:uniqueness}
Under the conditions of Theorem~\ref{thm:LP_perturbation_main_complete}, the candidate solution \(\yhat\) is the unique optimal solution to the
perturbed primal LP~\eqref{eq:primal_LP_pert}. Moreover,
\((\zetahat,\lambdahat,\hhat)\) is the unique optimal solution to the
perturbed dual LP~\eqref{eq:dual_LP_pert} satisfying
\(\muhat\hhat=0\).
\end{lemma}

Henceforth, we write
\(\yhat^\star\triangleq\yhat\),
\(\zetahat^\star\triangleq\zetahat\),
\(\lambdahat^\star\triangleq\lambdahat\), and
\(\hhat^\star\triangleq\hhat\). Define \(\muhst(s)\triangleq\yhat^\star(s,0)+\yhat^\star(s,1)\) and according to the identity
\(WJ=I\) from Lemma~\ref{lem:W_properties}, we have \(\muhst =\yhat^\star J=\yhat J=\muhat WJ=\muhat\). Parts 1 and 2 now follow from
Lemma~\ref{lem:main-perturbed-primal}.
Part 3 follows from the perturbation bounds in
Lemmas~\ref{lem:main-perturbed-primal}
and~\ref{lem:main-perturbed-dual}.
Part 4 follows from the strict-gap conclusion in
Lemma~\ref{lem:main-perturbed-dual}.
For Part 5, support preservation in Part 1 gives
$\phihat=\gmhat-\one\muhst$.
The bounds on $\gmhat-\Gamma$ and $\muhst-\mu^\star$
then control $\twonorm{\phihat-\Phi}$.
The key step is to use the bound on $\twonorm{\phihat-\Phi}$ and
$U-\Phi U\Phi^\top=I$ to obtain
\(
U-\phihat U\phihat^\top\succeq\frac12I,
\)
under the first accuracy condition
in~\eqref{eq:def-delta-min}.
This inequality establishes both $\rho(\phihat)<1$
and $I\preceq\uhat\preceq2U$.
The bound on
$\infinfnorm{(I-\phihat)^{-1}}$ follows from
Lemma~\ref{lem:gramian-inverse-bound}.
Appendix~\ref{subsec:proof-perturbed-spectral-gramian}
gives the full proof.
For Part 6, Parts 1 and 3 control the difference between $\pi^\star$ and $\widehat \pi^\star$.
Combining this difference with the model error gives
$\infinfnorm{\Phat^{\pihstar}-P^{\pistar}}
\le[4S\tau(5+\log_2S)]^{-1}$.
Lemma~\ref{lem:mixing_time_pert_bound} then shows that the
empirical chain is unichain and aperiodic, with mixing time
at most $(2+\log_2S)\tau$.
The flow equations in~\eqref{eq:primal_LP_pert} imply
that $\muhst$ is its stationary distribution.
Since Part 1 gives $\muhst(s)>0$ for every state,
the chain is also irreducible.
Appendix~\ref{subsec:pr-mc-ergordic} provides the details.
This completes the proof of
Theorem~\ref{thm:LP_perturbation_main_complete}.
\section{Conclusion}
In this work, we study the sample complexity of learning in average-reward weakly-coupled
MDPs and restless bandits under a generative model. We analyze the plug-in method that applies an
efficient planning policy to the empirical model 
estimated from data. For both heterogeneous WCMDPs and homogeneous RBs, we obtain finite-sample PAC guarantees with sample and
computational complexities scaling polynomially in the number of arms $N$. 

Our analysis is based on a novel Lyapunov framework that transfers drift
bounds between the true and empirical models, avoiding direct control
of the bias function of the full \(N\)-armed system. We also develop an LP perturbation result showing that key structural properties of the LP solution and its induced policy are preserved under small perturbation. We believe that these tools are broadly applicable to the analysis of other stochastic systems and reinforcement learning algorithms. An interesting future direction is to investigate computationally efficient online learning for WCMDPs/RBs, particularly   LP-based algorithms. 

\section*{Acknowledgments}
T.\ Wu and Q.\ Xie are supported in part by National Science Foundation (NSF) Grants CNS-1955997, ECCS-2339794 and ECCS-2432546. T.\ Wu is also partially supported by NSF Award DMS-2023239.
W.\ Wang is supported in part by the NSF Grants ECCS-2145713, CCF-2403194, CCF-2428569, and ECCS-2432545.
Part of this research was performed while M.\ Zurek was visiting the Institute for Mathematical and Statistical Innovation (IMSI), which is supported by the NSF (Grant No.\ DMS-2425650). M.\ Zurek also acknowledges support by a Cisco Systems Fellowship and by NSF grant CCF-2233152.

\bibliographystyle{plainnat}
\bibliography{references}
\newpage
\appendix

\section{Notation table}\label{sec:notation}
We summarize notation used throughout the paper in the following table.
\begin{longtable}{p{0.22\linewidth}p{0.74\linewidth}}
\textbf{Symbol} & \textbf{Meaning} \\
\hline
$\mathcal{S}, \Actions$ & State and Action space of each arm \\
$S,A$ & Cardinality of the state and action space, $S = |\mathcal{S}|, A=|\Actions|$ \\
$s,a$ & one-dimension state and action \\
$\bms,\bma$ & $N$-dimension state and action  \\
$P(\cdot \mid s,a)$ & True transition probability given $(s,a)$ for single arm \\
$\widehat{P}(\cdot \mid s,a)$ & Estimated transition probability given $(s,a)$ for single arm  \\
$\bp(\cdot \mid \bms,\bma)$ & True transition probability given $(\bms,\bma)$ for the $N$-arm system \\
$\widehat{\bp}(\cdot \mid \bms,\bma)$ & Estimated transition probability given $(\bms,\bma)$ for the $N$-armed system  \\
$\pistar$ & Optimal single-armed policy for the true system\\
$\pihstar$ & Optimal single-armed policy for the empirical system\\
$\mu^{\star}$ & Stationary distribution obtained by running $\pi^{\star}$ under the true MDP \\
$\widehat{\mu}^{\star}$ & Stationary distribution obtained by running $\pihstar$ under the empirical MDP \\
$P^{\pi}$ & One-step transition matrix induced by a policy $\pi$ \\
$P^{\pi,\infty}$ & Limiting transition matrix of the Markov chain under policy $\pi$\\
$\|\cdot\|_{1}$ & $\ell_{1}$ norm of a vector \\
$\one$ & All-ones vector \\
$\mathbb{I}$ & Indicator function \\
$\lambda_{W}$ & Maximal eigenvalue of a matrix $W$ \\
$\mathbb{E}^\Pi$ & Expectation under policy $\Pi$ and the true transition kernel $\bp$ \\
$\widehat{\mathbb{E}}^\Pi$ & Expectation under policy $\Pi$ and the empirical transition kernel $\bphat$ \\
$\rho^\Pi$ & Average reward of policy $\Pi$ under true model \\
$\widehat{\rho}^\Pi$ & Average reward of policy $\Pi$ under empirical model \\
\hline
\end{longtable}

\section{Technical Lemmas}
\label{sec:common-technical-lemmas}

This section collects the proof of Lemma~\ref{lem:generic-drift-transfer} together with several
finite-state Markov-chain perturbation bounds that will used later in the proofs.

\begin{proof}{Proof of Lemma~\ref{lem:generic-drift-transfer}.}
Fix an arbitrary initial state \(z_0\), and let
\(\Delta\triangleq\sup_{z\in\mathcal Z}|[(P-Q)V](z)|\).
For every \(z\in\mathcal Z\), we have
\((P-I)V(z)=(Q-I)V(z)+(P-Q)V(z)\le -C_0F(z)+G_1(N)+\Delta\). Averaging this
inequality along the Markov chain governed by \(P\), we obtain, for every
\(T\ge1\),
\[
    \frac{\mathbb E_P[V(Z_T)\mid Z_0=z_0]-V(z_0)}{T}
    \le
    -C_0
    \frac1T
    \sum_{t=0}^{T-1}
    \mathbb E_P[F(Z_t)\mid Z_0=z_0]
    +
    G_1(N)+\Delta .
\]
Since \(\mathcal Z\) is finite, \(V\) is bounded, so
\((\mathbb E_P[V(Z_T)\mid Z_0=z_0]-V(z_0))/T\to0\). Rearranging the previous
display and taking limits gives
\begin{equation}\label{eq:limF}
    \lim_{T\to\infty}
    \frac1T
    \sum_{t=0}^{T-1}
    \mathbb E_P[F(Z_t)\mid Z_0=z_0]
    \le
    \frac{G_1(N)+\Delta}{C_0}.
\end{equation}
Next, averaging \(R-\bar r(z)\le C_0'F(z)+G_2(N)\) along the same trajectory
gives, for every \(T\ge1\),
\[
    R-
    \frac1T
    \sum_{t=0}^{T-1}
    \mathbb E_P[\bar r(Z_t)\mid Z_0=z_0]
    \le
    C_0'
    \frac1T
    \sum_{t=0}^{T-1}
    \mathbb E_P[F(Z_t)\mid Z_0=z_0]
    +
    G_2(N).
\]
Taking limits on both sides yields
\[
    R-\rho_P(z_0)
    \le
    C_0'
    \lim_{T\to\infty}
    \frac1T
    \sum_{t=0}^{T-1}
    \mathbb E_P[F(Z_t)\mid Z_0=z_0]
    +
    G_2(N).
\]
Combining this inequality with~\eqref{eq:limF} we obtain the desired result. \Halmos
\end{proof}

\begin{lemma}
    \label{lem:mixing_time_pert_bound}
    Let $Q, \widetilde{Q}$ be two Markov transition matrices on the same finite state space $\mathcal Z$. Suppose that the chain $Q$ is unichain, aperiodic, with mixing time bounded by $\tau$. If $\infinfnorm{\widetilde{Q} - Q} \leq \frac{1}{4|\mathcal Z|\tau (5 + \log_2 |\mathcal Z|)}$,
    then Markov chain $\widetilde{Q}$ is also unichain, aperiodic, and has mixing time bounded by $(2 + \log_2 |\mathcal Z|) \tau$.
\end{lemma}
\proof{Proof.}
    Let $Q^\infty = \one \mu^\top$ for some $\mu$, which is a valid expression since $Q$ is unichain and aperiodic. Suppose $\infinfnorm{\widetilde{Q} - Q} \leq\delta$. Let $t\ge\tau$ be a positive integer which we will choose later.
    By triangle inequality
    $
        \infinfnorm{\qtilde^t - Q^\infty}  
        \leq \infinfnorm{\qtilde^t - Q^t} + \infinfnorm{Q^t - Q^\infty} .
    $
    Now we bound each of these two terms. First,
    $
    \qtilde^t - Q^t=\sum_{k=0}^{t-1}\qtilde^{t-1-k}(\qtilde-Q)Q^k,
    $
    and so
    \begin{align*}
        \infinfnorm{\qtilde^t - Q^t} 
        \le & \sum_{k=0}^{t-1}\infinfnorm{\qtilde^{t-1-k}}\infinfnorm{\qtilde-Q}\infinfnorm{Q^k} 
        \leq  \sum_{k=0}^{t-1}1\cdot\delta\cdot 1 
        \leq   t \delta,
    \end{align*}
    since each $\qtilde^{t-k} Q^{k-1}$ (for $k = 1, \dots, t$) is a stochastic matrix and thus has $\infinfnorm{\cdot}$ equal to $1$. 
    By Lemma~\ref{lem:Qt-Qinfty} we have
$
    \infinfnorm{Q^t - Q^\infty}\le2^{-\lfloor t/\tau\rfloor}.
$
Combining the above two results, we have
    $
        \infinfnorm{\qtilde^t- Q^\infty}\leq t \delta + 2^{-\lfloor t/\tau\rfloor}< t \delta + 2^{- t/\tau+1}.
    $
    We can choose $t=(2+\log_2 |\mathcal Z|)\tau$, $\delta=\frac{1}{2|\mathcal Z|\tau(2+\log_2|\mathcal Z|)}$ so that $\infinfnorm{\qtilde^t - Q^\infty} < \frac{1}{|\mathcal Z|}$. Since $Q^\infty = \one \mu^\top$, we thus have $\max_{z\in\mathcal Z} \mu(z) \geq \frac{1}{|\mathcal Z|}$ by a standard argument. Let $z^\star$ have $\mu(z^\star) \geq \frac{1}{|\mathcal Z|}$. Then for all states $z_0$, we have
    \begin{align*}
        e_{z_0}^\top \qtilde^{t} e_{z^\star} &= e_{z_0}^\top Q^\infty e_{z^\star} - e_{z_0}^\top (Q^\infty - \qtilde^{t})e_{z^\star} \\
        & \geq e_{z_0}^\top Q^\infty e_{z^\star} - \infinfnorm{e_{z_0}^\top} \infinfnorm{Q^\infty - \qtilde^{t}} \infnorm{e_{z^\star}} \\
        &= \mu(z^\star) - \infinfnorm{Q^\infty - \qtilde^{t}} 
         > \frac{1}{|\mathcal Z|} - \frac{1}{|\mathcal Z|} = 0.
    \end{align*} 
    Therefore, for any $z_0$, there is positive probability of reaching $z^\star$ after $t$ steps. Consequently, if $z_0$ is recurrent, then $z^\star$ is also recurrent and they are in the same recurrent class. Thus all recurrent classes contain $z^\star$, so there is only one recurrent class. That is, $\qtilde$ is unichain.

    Next we bound the mixing time of $\qtilde$.
    By triangle inequality
    $
        \infinfnorm{\qtilde^t - \qtilde^\infty}  
        \leq \infinfnorm{\qtilde^t - Q^t} + \infinfnorm{Q^t - Q^\infty} + \infinfnorm{Q^\infty - \qtilde^\infty}.
    $
    We have already bounded the first two terms.    
    For the third term we have
    $
        \widetilde{Q}^\infty - Q^\infty 
        = \widetilde{Q}^\infty \widetilde{Q} - Q^\infty Q 
        = \widetilde{Q}^\infty (\widetilde{Q} - Q) + (\widetilde{Q}^\infty - Q^\infty)Q,
    $
    which implies
    $
        (\widetilde{Q}^\infty - Q^\infty)(I-Q)
        =\widetilde{Q}^\infty (\widetilde{Q} - Q).
    $
    Let
\(
    H_Q
    \triangleq
    \bigl(I-(Q-Q^\infty)\bigr)^{-1}-Q^\infty.
\)
Note that $H_Q$ satisfies
\((I-Q)H_Q=H_Q(I-Q)=I-Q^\infty.\) Right-multiplying by $H_Q$, we have
    $
        \widetilde{Q}^\infty (\widetilde{Q} - Q)H_Q 
        = (\widetilde{Q}^\infty - Q^\infty)(I-Q)H_Q 
        = (\widetilde{Q}^\infty - Q^\infty)(I-Q^\infty) 
        = \widetilde{Q}^\infty - Q^\infty,
    $
    using the fact that $\widetilde{Q}^\infty Q^\infty = \widetilde{Q}^\infty \one \mu^\top = \one \mu^\top  = Q^\infty$. Hence,
    \begin{align*}
        \infinfnorm{Q^\infty - \qtilde^\infty} &= \infinfnorm{\qtilde^\infty (\qtilde - Q) H_{Q}}\\
        &\leq \infinfnorm{\qtilde^\infty} \infinfnorm{\qtilde - Q} \infinfnorm{H_{Q}} 
         \leq \infinfnorm{\qtilde - Q} 3\tau 
         \leq 3\delta \tau,
    \end{align*}
    using Lemma \ref{lem:deviation_mtx_mixing_bound} in the second to last inequality. 
    Therefore, we have
    $
        \infinfnorm{\qtilde^t - \qtilde^\infty} \leq (t+3\tau)\delta + 2^{-\lfloor\frac{t}{\tau}\rfloor} .
    $
    By choosing $t=(2+\log_2 |\mathcal Z|)\tau,\delta=\frac{1}{4(5+\log_2|\mathcal Z|)\tau}$,  we can upper bound $\infinfnorm{\qtilde^t - \qtilde^\infty}$ $\leq \frac{1}{2}$, which implies that the mixing time of $\qtilde$ is within  $(2+\log_2 |\mathcal Z|)\tau$. It also implies that $\qtilde$ is aperiodic. 
\endproof
\begin{lemma}\label{lem:Qt-Qinfty}
    Suppose that the Markov chain $Q$ is unichain, aperiodic, and has mixing time $\tau$. Then for any time $t\ge \tau$,
    \(
\infinfnorm{Q^t - Q^\infty}\le2^{-\lfloor t/\tau\rfloor}.
\)
\end{lemma}
\proof{Proof.}
    Denote $\mu$ as the stationary distribution of the Markov chain $Q$. Hence $Q^\infty=\one\mu$. By the definition of the mixing time $\tau$, we have
\[
\|Q^{\tau} - Q^\infty\|_{\infty\to\infty}=\max_{s}\sum_{j}|Q^{\tau}(s,j)-\mu(j)|=2\max_s\|Q^{\tau}(s,\cdot)-\mu(\cdot)\|_{\text{TV}}  \le \frac12.
\]
Let \(
t = k\,\tau + r, 0 \le r < \tau.
\) Then
\(
Q^t - Q^\infty
= Q^r \bigl(Q^{k\,\tau} - Q^\infty\bigr).
\) Hence, $\forall t\ge\tau$,
\[
\infinfnorm{Q^t - Q^\infty}\le\infinfnorm{Q^{k\tau} - Q^\infty}=\infinfnorm{(Q^\tau-Q^\infty)^k}\le\infinfnorm{Q^{\tau} - Q^\infty}^k\le\left(\frac{1}{2}\right)^k=2^{-\lfloor t/\tau\rfloor}. \Halmos
\] 
\endproof
\begin{lemma}
\label{lem:deviation_mtx_mixing_bound}
Let \(Q\) be a finite-state Markov transition matrix. Suppose that
\(Q\) is unichain and aperiodic, with mixing time bounded by \(\tau\).
Let \(\mu\) be its stationary distribution, and define
\(Q^\infty\triangleq\one\mu\) and
\(H_Q\triangleq\bigl(I-(Q-Q^\infty)\bigr)^{-1}-Q^\infty\).
Then
\(\infinfnorm{H_Q}\le 3\tau\).
\end{lemma}

\proof{Proof.}
Since \(Q\) is unichain and aperiodic, \(Q^t\) converges to
\(Q^\infty=\one\mu\), so \(H_Q=\sum_{t=0}^{\infty}(Q^t-Q^\infty)\). For
\(0\le t<\tau\), both \(Q^t(s,\cdot)\) and \(\mu\) are probability
distributions, and hence
\(\infinfnorm{Q^t-Q^\infty}\le 2\).
For \(t\ge\tau\), Lemma~\ref{lem:Qt-Qinfty} gives
\(\infinfnorm{Q^t-Q^\infty}\le
2^{-\lfloor t/\tau\rfloor}\).
Therefore,
\begin{align*}
        \infinfnorm{H_Q}
        &= \infinfnorm{\sum_{t=0}^\infty \left( Q^{t} - Q^{\infty}\right)}
        \leq \sum_{t=0}^\infty \infinfnorm{ Q^{t} - Q^{\infty}}\\
        & \leq \sum_{t=0}^{\tau-1}2
        + \sum_{k=1}^\infty \sum_{j=0}^{\tau - 1}
        2^{-\lfloor\frac{\tau k + j}{\tau}\rfloor}
        \leq 2\tau+\sum_{k=1}^\infty \sum_{j=0}^{\tau - 1}
        2^{-\lfloor\frac{\tau k}{\tau}\rfloor}
        = 2\tau+\tau \sum_{k=1}^\infty 2^{-k}
        = 3 \tau. \Halmos
    \end{align*}
\endproof

\section{Preliminaries for the proof of Theorem~\ref{thm-heteWCMDP}}\label{sec:preliminary-thm-hete}

Theorem~\ref{thm-heteWCMDP} considers the ID policy, proposed by~\citet{zhang2025projection}, as the reference policy. To facilitate the presentation of the proof of Theorem~\ref{thm-heteWCMDP},  in this section we provide some preliminaries on the ID policy and its analysis.
We present the description of ID policy in Section~\ref{subsec:id-description}, the analysis of ID policy from~\citet{zhang2025projection} in Section~\ref{subsec:id-analysis}, and the empirical ID policy in Section~\ref{subsec:empirical-id-policy}.

\subsection{Description of the ID policy}
\label{subsec:id-description}
In this subsection, we describe the ID policy and the corresponding ID reassignment subroutine.

\paragraph{ID policy execution (ID-based prioritization).}
Each arm is assigned a unique \emph{ID} in $\{1,\dots,N\}$ (after reassignment described below).
At each time $t$, the policy first samples an \emph{ideal action}
$
    \Aideal_{i,t}\sim \pi_i^\star(\cdot\mid S_{i,t}),
$ for each arm $i\in[N].$
It then attempts to execute these ideal actions in ascending ID order.
Define the conforming number $N_t^\star$, i.e.,
the number of arms that actually conform to their ideal actions at time $t$:
\begin{equation}\label{eq:conforming-number}
    N_t^\star
    =\max\left\{n\in\{0,1,\ldots,N\}:
    \sum_{i=1}^n c_{k,i}(S_{i,t},\Aideal_{i,t})\le \alpha_kN,
    \quad \forall k\in[K]\right\}.
\end{equation}
The realized action is $A_{i,t}=\Aideal_{i,t}$ for $i\in[N_t^\star]$, and $A_{i,t}=0$ otherwise. Details of the ID policy is shown in Algorithm~\ref{alg:IDpolicy}.
\begin{algorithm}[htbp]
\caption{ID policy with reassignment}\label{alg:IDpolicy}
\begin{algorithmic}[1]
\Require \(N\)-armed WCMDP instance and LP relaxation~\eqref{eq:lprelax}
\Ensure actions \((A_{i,t})_{i\in[N]}\) at each time \(t\)
\State Solve the LP relaxation~\eqref{eq:lprelax} and construct \(\{\pi_i^\star\}_{i\in[N]}\).
\State Reassign IDs using Algorithm~\ref{alg:id-reassignment}.
\For{\(t=0,1,2,\ldots\)}
    \State Sample ideal actions \(\Aideal_{i,t}\sim\pi_i^\star(\cdot\mid S_{i,t})\) for all \(i\in[N]\).
    \State Compute the conforming number \(N_t^\star\) from \eqref{eq:conforming-number}.
    \State Set \(A_{i,t}=\Aideal_{i,t}\) for \(i\le N_t^\star\), and \(A_{i,t}=0\) for \(i>N_t^\star\).
\EndFor
\end{algorithmic}
\end{algorithm}
\begin{algorithm}[htbp]
\caption{ID reassignment}\label{alg:id-reassignment}
\begin{algorithmic}[1]
\Require optimal state--action frequencies \(\bigl(y_i^\star(s,a)\bigr)_{i\in[N],\,s\in\States,\,a\in\Actions}\), budgets \((\alpha_k)_{k\in[K]}\)
\Ensure new arm IDs recorded in \(\mathrm{newID}(i)\) for each \(i\in[N]\)
\State Compute \(C_{k,i}^\star\) for all \(i\in[N]\), \(k\in[K]\), and the active set $\Lambda$.
\If{\(\Lambda=\emptyset\)}
    \State Set \(\mathrm{newID}(i)\gets i\) for all \(i\in[N]\). \Comment{No reassignment needed}
\Else
    \State Set \(F\gets\emptyset\). \Comment{Arms already assigned new IDs}
    \State Set \(\alpha_{\min}\gets \min_{k\in[K]}\alpha_k, \delta\gets \alpha_{\min}/4, d\gets \left\lceil \frac{(c_{\max}-\delta)K}{\alpha_{\min}/2-\delta} \right\rceil.\)
    \State For each \(k\in\Lambda\), set \(D_k\gets \left\{ i\in[N]: C_{k,i}^\star\ge \delta \right\}.\)
    \For{\(\ell=0,1,\ldots,\lfloor N/d\rfloor-1\)}
        \State Set \(\mathcal I(\ell)\gets [\ell d+1:(\ell+1)d]\) and \(j\gets \ell d+1\).
        \ForEach{\(k\in\Lambda\)}
            \If{\(\sum_{i\in F} C_{k,i}^\star \mathbf 1\{\mathrm{newID}(i)\in\mathcal I(\ell)\}<\delta\)}
                \State Pick any \(i\in D_k\) and set \(\mathrm{newID}(i)\gets j\).
                \State Remove \(i\) from every \(D_{k'}\), \(k'\in\Lambda\); set \(F\gets F\cup\{i\}\) and \(j\gets j+1\).
            \EndIf
        \EndForEach
    \EndFor
    \State For all \(i\in[N]\setminus F\), assign \(\mathrm{newID}(i)\) arbitrarily from \([N]\setminus\{\mathrm{newID}(i'):i'\in F\}\).
\EndIf
\end{algorithmic}
\end{algorithm}
\paragraph{ID reassignment.}
In the fully heterogeneous setting, arms can have very different cost profiles under $\pi_i^\star$.
Define the expected type-$k$ cost of arm $i$ under $\pistar_i$ as
$
    C_{k,i}^\star \triangleq \sum_{s\in\mathcal{S},\,a\in\mathcal{A}} y_i^\star(s,a)\, c_{k,i}(s,a).
$
We call constraint $k$ \emph{active} if $\sum_{i=1}^N C_{k,i}^\star \ge \frac{\alpha_k}{2}N.$ Let $\Lambda\subseteq[K]$ denote the set of active constraints. 

For each subset $D\subseteq[N]$, define the total expected type-$k$ cost as
$C_k^\star(D)\triangleq \sum_{i\in D} C_{k,i}^\star$, and define the \emph{remaining budget} as 
$ 
\bar{C}_k^\star(D) =
\alpha_k N - C_k^\star(D) - \frac{\alpha_k}{3}|D| \cdot \mathbb{I}(k\notin\Lambda).
$
The reassignment algorithm is shown in Algorithm~\ref{alg:id-reassignment}. Algorithm~\ref{alg:id-reassignment} permutes the arm IDs so that, for every active constraint, the cumulative expected cost $C_k^\star([n])$ cannot remain nearly unchanged over a long range of $n$.
This regularity prevents the threshold $N_t^\star$ from becoming overly sensitive to the randomness in the sampled ideal actions.

\subsection{Analysis of the ID policy}
\label{subsec:id-analysis}
The analysis of the ID policy in~\citet{zhang2025projection}
uses a specific Lyapunov function. Throughout this subsection,
\(\bmx\) denotes a realization of the one-hot system state
\(\bm X_t\) defined in Section~\ref{sec:proof-lem-rhorel-rhohpiid-hete}. We first introduce the focus set
used in this Lyapunov function.

\paragraph{Focus set.}
The ID policy allows only a random prefix of \(N_t^\star\) arms to
follow their ideal actions. To analyze this random set, we define a
deterministic prefix based on the current system state, called the
\emph{focus set}. For a system state \(\bmx\) and
\(m\in[0,1]_N\triangleq\{0,1/N,\ldots,1\}\), let
\(h(\bmx,m)\) be the quantity defined in~\eqref{def:h}. Define
\begin{equation}\label{eq:def-m}
    m(\bmx)
    \triangleq
    \max\left\{
        m\in[0,1]_N:
        h(\bmx,m)
        \le
        \frac{1}{N}\min_{k\in[K]}
        \bar C_k^\star([Nm])
    \right\}.
\end{equation}
and define the \emph{focus set} as the prefix $[Nm(\bmx)]$.
Intuitively, \(m(\bmx)\) is the largest fraction of arms for which
\(h(\bmx,m)\) does not exceed the remaining budget divided by \(N\) under any constraint.

Their analysis establishes three key properties:
(i) \emph{majority conformity} \citep[Lemma~8]{zhang2025projection}---
all but \(O(\sqrt{N})\) arms in the focus set can follow their optimal
single-armed policies \(\pistar_i\), in expectation;
(ii) \emph{almost non-shrinking} \citep[Lemma~9]{zhang2025projection}---
the focus set shrinks by at most $O(\sqrt{N})$ on average over time;
and
(iii) \emph{sufficient coverage} \citep[Lemma~10]{zhang2025projection}---
the fraction of arms outside the focus set satisfies
\(
    1-m(\bm X_t)
    \lesssim
    h\bigl(\bm X_t,m(\bm X_t)\bigr)
    +1/N,
\)
thus, when \(h\bigl(\bm X_t,m(\bm X_t)\bigr)\) is small, the focus set
contains almost all arms. 

\paragraph{Projection-based Lyapunov function.}
We firstly introduce some necessary single-armed quantities.
For each cost type $k\in[K]$, define
$c^\star_{k,i}(s) \triangleq \sum_{a\in\Actions}\pi^\star_i(a\mid s)\,c_{k,i}(s,a)$ and the vector
$c^\star_k \triangleq (c^\star_{k,i})_{i\in[N]}$.
Similarly, define $r^\star_i(s) \triangleq \sum_{a\in\Actions}\pi^\star_i(a\mid s)\,r_i(s,a)$ and
$r^\star \triangleq (r^\star_i)_{i\in[N]}$.
We collect these vectors as $\mathcal{G}\triangleq\{c^\star_1,\dots,c^\star_K,r^\star\}$.
With Assumption \ref{ass:unichain-unifmixing}, the Markov chain induced by \( \pistar_i \) converges to a unique stationary distribution, denoted by \( \mu_i^\star = (\mu_i^\star(s))_{s \in \mathcal{S}} \). It is straightforward to verify that \( \mu_i^\star(s) = \sum_{a\in\Actions} y_i^\star(s,a) \). We scale the subset Lyapunov function introduced by \cite{zhang2025projection}, for any system state $\bmx$ and subset $D\subseteq [N]$,
\begin{equation*}
    h_{\bp}(\bmx,D)
    \triangleq \max_{g\in\mathcal{G}} \sup_{\ell\in\mathbb{N}}
    \left|
    \frac{1}{N}\sum_{i\in D}
    \left\langle
    (\bmx_i - \mu_i^\star)(P_i^{\pistar_i})^{\ell}/\gamma^{\ell},
    g_i
    \right\rangle
    \right|,
\end{equation*}
where 
\begin{equation}\label{def:gamma-tau-g}
    \gamma = 2^{-\frac{1}{2\tau}},\quad\tau=(2+\log_2S)\tmix.
\end{equation}
Here $\tmix$ is the uniform mixing time defined in Assumption~\ref{ass:unichain-unifmixing}. Then for any $m\in[0,1]_N$, we define
\begin{equation}\label{def:h}
h(\bmx,m)\triangleq\max_{m'\in[0,1]_N: m'\le m}h_{\bp}\bigl(\bmx,[N m']\bigr),
\end{equation}
Finally the Lyapunov function $V$ is defined as
\begin{equation}\label{def:lya-true}
     V(\bmx)\triangleq h(\bmx,m(\bmx))+L_h(1-m(\bmx)),
\end{equation}
where $m(\bmx)$ is defined in \eqref{eq:def-m} and 
\begin{equation}\label{eq:def-Lh-Ctau-tau}
    L_h = 2\max\{c_{\max},r_{\max}\} C_\tau,\quad
C_{\tau} = 32\tau.
\end{equation}
Moreover, by the \(L_h\)-Lipschitz property of \(h(\cdot,m(\cdot))\) in
\citet[Lemma~7]{zhang2025projection}, for any system state $\bmx$, we have
\begin{equation}\label{eq:bound-V-hete}
    V(\bmx)
    =
    h(\bmx,m(\bmx))+L_h(1-m(\bmx))  
    \le
    h(\bmx,0)+L_hm(\bmx)+L_h(1-m(\bmx))
    =
    L_h .
\end{equation}

By analyzing the drift of Lyapunov function $V$, \cite{zhang2025projection} establish an $O(1/\sqrt{N})$ optimality gap for the ID policy. 

\subsection{Empirical ID policy}
\label{subsec:empirical-id-policy}
In this subsection, we introduce the ID policy for the empirical system. 
The empirical LP is obtained from~\eqref{eq:lprelax} by
replacing \(P_i\) with \(\Phat_i\) in the flow-balance constraints. Let
\(\rhohat^{\mathrm{rel}}\) be its optimal value, and let
\(\bigl(\yhst_i(s,a)\bigr)_{i\in[N],\,s\in\States,\,a\in\Actions}
\)
be an optimal solution. The empirical optimal single-arm policies
\(\{\pihstar_i\}_{i\in[N]}\) are defined from \(\yhst_i\) in the same way
that \(\{\pistar_i\}_{i\in[N]}\) are defined from \(y_i^\star\). Define the cost and reward features induced by the empirical single-arm policies as
\(
    \widehat c_{k,i}^\star(s)
    \triangleq
    \sum_{a\in\Actions}\pihstar_i(a\mid s)c_{k,i}(s,a),
    \widehat r_i^\star(s)
    \triangleq
    \sum_{a\in\Actions}\pihstar_i(a\mid s)r_i(s,a).
\)Writing
\(
    \widehat c_k^\star
    \triangleq
    (\widehat c_{k,i}^\star)_{i\in[N]}
\)
and
\(
    \widehat r^\star
    \triangleq
    (\widehat r_i^\star)_{i\in[N]},
\)
we define
\begin{equation}
\label{eq:Ghat-def}
    \widehat{\mathcal G}
    \triangleq
    \{\widehat c_1^\star,\ldots,\widehat c_K^\star,\widehat r^\star\}.
\end{equation}
Plugging
\(\yhst_i\) and \(\pihstar_i\) into the ID-policy construction gives the
empirical ID policy \(\pihid\). For any policy \(\Pi\), we write
\(\rhohat^\Pi\in\mathbb R^{S^N}\) for its gain under the empirical
\(N\)-armed system.

\section{Proof details of the lemmas used in Theorem~\ref{thm-heteWCMDP}}
\label{app:full-proof-thm-heteWCMDP}

\subsection{Proofs of model accuracy lemmas}\label{subsec:model-accuracy}
\begin{proof}{Proof of Lemma~\ref{lem:samplesize-singlearm_accu}.}
Fix $i\in[N]$, let $e_{sa}\in\mathbb R^{SA}$ be the row-selector 
vector whose $(s,a)$-th entry is $1$ and others $0$. Let $P_i,\widehat P_i\in\mathbb{R}^{SA\times S}$ be matrices with $P_i(s,a,s')=P_i(s'\mid s,a)$ and $ \widehat P_i(s,a,s')=\widehat P_i(s'\mid s,a)$. 
Then by definition,
    \[
        \max_{s \in \mathcal{S}, a \in \mathcal{A}} \onenorm{P_i(\cdot\mid s,a) - \widehat{P}_i(\cdot\mid s,a)}  
        = \max_{s \in \mathcal{S}} \max_{a \in \mathcal{A}} \max_{v \in \{-1,1\}^{S}}e_{sa}^\top \left( P_i - \widehat{P}_i\right)v .
    \]
    Now fixing some $s \in \mathcal{S}, a \in \mathcal{A}, v \in \{-1,1\}^{S}$, we have by Hoeffding's inequality that
    \[
        \P \left(e_{sa}^\top \left( P_i - \widehat{P}_i\right)v \geq \sqrt{\frac{2\log (1/\eta')}{n}} \right) 
        \leq 
        \exp \left(-\frac{2n}{4} \sqrt{\frac{2\log (1/\eta')}{n}}^2 \right) = \eta'.
    \]
    Setting $\eta' = \frac{\eta}{NSA 2^S}$ and taking a union bound over all $s \in \mathcal{S}, a \in \mathcal{A}, v \in \{-1,1\}^{{S}}, i\in[N]$, we obtain the bound in Lemma~\ref{lem:samplesize-singlearm_accu}. \Halmos
\end{proof}

\proof{Proof of Lemma~\ref{lem:lift-single-to-Narm}.}
    By Lemma \ref{tech_1}, for any fixed $(\bms,\bma)\in \States^N\times\Actions^N$, we have
\begin{align*}
    & \onenorm{\bp(\cdot\mid \bms,\bma) - \bphat(\cdot\mid \bms,\bma)} \\
     =& \sum_{s_1',\dots,s_N'}\left|\prod_{i=1}^N P_i(s_i'\mid s_i,a_i)-\prod_{i=1}^N\widehat{P}_i(s_i'\mid s_i,a_i)\right| \\
    =& \sum_{s_1',\dots,s_N'}\left|\sum_{j=1}^N\Big(\prod_{i<j}\widehat{P}_i(s_i'\mid s_i,a_i)\Big)\Big(P_j(s_j'\mid s_j,a_j)-\widehat{P}_j(s_j'\mid s_j,a_j)\Big)\Big(\prod_{i>j}P_i(s_i'\mid s_i,a_i)\Big)\right|\\
    \le & \sum_{s_1',\dots,s_N'}\sum_{j=1}^N\prod_{i<j}\widehat{P}_i(s_i'\mid s_i,a_i)\left|P_j(s_j'\mid s_j,a_j)-\widehat{P}_j(s_j'\mid s_j,a_j)\right|\prod_{i>j}P_i(s_i'\mid s_i,a_i)\\
    =&\sum_{j=1}^N\underbrace{\left(\sum_{s_1',\dots,s_{j-1}'}\prod_{i<j}\widehat{P}_i(s_i'\mid s_i,a_i)\right)}_{=1} \times \left(\sum_{s_j'}\left|P_j(s_j'\mid s_j,a_j)-\widehat{P}_j(s_j'\mid s_j,a_j)\right|\right)  \times \underbrace{\left(\sum_{s_{j+1}',\dots,s_{N}'}\prod_{i>j}P_i(s_i'\mid s_i,a_i)\right)}_{=1}\\
     =&\sum_{j=1}^N\onenorm{P_j(\cdot\mid s_j,a_j)-\widehat{P}_j(\cdot\mid s_j,a_j)}\le N\delta.
\end{align*}
It follows that
$
\max_{\bms\in\States^N,\bma\in\Actions^N}\onenorm{\bp(\cdot\mid \bms,\bma) - \bphat(\cdot\mid \bms,\bma)} \le N\delta
$
as desired. And accordingly,
\begin{align*}
     \infinfnorm{\bp^{\Pi}-\bphat^{\Pi}}
    =& \max_{\bms}\onenorm{\bp^{\Pi}(\bms,\cdot)-\bphat^{\Pi}(\bms,\cdot)}
    \le \max_{\bms}\sum_{\bma}\Pi(\bma\mid\bms)\onenorm{\bp(\cdot\mid\bms,\bma)-\bphat(\cdot\mid\bms,\bma)}
    \le N\delta. \Halmos
\end{align*} 
\endproof
Finally, we state the following standard identity, which is used in the proof of Lemma~\ref{lem:lift-single-to-Narm}.
\begin{lemma}\label{tech_1}
For any real numbers $a_i, b_i \in \mathbb{R}$ for $i = 1, \dots, N$, the following identity holds:
$
\prod_{i=1}^N a_i - \prod_{i=1}^N b_i
=
\sum_{j=1}^N
\left( \prod_{i<j} b_i \right)(a_j - b_j)\left( \prod_{i>j} a_i \right).
$
\end{lemma}

\begin{proof}{Proof.}
We proceed by induction on $N$. The base case $N = 2$ follows from:
$
a_1 a_2 - b_1 b_2 = (a_1 - b_1)a_2 + b_1(a_2 - b_2).
$
Assume the identity holds for $N - 1$ terms. Then,
\[
\prod_{i=1}^N a_i - \prod_{i=1}^N b_i
= \left( \prod_{i=1}^{N-1} a_i \right) a_N - \left( \prod_{i=1}^{N-1} b_i \right) b_N
= \left( \prod_{i=1}^{N-1} a_i - \prod_{i=1}^{N-1} b_i \right) a_N
+ \left( \prod_{i=1}^{N-1} b_i \right)(a_N - b_N).
\]
Applying the inductive hypothesis to the first RHS term completes the proof.
\Halmos
\end{proof}

\subsection{Proof of Lemma~\ref{lem:rhohrel-rhopihid-hete}}\label{subsec:pr-C2-hete}
Lemma~\ref{lem:rhohrel-rhopihid-hete}  bounds $\rhohrel-\rho^{\pihid}(\bms_0)$, which is the second term in the decomposition~\eqref{eq:decomp}. Hence, we work with the \emph{empirical} system. The proof mirrors that of
Lemma~\ref{lem:rhorel-rhohpiid-hete}, but with the empirical LP solution
and the empirical kernel. We first verify that the empirical single-arm
chains needed to define the empirical Lyapunov function are well behaved, as stated in the following Lemma~\ref{lem:unichain}.

\subsubsection{Markov chain under empirical model}
\label{subsec:MCs-pihbst}

\begin{lemma}
\label{lem:unichain}
\label{lem:empirical-single-arm-chain}
Under the conditions of Lemma~\ref{lem:rhohrel-rhopihid-hete}, for every \(i\in[N]\), the Markov
chain induced by \(\pihstar_i\) under \(\Phat_i\) is unichain and aperiodic,
with mixing time at most \(\tau\), where $\tau$ is defined in~\eqref{def:gamma-tau-g}. 
\end{lemma}

\begin{proof}{Proof.}
Fix \(i\in[N]\). Let $Q_i$ and $\widehat{Q}_i$ denote the transition matrices induced by \(\pihstar_i\) in the true model $P_i$ and the empirical model $\widehat{P}_i$, respectively. That is, $\forall s,s'\in\mathcal{S},$
\(
    Q_i(s,s')
    =
    \sum_{a\in\Actions}\pihstar_i(a\mid s)P_i(s'\mid s,a),
    \widehat Q_i(s,s')
    =
    \sum_{a\in\Actions}\pihstar_i(a\mid s)\Phat_i(s'\mid s,a).
\)
Although \(\pihstar_i\) is computed from the empirical model, it is still a
stationary policy for the true single-arm MDP. Hence
Assumption~\ref{ass:unichain-unifmixing} implies that \(Q_i\) is unichain,
aperiodic, and has mixing time at most \(\tmix\). Moreover, by convexity of
the \(\ell_1\) norm,
\[
    \infinfnorm{\widehat Q_i-Q_i}
    \le
    \max_{s\in\States,a\in\Actions}
    \onenorm{\Phat_i(\cdot\mid s,a)-P_i(\cdot\mid s,a)}
    \le
    \delta\le
    \frac{1}{4S\tmix(5+\log_2 S)} .
\]
Applying Lemma~\ref{lem:mixing_time_pert_bound} with
\(\mathcal Z=\States\), \(Q=Q_i\), and
\(\widetilde Q=\widehat Q_i\), gives the desired conclusion. \Halmos
\end{proof}

\subsubsection{Lyapunov function for the empirical system}\label{subsec:lya-hatsys}
In this subsection we introduce the Lyapunov function for the empirical system. We firstly introduce the following subset Lyapunov function:
$h_{\bphat}(\bmx,D)
    = \max_{g\in\widehat{\mathcal G}} \sup_{\ell\in\mathbb{N}}
    \left|\frac{1}{N}\sum_{i\in D}\left\langle(\bmx_i - \widehat\mu_i^\star)(\widehat{P}_i^{\pihstar_i})^{\ell}/\gamma^{\ell},g_i\right\rangle\right|,$
where \(\gamma\) is the same as \eqref{def:gamma-tau-g} and $\widehat{\mathcal G}$ is defined in~\eqref{eq:Ghat-def}. Here \(\muhst_i\) denotes the stationary distribution of the Markov chain
induced by \(\pihstar_i\) under \(\Phat_i\), whose existence and uniqueness
are guaranteed by Lemma~\ref{lem:unichain}. 
For the empirical system, define the expected type-\(k\) cost of arm
\(i\) as
\(\widehat C_{k,i}^\star
\triangleq
\sum_{s\in\States}\sum_{a\in\Actions}
\widehat y_i^\star(s,a)c_{k,i}(s,a)\).
For each subset \(D\subseteq[N]\), define
\(\widehat C_k^\star(D)
\triangleq
\sum_{i\in D}\widehat C_{k,i}^\star\).
As in the true system, we call constraint \(k\) empirically active if
\(\sum_{i=1}^N\widehat C_{k,i}^\star
\ge \frac{\alpha_k}{2}N\).
Let \(\widehat\Lambda\subseteq[K]\) denote the set of empirically active
constraints. We define the empirical remaining budget as
\(\widehat{\overline C}_k^\star(D)
\triangleq
\alpha_kN-\widehat C_k^\star(D)
-\frac{\alpha_k}{3}|D|\mathbb I\{k\notin\widehat\Lambda\}\).
We then define
\(\widehat{h}(\bmx,m)\triangleq\max_{m'\in[0,1]_N: m'\le m}h_{\bphat}\bigl(\bmx,[N m']\bigr)\) and \(\widehat m(\bmx)
\triangleq
\max\{m\in[0,1]_N:
\widehat h(\bmx,m)
\le
\frac{1}{N}\min_{k\in[K]}
\widehat{\overline C}_k^\star([Nm])\}\).
Finally, the Lyapunov function for the empirical system is
\begin{equation}\label{def:vhat-hete}
    \widehat V(\bmx)
    \triangleq
    \widehat h\bigl(\bmx,\widehat m(\bmx)\bigr)
    +L_h\bigl(1-\widehat m(\bmx)\bigr),
\end{equation}
where \(L_h\) is defined in~\eqref{eq:def-Lh-Ctau-tau}.
The focus set of \(\pihid\) is
\([N\widehat m(\bmx)]\). 

\subsubsection{Completing the proof of Lemma~\ref{lem:rhohrel-rhopihid-hete}}\label{subsec:complete-pf-2part-hete}

\begin{proof}{Proof.}
The next two lemmas state the two key conditions needed for the Lyapunov framework: Lemma~\ref{lem:DC-Vhat-hete} provides the drift bound (Condition~\ref{cond:drift}) and Lemma~\ref{lem:simu-lemma-hat-hete} provides the gap-dominance bound (Condition~\ref{cond:dominance}), both for the \emph{empirical} quantities and the empirical policy $\pihid$.
Their proofs are identical to those for Lemmas~\ref{lem:lemma5-zhang2025} and~\ref{lem:simu-lemma-hete}, just replacing the true model $P$ by the empirical model $\widehat{P}$; we omit the details.

\begin{lemma}\label{lem:DC-Vhat-hete}
    Under the assumptions of Lemma~\ref{lem:rhohrel-rhopihid-hete}, the Lyapunov function $\Vhat(\cdot)$ defined in~\eqref{def:vhat-hete} satisfies
    \begin{equation}\label{drift_cond_hat_hetero}
        \widehat{\E}^{\pihid}\!\left[\Vhat(\bm{X}_{t+1})-\Vhat(\bm{X}_t)\mid \bm{X}_t=\bmx\right]
        \le -\rho_V\,\Vhat(\bmx)+K_V/\sqrt{N},
    \end{equation}
    where the constants $\rho_V$ and $K_V$ are the same as in~\eqref{eq:cons-hetero-drift}.
\end{lemma}
\begin{remark}
The constants in~\eqref{drift_cond_hat_hetero} are the same as those in
the true-system drift bound in Lemma~\ref{lem:lemma5-zhang2025}, since the
constants in~\eqref{eq:cons-hetero-drift} do not depend on the transition
kernels.
\end{remark}
\begin{lemma}\label{lem:simu-lemma-hat-hete}
    Suppose the assumptions of Lemma~\ref{lem:rhohrel-rhopihid-hete} hold. Fix a system state $\bms$, and let $A_i^{\pihid}$ be the action applied to arm $i$ under $\pihid$. Define
    $
    r^{\pihid}(\bms)=\frac{1}{N}\sum_{i=1}^N\E\!\left[r_i(s_i,A_i^{\pihid})\right].
    $
    Then
    $
        \rhohrel-r^{\pihid}(\bms)
        \le \frac{2r_{\max}+L_h}{L_h}\,\Vhat(\bmx)+\frac{2r_{\max}K_{\mathrm{conf}}}{\sqrt{N}}.
    $
\end{lemma}

\paragraph{Back to the proof of Lemma~\ref{lem:rhohrel-rhopihid-hete}.}
We apply Lemma~\ref{lem:generic-drift-transfer} with $Q=\bphat^{\pihid}$, $P=\bp^{\pihid}$, $R=\rhohrel$, $F=\Vhat$, and $\bar r=r^{\pihid}$. Lemmas~\ref{lem:DC-Vhat-hete} and~\ref{lem:simu-lemma-hat-hete} give the same constants as in Lemma~\ref{lem:rhorel-rhohpiid-hete}:
\(
    a=\rho_V,
    b=K_V/\sqrt N,
    c=\frac{2r_{\max}+L_h}{L_h},
    d=\frac{2r_{\max}K_{\mathrm{conf}}}{\sqrt N}.
\)
Therefore,
\[
    \rhohrel-\rho^{\pihid}(\bms_0)
    \le
    \frac{2r_{\max}+L_h}{L_h\rho_V}\sup_{\bmx}|(\bp^{\pihid}-\bphat^{\pihid})\Vhat(\bmx)|
    +\left[\frac{(2r_{\max}+L_h)K_V}{L_h\rho_V}+2r_{\max}K_{\mathrm{conf}}\right]\frac1{\sqrt N}.
\]
By Lemma~\ref{lem:lift-single-to-Narm}, $\infinfnorm{\bp^{\pihid}-\bphat^{\pihid}}\le N\delta$. The empirical Lyapunov function satisfies $\|\Vhat\|_\infty\le L_h$ by the same Lipschitz argument used for $V$. Thus
\(
    \sup_{\bmx}|(\bp^{\pihid}-\bphat^{\pihid})\Vhat(\bmx)|
    \le L_hN\delta,
\)
which gives the desired bound. \Halmos
\end{proof}

\section{Proof of lemmas  for Theorem~\ref{thm:LP_perturbation_main_complete}}\label{sec:proof-PA}
\subsection{Proof of Lemma~\ref{lem:two-condtions}}\label{subsec:proof-two-conditions}
\proof{Proof.}
Firstly, using that $\infinfnorm{\Phat - P } \leq \delta$ by assumption, we have
\begin{equation}\label{eq:gmhat-gm-infnorm}
    \begin{gathered}
        \infinfnorm{\xihat - \xi} \leq 2 \infinfnorm{\Phat - P} \leq 2 \delta, \\
        \begin{aligned}
        \infinfnorm{\gmhat - \Gamma} &= \infinfnorm{ P^{\pibst} - (c^{\pibst} -\alpha \one)\xi- \Phat^{\pibst} + (c^{\pibst} -\alpha \one)\widehat \xi} \\
    & \leq \infinfnorm{P^{\pibst} - \Phat^{\pibst}} + \infnorm{c^{\pibst} -\alpha \one} \infinfnorm{\xi - \widehat \xi} 
    \leq \delta + 1\cdot 2 \delta 
    = 3\delta.
    \end{aligned}
    \end{gathered}
\end{equation}
The second term in~\eqref{eq:def-delta-min} gives
$\infinfnorm{\gmhat-\Gamma}\le3\delta
\le\min_{s,a:y^\star(s,a)>0}y^\star(s,a)/(2\infinfnorm{H_{\Gamma}})$.\
Since $y^\star(\tilde{s},1)>0, y^\star(\tilde{s},0)>0$ and $\sum_{s,a}y^\star(s,a)=1$, the smallest positive entry of $y^\star$ is at most
$1/2$, and
$\infinfnorm{\gmhat-\Gamma}\infinfnorm{H_{\Gamma}}\le1/4$.
This proves~\eqref{eq:well-define-Huhat}. \Halmos
\endproof
\subsection{Algebraic properties of 
\texorpdfstring{$\Gamma$, $\gmhat$, $\mu^\star$, $\muhat$, $H_{\Gamma}$}
{Gamma, Gamma-hat, mu-star, mu-hat, H_Gamma}}
\label{subsec:properties}
The following two lemmas that check various algebraic properties of $\Gamma, \gmhat, \mu^\star, \muhat, H_{\Gamma}$, which will be used in the sequel.

\begin{lemma}
\label{lem:basic-U-HU-identities}
We have \(\Gamma\one=\one\) and \(\gmhat\one=\one\). Moreover, if
\(\rho(\Phi)<1\), then \(H_{\Gamma}\one=\zero\).
\end{lemma}
\begin{proof}{Proof.}
We firstly prove \(\Gamma\one=\one\); the proof of \(\gmhat\one=\one\) is identical.
Since \(P^{\pibst}\) is row-stochastic and
\(
    \xi\one
    =
    (P_{\tilde s1}-P_{\tilde s0})\one
    =
    1-1
    =
    0,
\)
we have
\(
    \Gamma\one
    =
    \bigl(P^{\pibst}-(c^{\pibst}-\alpha\one)\xi\bigr)\one
    =
    P^{\pibst}\one
    -
    (c^{\pibst}-\alpha\one)(\xi\one)
    =
    \one .
\)
Let $x = H_{\Gamma} \one = (I - (\Gamma - \one \mu^\star))^{-1} \one - \one\mu^\star \one = (I - (\Gamma - \one \mu^\star))^{-1} \one - \one$ (using that $\mu^\star \one = 1$). Then left-multiplying by $(I - (\Gamma - \one \mu^\star))$, we have
    \begin{align*}
        (I - (\Gamma - \one \mu^\star)) x &= (I - (\Gamma - \one \mu^\star))(I - (\Gamma - \one \mu^\star))^{-1} \one - (I - (\Gamma - \one \mu^\star)) \one 
        = \one - \one + \Gamma \one - \one \mu^\star \one 
        = \zero,
    \end{align*}
    where we use $\mu^\star \one = 1$ and $\Gamma \one = \one$ in the final equality. But since \(\rho(\Phi)<1\), \(I-(\Gamma-\one\mu^\star)\) is invertible and the equality $(I - (\Gamma - \one \mu^\star)) x = \zero$ implies that $x = \zero$. \Halmos
\end{proof}

\begin{lemma}
\label{lem:perturbed-Uhat-Huhat-identities}
Suppose \(\rho(\Phi)<1\) and
\(I-(\gmhat-\Gamma)H_{\Gamma}\) is invertible. Then
\(
    \muhat\gmhat=\muhat,
    \muhat\one=1.
\)
\end{lemma}
\begin{proof}{Proof.}
    Let $X = (I - (\gmhat - \Gamma)H_{\Gamma})^{-1}$. By the definition~\eqref{eq:muhat-definition} of $\muhat$, to show $\muhat \gmhat = \muhat$ it suffices to show that
        $
            \mu^\star X \gmhat = \mu^\star X.
        $
        From the definition of $X$ we have that
        \begin{align*}
            X (I - (\gmhat - \Gamma)H_{\Gamma}) = I 
            \implies &  X + X\Gamma H_{\Gamma} - I = X \gmhat H_{\Gamma} \\
            \implies & (X + X\Gamma H_{\Gamma} - I)(I-\Gamma) = X \gmhat H_{\Gamma}(I-\Gamma) = X\gmhat (I - \one \mu^\star) \\
            \implies & X\gmhat = (X + X\Gamma H_{\Gamma} - I)(I-\Gamma) + X\gmhat \one \mu^\star \\
            & \quad\;\;\, = X - X\Gamma +X\Gamma H_{\Gamma} - X\Gamma H_{\Gamma} \Gamma - I + \Gamma + X\gmhat \one \mu^\star,
        \end{align*}
        where we use the property that $H_{\Gamma}(I-\Gamma) = I - \one \mu^\star$.
        Now left-multiplying by $\mu^\star$, and also using the facts that $\gmhat\one = \Gamma\one = \one$ from Lemma \ref{lem:basic-U-HU-identities}, we have that
        \begin{align*}
            \mu^\star X \gmhat 
            &= \mu^\star \left( X - X\Gamma +X\Gamma H_{\Gamma} - X\Gamma H_{\Gamma} \Gamma - I + \Gamma + X\gmhat \one \mu^\star\right) \\
            &= \mu^\star X + \mu^\star X\Gamma  (-I + H_{\Gamma} -  H_{\Gamma}\Gamma  + \one \mu^\star) 
            = \mu^\star X,
        \end{align*}
        where in the final equality we use that $-I + H_{\Gamma} -  H_{\Gamma}\Gamma  + \one \mu^\star = 0$, which is equivalent to the fact that $H_{\Gamma}(I - \Gamma) = I - \one \mu^\star$.
        Hence we have that $\muhat \gmhat = \muhat$ as desired. $\muhat\one=1$ directly follows \eqref{eq:equaldef-muhat} and $\mu^\star\one=1$. \Halmos
\end{proof}

\subsection{Proof of Lemma~\ref{lem:well-define-Huhat}}\label{sec:pr-well-define-Huhat}
\begin{proof}{Proof of Lemma~\ref{lem:well-define-Huhat}.}
To prove that $I-(\gmhat-\one\muhat)$ is invertible, suppose a row
vector $\nu$ satisfies $\nu(I-\gmhat+\one\muhat)=0$.
Since $\gmhat\one=\one$ and $\muhat\one=1$, multiplying $\nu(I-\gmhat+\one\muhat)=0$ by $\one$ gives $\nu\one=0$. Therefore
$\nu\one\muhat=0$, so $\nu(I-\gmhat)=0$.
The identities $\mu^\star\Gamma=\mu^\star$ and $\Gamma\one=\one$
give $(I-\Gamma)H_{\Gamma}=I-\one\mu^\star$ directly from the
definition of $H_{\Gamma}$. Since $\nu\mathbf1=0$, it follows that
\(
\nu(I-\Gamma)H_\Gamma
=\nu(I-\mathbf1\mu^\star)
=\nu.
\)
On the other hand, $\nu(I-\widehat\Gamma)=0$ implies
\(
\nu(I-\Gamma)
=\nu(I-\widehat\Gamma)+\nu(\widehat\Gamma-\Gamma)
=\nu(\widehat\Gamma-\Gamma).
\)
Combining these two identities gives
$\nu=\nu(\widehat\Gamma-\Gamma)H_\Gamma$, or equivalently,
$\nu[I-(\widehat\Gamma-\Gamma)H_\Gamma]=0$. We have already established that
$I-(\widehat\Gamma-\Gamma)H_\Gamma$ is invertible.
Right-multiplying by its inverse yields $\nu=0$.
Thus $I-\widehat\Gamma+\mathbf1\widehat\mu$ is invertible.

Lemma~\ref{lem:perturbed-Uhat-Huhat-identities} gives
$\muhat\one=1$, so to prove that $\muhat$ is a probability vector, we only need to show \(\muhat(s)\ge 0\) for every \(s\in\States\). Notice that due to \eqref{eq:well-define-Huhat} and $\onenorm{\mu^\star}=1$,
\[
\onenorm{\muhat}\le\onenorm{\mu^\star}\infinfnorm{(I-(\gmhat-\Gamma)H_{\Gamma})^{-1}}\le\frac{1}{1-\infinfnorm{\gmhat-\Gamma}\infinfnorm{H_\Gamma}}\le\frac{4}{3}.
\]
Combining $\onenorm{\muhat}\le4/3$ with~\eqref{eq:equaldef-muhat} 
gives
\[
\onenorm{\muhat-\mu^\star}
\le\frac43\infinfnorm{\gmhat-\Gamma}\infinfnorm{H_{\Gamma}}
\le\frac23\min_{s,a:y^\star(s,a)>0}y^\star(s,a)\le\frac{2}{3}\mu^\star(s), \quad\forall s\in\States.
\]
Consequently, $\muhat(s)\ge\mu^\star(s)/3>0$ for every $s$, which proves that $\muhat$ is a probability vector. 
\end{proof}
\subsection{Properties of 
\texorpdfstring{$H_{\gmhat}$}
{H_Gamma-hat}}
\label{subsec:properties-Huhat}
\begin{lemma}\label{lem:properties-H-gmhat}
    Suppose that \(\rho(\Phi)<1\) and
\eqref{eq:well-define-Huhat} holds. Then
\(
    H_{\gmhat}\one=\zero,
    \muhat H_{\gmhat}=\zero,
    (I-\gmhat)H_{\gmhat}=I-\one\muhat.
\)
Moreover, there exists a row vector \(x\) such that
\(
    H_{\gmhat}-H_{\Gamma}
    =
    H_{\Gamma}(\gmhat-\Gamma)H_{\gmhat}+\one x.
\)
\end{lemma}
\proof{Proof.} Lemma~\ref{lem:perturbed-Uhat-Huhat-identities}
shows that $\muhat \gmhat = \muhat$ and $\muhat \one = 1$. We can use an analogous proof as for Lemma \ref{lem:basic-U-HU-identities} to show $H_{\gmhat} \one = \zero$. Next we check that $\muhat H_{\gmhat}=\zero$.
The identities $\muhat\gmhat=\muhat$ and $\muhat\one=1$ imply
$\muhat(I-\gmhat+\one\muhat)=\muhat$.
Since $I-\gmhat+\one\muhat$ is invertible by
Lemma~\ref{lem:well-define-Huhat}, right-multiplication by its
inverse gives $\muhat(I-\gmhat+\one\muhat)^{-1}=\muhat$.
Subtracting $\muhat\one\muhat=\muhat$ yields
$\muhat H_{\gmhat}=\zero$.

        Finally we show that $(I - \gmhat)H_{\gmhat} = I - \one \muhat$. We have
        \begin{align*}
        (I - \gmhat)H_{\gmhat} 
        = (I - \gmhat + \one \muhat)H_{\gmhat} - \one \muhat H_{\gmhat}  
        &= (I - \gmhat + \one \muhat) \left((I - (\gmhat - \one \muhat))^{-1} - \one \muhat \right) - \one \muhat H_{\gmhat} \\
        &= I - (I - \gmhat + \one \muhat)\one \muhat  - \one \muhat H_{\gmhat} 
        = I - \one \muhat+ \one \muhat - \one\muhat ,
    \end{align*}
    where the final equality follows from $\gmhat \one = \one $, $\muhat \one = 1$ and $\muhat H_{\gmhat} =\zero$.

    Note the resolvent identity $A^{-1} - B^{-1} = A^{-1}(B-A)B^{-1} = B^{-1} (B-A)A^{-1}$ and the fact that $I - (\gmhat - \one \muhat)$ and $I - (\Gamma - \one \mu^\star)$ are invertible by Lemma \ref{lem:well-define-Huhat} and Assumption~\ref{ass:RB}. We have
\begin{align}
    & (I - (\gmhat - \one \muhat))^{-1}  - (I - (\Gamma - \one \mu^\star))^{-1} \nonumber \\
    &= (I - (\Gamma - \one \mu^\star))^{-1} \left( (I - (\Gamma - \one \mu^\star)) - (I - (\gmhat - \one \muhat)) \right) (I - (\gmhat - \one \muhat))^{-1}\nonumber \\
    &= (I - (\Gamma - \one \mu^\star))^{-1} \left( \gmhat - \Gamma - \one \muhat + \one \mu^\star \right) (I - (\gmhat - \one \muhat))^{-1} \nonumber\\
    &= (I - (\Gamma - \one \mu^\star))^{-1} \left( \gmhat - \Gamma \right) (I - (\gmhat - \one \muhat))^{-1} - \one (\muhat - \mu^\star) (I - (\gmhat - \one \muhat))^{-1}, \label{eq:HU-Huhat_step1}
\end{align}
where in the final equality we use $(I - (\Gamma - \one \mu^\star))^{-1} \one = \one$, which follows from Lemma \ref{lem:basic-U-HU-identities}.
Next we have 
\begin{multline}
    (I - (\Gamma - \one \mu^\star))^{-1} \left( \gmhat - \Gamma \right) (I - (\gmhat - \one \muhat))^{-1} \\
    = H_{\Gamma}  \left( \gmhat - \Gamma \right) (I - (\gmhat - \one \muhat))^{-1} + \one \mu^\star \left( \gmhat - \Gamma \right) (I - (\gmhat - \one \muhat))^{-1} \label{eq:HU-Huhat_step2}
\end{multline}
and
\begin{align}
    H_{\Gamma}  \left( \gmhat - \Gamma \right) (I - (\gmhat - \one \muhat))^{-1} 
    &= H_{\Gamma}  \left( \gmhat - \Gamma \right) H_{\gmhat}  + H_{\Gamma}  \left( \gmhat - \Gamma \right) \one \muhat \nonumber\\
    &= H_{\Gamma}  \left( \gmhat - \Gamma \right) H_{\gmhat}  + H_{\Gamma}   \zero 
    = H_{\Gamma}  \left( \gmhat - \Gamma \right) H_{\gmhat}, \label{eq:HU-Huhat_step3}
\end{align}
where we use  $\Gamma \one = \one $ and $\gmhat \one = \one$ from Lemma \ref{lem:basic-U-HU-identities}.
Combining  these calculations, we have that
\begin{align*}
     H_{\gmhat} - H_{\Gamma} 
    =& (I - (\gmhat - \one \muhat))^{-1} - \one \muhat  - (I - (\Gamma - \one \mu^\star))^{-1} + \one \mu^\star  \\
    \stackrel{(i)}{=}& (I - (\Gamma - \one \mu^\star))^{-1} \left( \gmhat - \Gamma \right) (I - (\gmhat - \one \muhat))^{-1} - \one (\muhat - \mu^\star) (I - (\gmhat - \one \muhat))^{-1}  - \one \muhat + \one \mu^\star  \\
    \stackrel{(ii)}{=}& H_{\Gamma}  \left( \gmhat - \Gamma \right) (I - (\gmhat - \one \muhat))^{-1} + \one \mu^\star \left( \gmhat - \Gamma \right) (I - (\gmhat - \one \muhat))^{-1}  \\
    &- \one (\muhat - \mu^\star) (I - (\gmhat - \one \muhat))^{-1}  - \one \muhat + \one \mu^\star \\
    \stackrel{(iii)}{=}& H_{\Gamma}  \left( \gmhat - \Gamma \right) H_{\gmhat} + \one \mu^\star \left( \gmhat - \Gamma \right) (I - (\gmhat - \one \muhat))^{-1}  - \one (\muhat - \mu^\star) (I - (\gmhat - \one \muhat))^{-1}  - \one \muhat + \one \mu^\star
\end{align*}
where in steps $(i), (ii),$ and $(iii)$ we used equations~\eqref{eq:HU-Huhat_step1},~\eqref{eq:HU-Huhat_step2}, and~\eqref{eq:HU-Huhat_step3}, respectively. Hence the desired equality is true for
$
    x =  \mu^\star \left( \gmhat - \Gamma \right) (I - (\gmhat - \one \muhat))^{-1}  -  (\muhat - \mu^\star) (I - (\gmhat - \one \muhat))^{-1}  -  \muhat +  \mu^\star. \Halmos
$
\endproof

\subsection{Proof of Lemma~\ref{lem:yhat_pert_and_feasibility}}\label{subsec:pr-feasibility-yhat}
Before checking the feasibility of $\yhat$, we first show some useful facts about the matrix $W\in \R^{S\times SA}$. We first write down the entries of \(W\) explicitly. For each \(s\neq\Tilde{s}\),
\begin{align}
    W_{s, s0} = \pibst(0 \mid s), \ W_{s, s1} = \pibst(1 \mid s), \ W_{s, \Tilde{s}0} = (1-\alpha) - \pibst(0\mid s), \ W_{s, \Tilde{s}1} = \alpha - \pibst(1\mid s), \label{eq:W_entries_1}
\end{align}
and all other entries of the row \(W_{s,\cdot}\) are zero. For
\(s=\Tilde{s}\),
\begin{align}
    W_{\Tilde{s}, \Tilde{s}0} = 1-\alpha, \quad W_{\Tilde{s}, \Tilde{s}1} = \alpha, \label{eq:W_entries_2}
\end{align}
and all other entries of the row \(W_{\Tilde{s},\cdot}\) are zero. We have the following lemma.

\begin{lemma}
\label{lem:W_properties}
$W$ satisfies the following:
     $WJ = I$;
     $W \tilde{e} = \alpha \one$;
     $W \one = \one$;
     $W P = \Gamma$;
     $W \Phat = \gmhat$.
\end{lemma}
\begin{proof}{Proof.}
1. (Proof of $WJ=I$.) Let $\mu$ be an arbitrary row vector of size $S$ and let $y=\mu W$. For any non-neutral state $s\neq \tilde s$, we have
$
(yJ)(s)=\sum_{a\in\{0,1\}}y(s,a)=\mu(s)\sum_{a}\pibst(a\mid s)=\mu(s).
$
For the neutral state $\tilde s$, using $\pibst(0\mid s)+\pibst(1\mid s)=1$ for any $s \in \States$, we have
\begin{align*}
    (yJ)(\tilde s)
    =y(\tilde s,0)+y(\tilde s,1)
    &=(1-\alpha)(\mu\mathbf 1)+\alpha(\mu\mathbf 1)
    -\sum_{s\neq \tilde s}\big(\pibst(0\mid s)+\pibst(1\mid s)\big)\mu(s)\\
    &=(\mu\mathbf 1)-\sum_{s\neq \tilde s}\mu(s)=\mu(\tilde s).
\end{align*}
Hence $yJ=\mu$ holds for every $\mu$, that is $\mu(WJ)=\mu$ for all $\mu$,
so $WJ=I$.

2. (Proof of $W\tilde e=\alpha\one$.)
Recall $\tilde e\in\mathbb R^{SA}$ is defined by $\tilde e(s,a)=\mathbb{I} (a=1)$.
Again let $\mu$ be an arbitrary row vector of size $S$ and let $y=\mu W$. Then
\begin{align*}
    y\tilde e
    =\sum_{s\in S} y(s,1)
    &=\sum_{s\neq \tilde s}\mu(s)\pibst(1\mid s)+y(\tilde s,1)
    =\sum_{s\neq \tilde s}\mu(s)\pibst(1\mid s)
    +\alpha(\mu\one)-\sum_{s\neq \tilde s}\mu(s)\pibst(1\mid s)
    =\alpha(\mu\one).
\end{align*}
Equivalently, $\mu(W\tilde e)=\alpha(\mu\one)=\mu(\alpha\one)$ for all $\mu$, so $W\tilde e=\alpha\one$.

3. (Proof of $W\one=\one$.)
Let $\one_S$ denote the all-one vector in $\R^{S}$ and
$\one_{SA}$ the all-one vector in $\R^{SA}$.
By definition of $J$, we have $J \one_S=\one_{SA}$.
Therefore, using $WJ=I$ from the previous part, we have
$
W\one_{SA}=WJ\one_S=I\one_S=\one_S.
$

4. (Proof of $W \Phat= \gmhat$.)
Take an arbitrary $\mu$ and set $y=\mu W$.
Then
$
    y \Phat=\sum_{s\in \States}\sum_{a\in\{0,1\}} y(s,a)\Phat_{sa}.
$
For $s\neq \tilde s$, we have $y(s,a)=\mu(s)\pibst(a\mid s)$, so
$
    \sum_{a\in\{0,1\}} y(s,a)\Phat_{sa} = \mu(s)e_s^\top \Phat^{\pibst}.
$
For the neutral state,
\begin{align*}
    y(\tilde s,0) \Phat_{\tilde{s}0} + y(\tilde s,1) \Phat_{\tilde{s}1} 
    &=\Big((1-\alpha)(\mu\one)-\sum_{s\neq \tilde s}\pibst(0\mid s)\mu(s)\Big)\Phat_{\tilde{s}0} +\Big(\alpha(\mu\one)-\sum_{s\neq \tilde s}\pibst(1\mid s)\mu(s)\Big)\Phat_{\tilde{s}1} \\
    &= \mu(\tilde{s}) e_s^\top \Phat^{\pibst} + \left(\alpha (\mu \one) - \sum_{s \in \States} c^{\pibst}(s)\mu(s)\right) \xihat
\end{align*}
since $\Phat_{\tilde{s}1} = \Phat_{\tilde{s}0} + \xihat$ and $\pibst(0 \mid s) = 1 - c^{\pibst}(s)$ for all $s \in \States$.  Summing over all states we have
\begin{align*}
    y \Phat &= \left(\alpha (\mu \one) - \sum_{s \in \States} c^{\pibst}(s)\mu(s)\right) \xihat + \sum_{s \in \States} \mu(s) e_s^\top \Phat^{\pibst} 
    = \mu \left( \left( \alpha \one - c^{\pibst} \right) \xihat + \Phat^{\pibst}\right) 
    = \mu \gmhat.
\end{align*}
We have thus shown that $\mu W \Phat =  y \Phat = \mu \gmhat$, so we have $W \Phat = \gmhat$.

5. (Proof of $W  P=\Gamma$.)
The proof of this fact is analogous to the previous part after replacing $\Phat, \gmhat, \xihat$ by $P, \Gamma, \xi$, respectively.
This completes the proof. \Halmos
\end{proof}

We now prove Lemma~\ref{lem:yhat_pert_and_feasibility}. 
\begin{proof}{Proof of Lemma~\ref{lem:yhat_pert_and_feasibility}.}
First we check the  bound~\eqref{eq:yhat_pert_bound}. Since $\yhat=\muhat W$ and $y^\star=\mu^\star W$, using \eqref{eq:equaldef-muhat}, we have
\begin{align}
    & \yhat - y^\star = (\muhat - \mu^\star)W =  \muhat \left(\gmhat - \Gamma \right) H_{\Gamma} W \nonumber\\
    \implies & \yhat_i - y_i^\star = (\muhat - \mu^\star)We_i =  \muhat \left(\gmhat - \Gamma \right) H_{\Gamma} We_i \nonumber\\
    \implies & \left| \yhat_i - y_i^\star \right| \leq  \onenorm{\muhat} \infinfnorm{\gmhat - \Gamma} \infinfnorm{H_{\Gamma}} \infnorm{We_i} .\label{eq:yhat_pert_bound_int}
\end{align}
We have $\infnorm{We_i} \leq \max\{\alpha, 1-\alpha\} \leq 1$ for $i = (\Tilde{s}, 0), (\Tilde{s}, 1)$, and $\infnorm{We_i} = 1$ for the remaining $i=(s,a)$ which have $\pibst(a \mid s) = 1$. Also $We_i = \zero$ if $i = (s,a)$ where $s \neq \tilde{s} $ and $\pibst(a \mid s) = 0$. In all cases we have $\infnorm{W e_i} \leq 1$. Combining the above bound with~\eqref{eq:gmhat-gm-infnorm} and the fact that $\onenorm{\muhat} = 1$, we obtain the desired bound.

We now verify feasibility of \(\yhat\) for the perturbed primal
LP~\eqref{eq:primal_LP_pert}. Using Lemma~\ref{lem:W_properties} and Lemma~\ref{lem:perturbed-Uhat-Huhat-identities}, we have
\(
    \yhat\tilde e=\muhat W\tilde e=\muhat(\alpha\one)=\alpha,
    \yhat\one=\muhat W\one=\muhat\one=1, \yhat\Phat=\muhat W\Phat=\muhat\gmhat=\muhat=\muhat WJ=\yhat J.
\)
Thus it remains to check nonnegativity. If \(i=(s,a)\), \(s\neq\tilde s\), and
\(\pibst(a\mid s)=0\), then \(y^\star_i=0\) and \(We_i=\zero\), so
\eqref{eq:yhat_pert_bound_int} gives \(\yhat_i=0\). For indices
\(i\) with \(y^\star_i>0\), using \(\infnorm{We_i}\le1\) and
\eqref{eq:yhat_pert_bound_int},
\begin{align*}
    |\yhat_i - y^\star_i| 
    &\leq \onenorm{\muhat} \infinfnorm{\gmhat - \Gamma} \infinfnorm{H_{\Gamma}} \infnorm{We_i} 
    \leq \infinfnorm{\gmhat - \Gamma} \infinfnorm{H_{\Gamma}} 
     \stackrel{(*)}{\leq} \frac{1}{2}\min_{j : y^\star_j > 0} y^\star_j \leq \frac{1}{2} y^\star_i,
\end{align*}
where $(*)$ follows from~\eqref{eq:well-define-Huhat}. Hence
\(\yhat_i\ge y^\star_i/2>0\) for all \(i\) with \(y^\star_i>0\).
Therefore, \(\yhat\ge0\), and all constraints of
\eqref{eq:primal_LP_pert} are satisfied. \Halmos
\end{proof}

\subsection{Proof of Lemma~\ref{lem:true_dual_solution_form}}\label{subsec:pr-h-construction}
\label{sec:}
\begin{proof}{Proof.}
    Note that each row of $W$ has non-zero entries on indices $i$ such that $y^\star_i > 0$, where~\eqref{eq:true_CS_equality_consequence} holds. Therefore, we have
    $
        W (r + \lambda^\star \tilde{e} + Ph^\star) = W(\zeta^\star \one + Jh^\star).
    $
    Using the properties of $W$ from Lemma \ref{lem:W_properties} that $WP = \Gamma, W \one = \one, W J = I$, and $W \tilde{e} = \alpha \one$, the above equation can be simplified to
    $
        Wr + \lambda^\star \alpha \one + \Gamma h^\star = \zeta^\star \one +  h^\star.
    $ 
    Plugging  $\zeta^\star = \mu^\star W r + \lambda^\star \alpha$ (since the primal value $y^\star r = \mu^\star W r$ must be equal to the dual value $\zeta^\star - \alpha \lambda^\star$), we have
    $
        Wr + \lambda^\star \alpha \one + \Gamma h^\star = (\mu^\star W r + \lambda^\star \alpha) \one +  h^\star,
    $
    which can be further simplified to
    $
        (I - \Gamma)h^\star = Wr - \mu^\star W r \one.
    $
    The solutions to this Poisson-like equation can be characterized using the assumed invertibility of $(I - (\Gamma-\one \mu^\star))$. Rewriting the equation as
    $
        \left(I - (\Gamma-\one \mu^\star)\right)h^\star = (I - \one \mu^\star)Wr + \one \mu^\star h,
    $
    we obtain
    \begin{align*}
        h^\star &= \left(I - (\Gamma-\one \mu^\star)\right)^{-1}\left( (I - \one \mu^\star)Wr + \one \mu^\star h^\star\right)\\
        &=\left(I - (\Gamma-\one \mu^\star)\right)^{-1} Wr - \one \mu^\star Wr + \one \mu^\star h^\star
        = H_{\Gamma} Wr + \one \mu^\star h^\star
    \end{align*}
    using that $\left(I - (\Gamma-\one \mu^\star)\right)^{-1} \one = \one$, which follows from the fact that $H_{\Gamma} \one = \zero$ (Lemma \ref{lem:basic-U-HU-identities}). \Halmos
\end{proof}

\subsection{Proof of Lemma~\ref{lem:pert_dual_LP_equality_rows} and Lemma~\ref{lem:pert_dual_LP_ineq_consts}}\label{subsec:pr-feasibility-dual}

Before checking the feasibility of $(\zetahat, \lambdahat, \hhat)$, we show one more useful fact about $Wr$.
\begin{lemma}
    \label{lem:Wr_computation}
    We have
    $
    Wr = r^{\pibst} + (c^{\pibst} - \alpha \one)(r(\Tilde{s},0) - r(\Tilde{s},1)).
    $
\end{lemma}
\begin{proof}{Proof.}
    For $s \neq \Tilde{s}$, using~\eqref{eq:W_entries_1} we have
\begin{align*}
    (Wr)(s) 
    &= \sum_{s'a'} W_{s, s'a'} r(s',a') \\
    &= \pibst(0 \mid s) r(s,0) + \pibst(1 \mid s) r(s,1) + \left((1-\alpha) - \pibst(0\mid s) \right) r(\Tilde{s},0) + \left( \alpha - \pibst(1\mid s)\right) r(\Tilde{s},1) \\
    &= r^{\pibst}(s) - \alpha r(\Tilde{s},0) + \alpha r(\Tilde{s},1) + c^{\pibst}(s)r(\Tilde{s},0) - c^{\pibst}(s)r(\Tilde{s},1) \\
    &=r^{\pibst}(s) + (c^{\pibst}(s) - \alpha)(r(\Tilde{s},0) - r(\Tilde{s},1)).
\end{align*}
For $s = \Tilde{s}$, using~\eqref{eq:W_entries_2} we have
\begin{align*}
    (Wr)(\Tilde{s}) 
    &= \sum_{s'a'} W_{\Tilde{s}, s'a'} r(s',a') = (1-\alpha )r(\Tilde{s},0) + \alpha r(\Tilde{s},1) \\
    &= (1-\alpha )r(\Tilde{s},0) + \alpha r(\Tilde{s},1) + r^{\pibst}(\Tilde{s}) - \pibst(1 \mid \Tilde{s})r(\Tilde{s},1) - \pibst(0 \mid \Tilde{s})r(\Tilde{s},0) \\
    &= r^{\pibst}(\Tilde{s}) + (c^{\pibst}(\Tilde{s}) - \alpha)(r(\Tilde{s},0) - r(\Tilde{s},1)).
\end{align*}
Therefore we have that
$
    Wr = r^{\pibst} + (c^{\pibst} - \alpha \one)(r(\Tilde{s},0) - r(\Tilde{s},1))
$
as desired. \Halmos
\end{proof}

\begin{proof}{Proof of Lemma~\ref{lem:pert_dual_LP_equality_rows}.}
By the fact that $(I - \gmhat)H_{\gmhat} = I - \one \muhat$ from Lemma \ref{lem:properties-H-gmhat} and the definition of $\widehat{h}=H_{\gmhat}Wr$, we have that
$
    \hhat - \gmhat \hhat = Wr - \one \muhat Wr. \nonumber
$
Rearranging and using $\muhat W r = \zetahat - \lambdahat \alpha $ from the definition~\eqref{eq:zetahat_defn} of $\zetahat$, we have
\begin{align}
    \hhat + \zetahat\one = Wr + \gmhat \hhat + \lambdahat \alpha \one= Wr + \Phat^{\pibst} \hhat - \left(c^{\pibst} - \alpha \one\right)\xihat \hhat + \lambdahat \alpha \one. \label{eq:h'+rho'1=Wr+U'h'}
\end{align}
Using Lemma \ref{lem:Wr_computation} as well as the definition~\eqref{eq:lambdahat_defn} of $\lambdahat=r(\tilde{s},0) - r(\tilde{s},1) - \xihat \hhat$, we have
\begin{align*}
    &Wr + \Phat^{\pibst} \hhat - \left(c^{\pibst} - \alpha \one\right)\xihat \hhat + \lambdahat \alpha \one\\
    &= \left(r^{\pibst} + (c^{\pibst} - \alpha \one)(r(\Tilde{s},0) - r(\Tilde{s},1)) \right) + \Phat^{\pibst} \hhat - \left(c^{\pibst} - \alpha \one\right)\left( r(\tilde{s},0) - r(\tilde{s},1) - \lambdahat\right) + \lambdahat \alpha \one \\
    &= r^{\pibst} + \Phat^{\pibst} \hhat + \lambdahat c^{\pibst}.
\end{align*}
Combining the above equality with \eqref{eq:h'+rho'1=Wr+U'h'}, we obtain that
\begin{align}
    r^{\pibst} + \Phat^{\pibst} \hhat + \lambdahat c^{\pibst}=\hhat+\zetahat\one. 
    \label{eq:pert_LP_bellman_eqn}
\end{align}

Finally we check the equality of the rows involving the neutral state. By the definition~\eqref{eq:lambdahat_defn} of $\lambdahat$,
$
    \lambdahat
    = r(\tilde{s},0) - r(\tilde{s},1) - \xihat \hhat= r(\tilde{s},0) - r(\tilde{s},1) -(\Phat_{\Tilde{s}1} - \Phat_{\Tilde{s}0})\hhat,
$
we thus have
\begin{align}
    r(\tilde{s},0) + \Phat_{\Tilde{s}0} \hhat = r(\tilde{s},1) + \lambdahat + \Phat_{\Tilde{s}1} \hhat. 
    \label{eq:pert_equal_Q_cond}
\end{align}
Therefore, \eqref{eq:pert_equal_Q_cond} implies
\begin{align*}
    r^{\pibst}(\Tilde{s}) + e_{\Tilde{s}}^\top\Phat^{ \pibst} \hhat + \lambdahat c^{\pibst}(\Tilde{s}) &= \pibst(0 \mid \Tilde{s})\left(r(\tilde{s},0) + \Phat_{\Tilde{s}0} \hhat\right) + \pibst(1 \mid \Tilde{s})\left(r(\tilde{s},1) + \lambdahat + \Phat_{\Tilde{s}1} \hhat\right)\\
    &= r(\tilde{s},0) + \Phat_{\Tilde{s}0} \hhat = r(\tilde{s},1) + \lambdahat + \Phat_{\Tilde{s}1} \hhat.
\end{align*}
Combining with~\eqref{eq:pert_LP_bellman_eqn} (specialized to just row $\Tilde{s}$), we obtain
$
    r(\tilde{s},0) + \Phat_{\Tilde{s}0} \hhat  = r(\tilde{s},1) + \lambdahat + \Phat_{\Tilde{s}1} \hhat = \hhat(\Tilde{s})+\zetahat. \Halmos
$
\end{proof}
\begin{proof}{Proof of Lemma~\ref{lem:pert_dual_LP_ineq_consts}.}
We start by showing the first three perturbation bounds \eqref{eq:dual-perturbation-h}--\eqref{eq:dual-perturbation-zeta} between the true dual solution and the perturbed dual candidate solution.
Using Lemma \ref{lem:true_dual_solution_form}, and that $\spannorm{\cdot}$ is invariant to shifts by $\one$, we have
\begin{align*}
    \spannorm{\hhat - h^\star} 
    & =\spannorm{H_\Gamma Wr-H_{\gmhat}Wr}\\
    &= \spannorm{H_{\Gamma}(\gmhat-\Gamma)H_{\gmhat}Wr + \one x Wr} \tag{by Lemma~\ref{lem:properties-H-gmhat}}\\
    &= \spannorm{H_{\Gamma} (\gmhat-\Gamma) h^\star + H_{\Gamma} (\gmhat-\Gamma) (\hhat-h^\star)} \\
    & \le \spannorm{H_{\Gamma} (\gmhat-\Gamma) h^\star} + \infinfnorm{H_{\Gamma}} \infinfnorm{\gmhat-\Gamma} \spannorm{\hhat-h^\star} .
\end{align*}
Rearranging yields
\begin{equation}\label{eq:hhat-hstar}
    \begin{aligned}
        \spannorm{\hhat - h^\star} & \leq \frac{\spannorm{H_{\Gamma} (\gmhat-\Gamma) h^\star}}{1-\infinfnorm{H_{\Gamma}} \infinfnorm{\gmhat-\Gamma}} \leq \frac{\infinfnorm{H_{\Gamma}} \infinfnorm{\gmhat - \Gamma} \spannorm{h^\star}}{1-\infinfnorm{H_{\Gamma}} \infinfnorm{\gmhat-\Gamma}}\\
        & \le 2 \infinfnorm{H_{\Gamma}} \infinfnorm{\gmhat - \Gamma} \spannorm{h^\star}\leq 6 \delta \infinfnorm{H_{\Gamma}} \spannorm{h^\star},
    \end{aligned}
\end{equation}
where the second inequality holds because
\begin{align*}
    \spannorm{H_{\Gamma} (\gmhat-\Gamma) h^\star} &= \inf_{\kappa \in \R} \spannorm{H_{\Gamma} (\gmhat-\Gamma) (h^\star- \kappa \one)}\\
    &\leq 2 \inf_{\kappa \in \R}\infnorm{H_{\Gamma} (\gmhat-\Gamma) (h^\star- \kappa \one)} \\
    &\leq 2 \infinfnorm{H_{\Gamma}} \infinfnorm{\gmhat - \Gamma} \inf_{\kappa \in \R}\infnorm{h^\star - \kappa \one} 
    = \infinfnorm{H_{\Gamma}} \infinfnorm{\gmhat - \Gamma} \spannorm{h^\star}.
\end{align*}
And the third inequality of~\eqref{eq:hhat-hstar} follows from~\eqref{eq:well-define-Huhat}. 
The last inequality of~\eqref{eq:hhat-hstar} follows by~\eqref{eq:gmhat-gm-infnorm}.  
Next, for $\lambdahat-\lambda^\star$, we have
\begin{align*}
    \lambdahat - \lambda^\star &= \xi h^\star - \xihat \hhat 
    = (\xi - \xihat) h^\star + \xihat (h^\star - \hhat) 
    = (\xi - \xihat) (h^\star - \kappa \one) + \xihat (h^\star - \hhat - \kappa' \one)
\end{align*}
for any $\kappa, \kappa' \in \R$, since as checked previously $\xi \one = \xihat \one = 0$. Hence
\begin{equation}\label{eq:lambdahat-lamstar}
    \begin{aligned}
        \left| \lambdahat - \lambda^\star \right| & \leq \infinfnorm{\xi - \xihat} \inf_{\kappa \in \R}\infnorm{h^\star - \kappa \one} +  \inf_{\kappa' \in \R}\infinfnorm{\xihat} \infnorm{h^\star - \hhat - \kappa' \one} \\
    &= \frac{1}{2}\infinfnorm{\xi - \xihat} \spannorm{h^\star} +  \frac{1}{2}\infinfnorm{\xihat} \spannorm{\hhat - h^\star} \\
    &\leq \frac{1}{2}\infinfnorm{\xi - \xihat} \spannorm{h^\star} + \spannorm{\hhat - h^\star}\le \delta\spannorm{h^\star}
\left(1+6\infinfnorm{H_{\Gamma}}\right),
    \end{aligned}
\end{equation}
where the second inequality uses that $\infinfnorm{\xihat} = \infinfnorm{\Phat_{\tilde s 1} - \Phat_{\tilde s 0}} \leq \infinfnorm{\Phat_{\tilde s 1}} + \infinfnorm{ \Phat_{\tilde s 0}} = 2$ and the last inequality uses~\eqref{eq:hhat-hstar} and \eqref{eq:gmhat-gm-infnorm}.

Next we bound $|\zetahat - \zeta^\star|$. We have, for an arbitrary $\kappa \in \R$, that
\begin{align*}
    \zetahat - \zeta^\star 
    = \muhat Wr + \lambdahat \alpha - \mu^\star Wr - \lambda^\star \alpha 
    &= (\muhat - \mu^\star) Wr + (\lambdahat - \lambda^\star) \alpha \\
    & = \muhat (\gmhat - \Gamma) H_{\Gamma} Wr + (\lambdahat - \lambda^\star) \alpha = \muhat (\gmhat - \Gamma ) (h^\star +\kappa \one) + (\lambdahat - \lambda^\star) \alpha,
\end{align*}
using \eqref{eq:equaldef-muhat} and that $h^\star$ is equal to $H_{\Gamma} W r$ plus a shift by a multiple of $\one$ (by Lemma \ref{lem:true_dual_solution_form}), as well as $(\gmhat - \Gamma ) \one = \zero$ (by Lemma \ref{lem:basic-U-HU-identities}). 
By triangle inequality we have
\begin{align*}
    \left| \zetahat - \zeta^\star \right| 
    & \leq \infinfnorm{\muhat} \infinfnorm{\gmhat - \Gamma} \inf_{\kappa \in \R} \infnorm{h^\star + \kappa \one} + \left|\lambdahat - \lambda^\star \right||\alpha| \\
    & \leq \infinfnorm{\gmhat - \Gamma} \frac{1}{2}\spannorm{h^\star} + \left|\lambdahat - \lambda^\star \right|\le \delta\spannorm{h^\star}
\left(\frac{5}{2}+6\infinfnorm{H_{\Gamma}}\right),
\end{align*}
using $\infinfnorm{\muhat} = 1$ ($\muhat $ is a probability distribution by Lemma \ref{lem:yhat_pert_and_feasibility}), $|\alpha| \leq 1$, \eqref{eq:lambdahat-lamstar} and \eqref{eq:gmhat-gm-infnorm}.

Now we show the final part of the lemma statement.
Equivalently, our goal is to show that elementwise
$
    M \left( r + \lambdahat \Tilde{e} + \Phat\hhat - \zetahat\one - J \hhat\right) < \zero,
$
where $M$ select the rows appearing in the conclusion of the lemma: $M \in \R^{(S-1) \times (SA)}$, each row has a single nonzero entry corresponding to a pair $(s,a)$ with $s \neq \tilde{s}$ and $\pibst(a \mid s) = 0$, and the selected rows are distinct. By support preservation these are exactly the state-action pairs with $y^\star(s,a)=0$.
We start from the identity
$
    r + \lambdahat \Tilde{e} + \Phat\hhat - \zetahat\one - J \hhat 
    = (\Phat - P)h^\star + (\Phat - J)(\hhat-h^\star) + (\lambdahat - \lambda^\star)\Tilde{e} - (\zetahat-\zeta^\star)\one + \left(r +  \lambda^\star \Tilde{e} + Ph^\star - \zeta^\star\one - J h^\star \right).
$
Now we apply the matrix $M$ to both sides, 
\begin{multline*}
    M \left( r + \lambdahat \Tilde{e} + \Phat\hhat - \zetahat\one - J \hhat \right) \\
    \begin{aligned}
        &= M \left( (\Phat - P)h^\star + (\Phat - J)(\hhat-h^\star) + (\lambdahat - \lambda^\star)\Tilde{e} - (\zetahat-\zeta^\star)\one + \left(r +  \lambda^\star \Tilde{e} + Ph^\star - \zeta^\star\one - J h^\star \right)\right) \\
        &= M \Big( (\Phat - P)(h^\star - \kappa \one) + (\Phat - J)(\hhat-h^\star - \kappa' \one) + (\lambdahat - \lambda^\star)\Tilde{e} - (\zetahat-\zeta^\star)\one 
         + \left(r +  \lambda^\star \Tilde{e} + Ph^\star - \zeta^\star\one - J h^\star \right)\Big) \\
        & \leq M \left(r +  \lambda^\star \Tilde{e} + Ph^\star - \zeta^\star\one - J h^\star \right) + \infinfnorm{\Phat - P }\infnorm{h^\star - \kappa \one}\one  \\
        & \qquad + \infinfnorm{\Phat - J} \infnorm{\hhat-h^\star - \kappa' \one}\one+ \left|\lambdahat - \lambda^\star \right| \infnorm{\Tilde{e}} \one+ \left|\zetahat-\zeta^\star\right|\one \\
        & \leq M \left(r +  \lambda^\star \Tilde{e} + Ph^\star - \zeta^\star\one - J h^\star \right) + \infinfnorm{\Phat - P } \frac{1}{2}\spannorm{h^\star}\one 
        + \spannorm{\hhat - h^\star}\one + \left|\lambdahat - \lambda^\star \right|\one  + \left|\zetahat-\zeta^\star\right|\one. 
    \end{aligned}
\end{multline*}
For the second equality, we use $(\Phat -P) \one = \zero$ and $(\Phat - J) \one = \zero$; for the inequality steps, we use $\infinfnorm{J} = \max_{s a} \onenorm{e_s} = 1$, $\infnorm{\Tilde{e}} = 1$, $\infinfnorm{M}= 1$, and choose $\kappa \in \R$ so that $\infnorm{h^\star - \kappa \one}= \spannorm{h^\star}/2$. Using the perturbation bounds
\eqref{eq:dual-perturbation-h}--\eqref{eq:dual-perturbation-zeta}, we have that
\begin{align*}
    \infinfnorm{\Phat - P } \frac{1}{2}\spannorm{h^\star} + \spannorm{\hhat - h^\star} + \left|\lambdahat - \lambda^\star \right|  + \left|\zetahat-\zeta^\star\right|\le \delta \spannorm{h^\star}(4 + 18\infinfnorm{H_{\Gamma}}).
\end{align*}

Finally, letting
$
    G = \min_{sa:y^\star(s,a)=0} [\zeta^\star +  h^\star(s) -r(s, a) - \lambda^\star \mathbb{I}(a=1) - P_{s a}h^\star] > 0,
$
we can combine the above calculations to obtain that
\begin{align*}
    M \left( r + \lambdahat \Tilde{e} + \Phat\hhat - \zetahat\one - J \hhat \right) 
    & \leq M \left(r +  \lambda^\star \Tilde{e} + Ph^\star - \zeta^\star\one - J h^\star \right) + \delta \spannorm{h^\star}(4 + 18\infinfnorm{H_{\Gamma}})\one \\
    & \leq -G \one + \delta \spannorm{h^\star}(4 + 18\infinfnorm{H_{\Gamma}})\one,
\end{align*}
which is elementwise strictly less than $\zero$ when
$
    \delta \leq \frac{G/2}{4  \spannorm{h^\star} + 18\infinfnorm{H_{\Gamma}}  \spannorm{h^\star}},
$
as desired. \Halmos
\end{proof}

\subsection{Proof of Lemma~\ref{lem:uniqueness}}\label{subsec:pr-uniqueness}
\proof{Proof.}
We first show uniqueness of the primal optimal solutions. Let $\tilde{y}$ be another optimal solution to the perturbed primal LP, then $\tilde{y}$ and $(\zetahat, \lambdahat, \hhat)$ must satisfy the complementary slackness property~\eqref{eq:LP_CS_pert} (as optimal solutions). By Lemma \ref{lem:pert_dual_LP_ineq_consts}, the dual constraints are strict for all $s,a$ such that $y^\star(s,a) = 0$, which implies that we must have that $\tilde{y}(s,a)=0$ for all $s,a$ such that $y^\star(s,a)=0$. Therefore, letting $\tilde{\mu}(s)= \tilde{y}(s,0)+\tilde{y}(s,1)$, we have that $\tilde{y} = \tilde{\mu} W$, since $\tilde{y}$ is supported on a (possibly non-proper) subset of the support of $y^\star$, and also $\tilde{y} \tilde{e} = \alpha$ by its feasibility for the primal perturbed LP. Then since primal feasibility also guarantees $\tilde{y}\Phat = \tilde{y}J$, we have
$
    \tilde{\mu} \gmhat = \tilde{\mu} W \Phat = \tilde{y} \Phat = \tilde{y} J = \tilde{\mu} W J = \tilde{\mu} I = \tilde{\mu}.
$
From the fact that $\tilde{\mu} \gmhat = \tilde{\mu}$, we can derive that
\begin{align*}
    \tilde{\mu} (\gmhat - \one \muhat) = \tilde{\mu}\gmhat - \tilde{\mu}\one \muhat = \tilde{\mu} - \muhat 
    \quad\implies\quad &\muhat = \tilde{\mu}\left(I - (\gmhat - \one \muhat) \right) \\
    \quad\implies\quad & \tilde{\mu} = \muhat \left(I - (\gmhat - \one \muhat) \right)^{-1} = \muhat H_{\gmhat} + \muhat \one \muhat = \muhat
\end{align*}
using that $\tilde{\mu} \one = \tilde{\mu} W \one = \tilde{y} \one = 1$ by primal feasibility of $\tilde{y}$ in the first line, and then invertibility of $I - (\gmhat - \one \muhat)$ (by Lemma \ref{lem:well-define-Huhat}) and then the facts that $\muhat H_{\gmhat} = \zero$ and $\muhat \one = 1$ (both from Lemma \ref{lem:perturbed-Uhat-Huhat-identities}) in the final line. Finally, $\muhat = \tilde{\mu}$ implies that $\tilde{y} = \tilde{\mu} W = \muhat W = \yhat$, so we have shown uniqueness of $\yhat$.

Next we show uniqueness of the dual solution.
These steps are analogous to those of Lemma \ref{lem:true_dual_solution_form}. Let $(\tilde{\zeta}, \tilde{\lambda}, \tilde{h})$ be some dual optimal solution.
By complementary slackness between
\((\tilde{\zeta},\tilde{\lambda},\tilde h)\) and \(\yhat\),
for every \((s,a)\) such that \(\yhat(s,a)>0\),
\begin{equation}
\label{eq:dual-tight-on-yhat-support}
    r(s,a)
    +\tilde{\lambda}\mathbb I\{a=1\}
    +\Phat_{sa}\tilde h
    =
    \tilde{\zeta}+\tilde h(s).
\end{equation}
By the construction of $\yhat$ and Lemma \ref{lem:yhat_pert_and_feasibility}, these are exactly the $s,a$ such that $y^\star(s,a)>0$. Each row of $W$ is only supported on such $(s,a)$, so we have that $W(r + \tilde{\lambda}\Tilde{e} + \Phat \tilde{h})=W(\tilde{\zeta}\one + J\tilde{h}).$
Using the properties of $W$ from Lemma \ref{lem:W_properties} that $W\Phat = \gmhat$, $W \one = \one$, $WJ = I$, and $W \tilde{e} = \alpha \one$, we can further simply the equation to
$
    Wr + \tilde{\lambda} \alpha \one + \gmhat \tilde{h} = \tilde{\zeta} \one + \tilde{h}.
$
Since $(\tilde{\zeta}, \tilde{\lambda}, \tilde{h})$ is dual optimal and has the same objective value as $\yhat$ for the perturbed primal LP, we have that $\tilde{\zeta} = \muhat W r + \tilde{\lambda} \alpha$. Combining the last two equalities, we obtain that
$
    (I - \gmhat)\tilde{h} = Wr - \muhat W r \one.
$
After rewriting this as $(I - (\gmhat - \one \muhat)) \tilde{h} = (I - \one \muhat) Wr + \one \muhat \tilde{h}$ and then multiplying by $(I - (\gmhat - \one \muhat))^{-1}$, we obtain that
\begin{align*}
    \tilde{h} &= (I - (\gmhat - \one \muhat))^{-1} \left( (I - \one \muhat) Wr + \one \muhat \tilde{h} \right) \\
    &= (I - (\gmhat - \one \muhat))^{-1}Wr - \one \muhat Wr + \one \muhat \tilde{h} 
    = H_{\gmhat} Wr + \one \muhat \tilde{h} = \hhat + \one \muhat \tilde{h}
\end{align*}
using that $(I - (\gmhat - \one \muhat))^{-1} \one = \one$, which follows from $H_{\gmhat} \one = \zero$ (Lemma \ref{lem:properties-H-gmhat}).
By support preservation,
\(\yhat(\tilde s,0)>0\) and \(\yhat(\tilde s,1)>0\).
Applying~\eqref{eq:dual-tight-on-yhat-support} to
\((\tilde s,0)\) and \((\tilde s,1)\), respectively, gives
\[
    r(\tilde s,0)+\Phat_{\tilde s0}\tilde h
    =
    \tilde{\zeta}+\tilde h(\tilde s)
    =
    r(\tilde s,1)
    +\tilde{\lambda}
    +\Phat_{\tilde s1}\tilde h.
\]
Hence by rearranging we have
\begin{align*}
    \tilde{\lambda} 
    = r(\tilde{s},0)-r(\tilde{s},1) - \Phat_{\tilde{s}1}\tilde{h} +\Phat_{\tilde{s}0}\tilde{h} 
    &= r(\tilde{s},0)-r(\tilde{s},1) - \xihat \tilde{h} \\
    &= r(\tilde{s},0)-r(\tilde{s},1) - \xihat (\hhat + \one \muhat \tilde{h}) = r(\tilde{s},0)-r(\tilde{s},1) - \xihat \hhat 
    = \lambdahat
\end{align*}
using the fact that $\xihat \one = \zero$ and the definition~\eqref{eq:lambdahat_defn} of $\lambdahat$.
Finally, we have already shown that $\tilde{\zeta} = \muhat W r + \tilde{\lambda} \alpha$, but now since we have also shown that $\tilde{\lambda} = \lambdahat$, by combining with the definition~\eqref{eq:zetahat_defn} of $\zetahat$, we can conclude that $\tilde{\zeta} = \zetahat$. \Halmos
\endproof

\subsection{Proof of Part~5 of
Theorem~\ref{thm:LP_perturbation_main_complete}}
\label{subsec:proof-perturbed-spectral-gramian}

We first bound $\twonorm{\phihat-\Phi}$, then use this bound
to control $\rho(\phihat)$, $\uhat$ and $(I-\phihat)^{-1}$.

The policies $\pihstar$ and $\pibst$ differ only at $\tilde s$,
so
$\Phat^{\pihstar}-\Phat^{\pibst}
=(c^{\pihstar}-c^{\pibst})\xihat$.
Substituting this identity into the definition of $\phihat$
gives $\phihat=\gmhat-\one\muhat$.
Since $\Phi=\Gamma-\one\mu^\star$, we have
$\phihat-\Phi=\gmhat-\Gamma-\one(\muhat-\mu^\star)$.
By~\eqref{eq:equaldef-muhat} and~\eqref{eq:gmhat-gm-infnorm},
\(
\infinfnorm{\one(\muhat-\mu^\star)}
\le \onenorm{\muhat}
\infinfnorm{\gmhat-\Gamma}\infinfnorm{H_{\Gamma}}
\le3\delta\infinfnorm{H_{\Gamma}}.
\)
Therefore, 
\begin{equation}
\label{eq:perturbation-local-matrix-error}
\twonorm{\phihat-\Phi}\le \sqrt{S}\infinfnorm{\phihat-\Phi}
\le\sqrt S\left(\infinfnorm{\gmhat-\Gamma}
             +\infinfnorm{\one(\muhat-\mu^\star)}\right)
\le3\sqrt S\,\delta(1+\infinfnorm{H_{\Gamma}}).
\end{equation}

Now we prove that $\rho(\phihat)<1$.
The definition of $U=\sum_{k=0}^{\infty}\Phi^k(\Phi^\top)^k$ gives
$U-\Phi U\Phi^\top=I$ and
$I\preceq U\preceq\lambda_U I$.
Writing $E=\phihat-\Phi$, the first term
in~\eqref{eq:def-delta-min}, together with
\eqref{eq:perturbation-local-matrix-error}, ensures that
$\twonorm{E}\le(2\lambda_U(2\twonorm{\Phi}+1))^{-1}\le1$.
Since $\phihat=\Phi+E$, we have
\(
\phihat U\phihat^\top-\Phi U\Phi^\top
=EU\Phi^\top+\Phi UE^\top+EUE^\top.
\)
Thus, using $\twonorm{U}=\lambda_U$,
$\twonorm{\Phi^\top}=\twonorm{\Phi}$, and
$\twonorm{E^\top}=\twonorm{E}$, we obtain
\begin{align*}
\twonorm{\phihat U\phihat^\top-\Phi U\Phi^\top}
&\le\twonorm{E}\twonorm{U}\twonorm{\Phi^\top}
   +\twonorm{\Phi}\twonorm{U}\twonorm{E^\top}
   +\twonorm{E}\twonorm{U}\twonorm{E^\top}\\
&=\lambda_U
  \bigl(2\twonorm{\Phi}\twonorm{E}+\twonorm{E}^2\bigr)\le\lambda_U(2\twonorm{\Phi}+1)\twonorm{E}\le\frac12.
\end{align*}
The penultimate inequality uses $\twonorm{E}\le1$,
so $\twonorm{E}^2\le\twonorm{E}$.
The last inequality follows from
$\twonorm{E}\le[2\lambda_U(2\twonorm{\Phi}+1)]^{-1}$.

We next use this norm bound to compare
$\phihat U\phihat^\top$ with $U$.
Because $U$ is symmetric, the matrix
$\phihat U\phihat^\top-\Phi U\Phi^\top$ is also symmetric.
For a symmetric matrix, the Euclidean operator norm is
the largest absolute value of its eigenvalues.
Thus $\twonorm{\phihat U\phihat^\top-\Phi U\Phi^\top}\le\frac{1}{2}$ implies
\(
-\frac12I
\preceq
\phihat U\phihat^\top-\Phi U\Phi^\top
\preceq
\frac12I.
\)  
Thus,
\[
\begin{aligned}
U-\phihat U\phihat^\top
=(U-\Phi U\Phi^\top)
  -(\phihat U\phihat^\top-\Phi U\Phi^\top)\succeq I-\frac12I
=\frac12I.
\end{aligned}
\]
Therefore $\phihat U\phihat^\top\preceq U-I/2$.
Moreover, $U\preceq\lambda_U I$ gives
$I\succeq U/\lambda_U$, and hence
\(
U-\frac12I
\preceq U-\frac{1}{2\lambda_U}U
=\left(1-\frac{1}{2\lambda_U}\right)U.
\)
Combining these inequalities yields
$U-\phihat U\phihat^\top\succeq\frac12I$ and 
$\phihat U\phihat^\top
\preceq\left(1-\frac{1}{2\lambda_U}\right)U.$

Since $U\succ0$, 
we can multiply on the left and right of the
inequality $\phihat U\phihat^\top
\preceq\left(1-\frac{1}{2\lambda_U}\right)U$ by $U^{-1/2}$
and obtain
\(
U^{-1/2}\phihat U\phihat^\top U^{-1/2}
\preceq
\left(1-\frac{1}{2\lambda_U}\right)I.
\)
Then define $B\triangleq U^{-1/2}\phihat U^{1/2}$, note that $BB^\top
=U^{-1/2}\phihat U\phihat^\top U^{-1/2}$. 
It follows that
$\twonorm{B}^2=\lambda_{\max}(BB^\top)
\le1-\frac{1}{2\lambda_U}$.
Also, $B$ is similar to $\phihat$, so they have the same
eigenvalues and hence the same spectral radius.
We then conclude that
\[
\rho(\phihat)=\rho(B)
\le\twonorm{B}
\le\sqrt{1-\frac{1}{2\lambda_U}}<1.
\]
Here $\lambda_U\ge1$ follows from $U\succeq I$,
and $\lambda_U$ is finite because $\rho(\Phi)<1$.

Finally we turn to control $\uhat$.
Because $\rho(\phihat)<1$, the matrix series
$\uhat=\sum_{k=0}^{\infty}\phihat^k(\phihat^\top)^k$
converges.
To compare $\uhat$ with $U$, we multiply the inequality
$U-\phihat U\phihat^\top\succeq\frac12I$
on the left by $\phihat^k$ and on the right by
$(\phihat^\top)^k$, and sum over $k=0,\ldots,m$ to get
\[
U-\phihat^{m+1}U(\phihat^\top)^{m+1}
\succeq\frac12\sum_{k=0}^{m}
\phihat^k(\phihat^\top)^k.
\]
Since $\rho(\phihat)<1$, the matrix
$\phihat^{m+1}U(\phihat^\top)^{m+1}$ tends to zero.
Letting $m\to\infty$ gives $\uhat\preceq2U$.
The $k=0$ term in the definition of $\uhat$ is $I$,
and every other term is positive semidefinite.
Thus $I\preceq\uhat\preceq2U$ and
$\lambda_{\uhat}\le2\lambda_U$.

It remains to bound $(I-\phihat)^{-1}$.
The following lemma shows how the identity
$\uhat-\phihat\uhat\phihat^\top=I$ gives such a bound.
\begin{lemma}
\label{lem:gramian-inverse-bound}
Let $A,G\in\mathbb R^{S\times S}$ satisfy
$\rho(A)<1$, $G\succ0$, and $G-AGA^\top=I$.
Then
\begin{equation}
\label{eq:perturbation-resolvent-proof}
\twonorm{(I-A)^{-1}}\le2\twonorm{G},
\qquad
\infinfnorm{(I-A)^{-1}}\le2\sqrt S\twonorm{G}.
\end{equation}
\end{lemma}

\proof{Proof.}
Since $\rho(A)<1$, the matrix $I-A$ is invertible.
For any nonzero row vector $w$, let $v\triangleq w(I-A)^{-1}$.
Then $v\ne0$ and $vA=v-w$.
Using $G-AGA^\top=I$, we obtain
\[
\begin{aligned}
\twonorm{v}^2
=v(G-AGA^\top)v^\top=vGv^\top-(v-w)G(v-w)^\top
=2vGw^\top-wGw^\top\le2\twonorm{G}\twonorm{v}\twonorm{w}.
\end{aligned}
\]
The last inequality follows from $wGw^\top\ge0$ and
$|vGw^\top|\le\twonorm{G}\twonorm{v}\twonorm{w}$.
Dividing by $\twonorm{v}$ gives
$\twonorm{w(I-A)^{-1}}\le2\twonorm{G}\twonorm{w}$.
Then
\[
\twonorm{(I-A)^{-1}}
=\sup_{w\ne0}
\frac{\twonorm{w(I-A)^{-1}}}{\twonorm{w}}
\le2\twonorm{G}.
\]
The second bound in~\eqref{eq:perturbation-resolvent-proof}
then follows from
$\infinfnorm{C}\le\sqrt S\twonorm{C}$ for an
$S\times S$ matrix $C$.
\Halmos
\endproof

We have proved that $\rho(\phihat)<1$ and $\uhat\succeq I$.
The definition of $\uhat$ gives
$\uhat-\phihat\uhat\phihat^\top=I$.
Applying Lemma~\ref{lem:gramian-inverse-bound}
with $(A,G)=(\phihat,\uhat)$ therefore yields
$\infinfnorm{(I-\phihat)^{-1}}
\le2\sqrt S\lambda_{\uhat}\le4\sqrt S\lambda_U$,
where the last inequality uses $\lambda_{\uhat}\le2\lambda_U$.
\Halmos

\subsection{Proof of Part 6 of Theorem~\ref{thm:LP_perturbation_main_complete}}\label{subsec:pr-mc-ergordic}
\proof{Proof.}
 By support preservation, \(\pihstar(s)=\pistar(s)\) for all
\(s\neq\tilde{s}\). Hence the two policies can differ only at the neutral
state \(\tilde{s}\). By the perturbation bound for \(\yhat^\star\), and by the definition of
\(\delta_{\min}\), for each \(a\in\{0,1\}\),
\begin{equation}\label{eq:ystar-yhstar}
    |y^\star(\tilde{s},a)-\yhat^\star(\tilde{s},a)|
    \le
    3\delta\infinfnorm{H_{\Gamma}}
    \le
    \frac{\mu^\star_{\min}}{48S\tau(5+\log_2S)}.
\end{equation}
Since
\(
    \pistar(a\mid\tilde{s})
    =
    \frac{y^\star(\tilde{s},a)}{\mu^\star(\tilde{s})},
    \pihstar(a\mid\tilde{s})
    =
    \frac{\yhat^\star(\tilde{s},a)}{\muhat^\star(\tilde{s})},
\)
and
\(
    |\muhat^\star(\tilde{s})-\mu^\star(\tilde{s})|
    \le
    \sum_{a\in\{0,1\}}
    |y^\star(\tilde{s},a)-\yhat^\star(\tilde{s},a)|,
\)
we obtain 
\[
    |\pistar(a\mid\tilde{s})-\pihstar(a\mid\tilde{s})|
    \le
    \frac{1}{16S\tau(5+\log_2S)},
    \qquad a\in\{0,1\}.
\]
Therefore,
\[
    \infinfnorm{P^{\pistar}-P^{\pihstar}}
    \le
    \sum_{a\in\{0,1\}}
    |\pistar(a\mid\tilde{s})-\pihstar(a\mid\tilde{s})|
    \le
    \frac{1}{8S\tau(5+\log_2S)}.
\]
On the other hand, by convexity of the \(\ell_1\) norm and the fifth term in the
definition of \(\delta_{\min}\), 
\[
    \infinfnorm{P^{\pihstar}-\Phat^{\pihstar}}
    \le
    \max_{s,a}\|P(\cdot\mid s,a)-\Phat(\cdot\mid s,a)\|_1
    \le
    \delta
    \le
    \frac{1}{8S\tau(5+\log_2S)}.
\]
Combining the last two displays gives
\(
    \infinfnorm{P^{\pistar}-\Phat^{\pihstar}}
    \le
    \frac{1}{4S\tau(5+\log_2S)}.
\)
Lemma~\ref{lem:mixing_time_pert_bound}, applied with
\(\mathcal Z=\States\), \(Q=P^{\pistar}\), and
\(\widetilde Q=\Phat^{\pihstar}\), implies that the Markov chain induced by
\(\pihstar\) under \(\Phat\) is aperiodic and unichain, with mixing time at
most \((2+\log_2S)\tau\). 

By the definition of $\muhst$, it is easy to verify that $\muhst$ is a stationary distribution of the Markov chain induced by
\(\pihstar\) under \(\Phat\). In addition, $\muhst$ is the unique stationary distribution since this Markov chain is unichain. Part 1 shows
\(\muhat(s)>0\) for all \(s\in\States\). Since the chain has a single
recurrent class and its stationary distribution puts positive mass on
every state, no state can be transient. Hence the chain is irreducible. This proves part 6
and completes the proof of
Theorem~\ref{thm:LP_perturbation_main_complete}. \Halmos
\endproof

\section{Preliminaries for the proof of Theorem~\ref{thm-homoRB}}
\subsection{Two-set Policy}
\label{subsec:two-set-policy}
We first provide a description of two-set policy \citep{hong2024achieving}, as presented in Algorithm~\ref{alg:two-set-policy}.
\paragraph{Two dynamic sets.}
At each time $t$, the two-set policy maintains disjoint subsets:
(i) Set $D_t^{\mathrm{OL}}$ on which we run \emph{Optimal Local Control} (OLC) to obtain near-optimal
reward; (ii) Set $D_t^{\pi^\star}$ on which we run \emph{Unconstrained Optimal Control} (UOC)
to steer arms toward $\mu^\star$ and eventually merge them into $D_t^{\mathrm{OL}}$.
The remaining arms act as a buffer to satisfy the global constraint $\sum_i A_t(i)=\alpha N$ exactly.
\begin{algorithm}[htbp]
\caption{Two-set policy}\label{alg:two-set-policy}
\begin{algorithmic}[1]
\Require LP solution \(y^\star\), induced policy \(\pi^\star\), feasibility parameters \((\eta,\epsilon_N^{\mathrm{fe}})\), and tolerance \(\epsilon_N^{\mathrm{rd}}\)
\Ensure actions \((A_t(i))_{i\in[N]}\) at each time \(t\)
\State Initialize \(D_{-1}^{\mathrm{OL}}\gets\emptyset\) and \(D_{-1}^{\pi^\star}\gets\emptyset\).
\State Set \(\omega\gets\min\{\alpha,1-\alpha\}\).
\For{\(t=0,1,2,\ldots\)}
    \State Select \(D_t^{\mathrm{OL}}\): if \(\delta(\bm X_t,[N])\ge0\), set \(D_t^{\mathrm{OL}}\gets[N]\); otherwise choose an \(\epsilon_N^{\mathrm{rd}}\)-maximal feasible set, preserving \(D_{t-1}^{\mathrm{OL}}\) when it remains feasible.
    \State Choose \(D_t^{\pi^\star}\subseteq [N]\setminus D_t^{\mathrm{OL}}\) with \(|D_t^{\pi^\star}|=\left\lfloor \omega\bigl(N-|D_t^{\mathrm{OL}}|\bigr)\right\rfloor\), such that either \(D_t^{\pi^\star}\supseteq D_{t-1}^{\pi^\star}\setminus D_t^{\mathrm{OL}}\) or \(D_t^{\pi^\star}\subseteq D_{t-1}^{\pi^\star}\setminus D_t^{\mathrm{OL}}\).
    \State Apply Optimal Local Control on \(D_t^{\mathrm{OL}}\), Unconstrained Optimal Control on \(D_t^{\pi^\star}\), and use the remaining arms as a buffer to enforce \(\sum_i A_t(i)=\alpha N\).
\EndFor
\end{algorithmic}
\end{algorithm}
\paragraph{Subroutines.}
For a subset $D$ with state counts $z_D(s)=|\{i\in D:S_t(i)=s\}|$:
(i) \emph{Unconstrained Optimal Control} activates an expected $\pi^\star(1\mid s)\,z_D(s)$ arms in each state $s$
(using randomized rounding to handle integer effect);
(ii) \emph{Optimal Local Control} draws a target budget $B\in\{\lfloor \alpha|D|\rfloor,\lceil \alpha|D|\rceil\}$ with
$\mathbb E[B]=\alpha|D|$, and then allocates actions in $D$ following a priority rule based on the LP solution $y^\star$ and using the neutral state $\tilde s$ to meet the budget constraint.

\paragraph{Augmented state space.}
Recall from Section~\ref{sec:proof-lem-rhohrel-rhopihat-homo} that
\(\bm X_t\) denotes the map \(D\mapsto\bm X_t(D)\).
The two-set policy induces a finite-state Markov chain with augmented state \(\Sigma_t=\bigl(\bm X_t,D_t^{\mathrm{OL}},D_t^{\pistar}\bigr).\) 

\subsection{\texorpdfstring{Weighted $\ell_2$ norms for quantifying distributional convergence}{Weighted l2 norms for quantifying distributional convergence}}
\label{subsec:weighted_L2}

A recurring theme in our analysis is to control how fast the empirical state distribution of a collection of arms approaches a target stationary distribution. To this end, we follow the weighted-$\ell_2$ approach and introduce two matrices, $W$ and $U$,
which  induce two weighted norms on $\mathbb{R}^{S}$~\citep{hong2024achieving}. 
Recall that $\mu^\star$ is the stationary distribution of the Markov chain induced by $\pistar$ under $P$, and the matrix $\Phi$ is defined in~\eqref{def:phi}. 

\begin{definition}
\label{def:WU}
Define $W,U\in\mathbb{R}^{S\times S}$ by
\[
    W \triangleq \sum_{k=0}^{\infty} (P^{\pi^\star}-\one\mu^\star)^k\bigl((P^{\pi^\star}-\one\mu^\star)^{\top}\bigr)^k,
    \quad
    U \triangleq \sum_{k=0}^{\infty} \Phi^k(\Phi^{\top})^k. 
\]
\end{definition}

The matrices $W$ and $U$ satisfy the identities
$
    W \;=\; I + (P^{\pi^\star}-\one\mu^\star)W(P^{\pi^\star}-\one\mu^\star)^{\top}
$
and
$
    U \;=\; I + \Phi\,U\,\Phi^{\top}.
$
They induce the weighted $\ell_2$ norms
$
    \|v\|_{W} \;\triangleq\; \sqrt{v^{\top}Wv}
$
and
$
    \|v\|_{U} \;\triangleq\; \sqrt{v^{\top}Uv}
$
for  $v\in\mathbb{R}^{S}.$
The local stability condition in Assumption~\ref{ass:RB} implies that the moduli of all eigenvalues of $\Phi$ are less than 1. Hence Lemma 4 from \citet{hong2024achieving} guarantees that $W$ and $U$ are symmetric positive definite and the eigenvalues of $W, U$ are in the range $[1,\lambda_W]$ and $ [1,\lambda_U]$, respectively, where $\lambda_W$ and $\lambda_U$ are the largest eigenvalues of $W$ and $U$. 
In particular, for all $v\in\mathbb{R}^{S}$, it holds that
$
\|v\|_2 \;\le\; \|v\|_W \;\le\; \sqrt{\lambda_W}\,\|v\|_2
$
and
$
\|v\|_2 \;\le\; \|v\|_U \;\le\; \sqrt{\lambda_U}\,\|v\|_2,
$
which will be repeatedly used to translate between Euclidean and weighted metrics.

\subsection{Subset Lyapunov functions}
We used the following two subset Lyapunov functions which are introduced in \citet{hong2024achieving} as building blocks in our Lyapunov analysis. These two classes of functions quantify the convergence of state distributions for subsets of arms that follow the \emph{Unconstrained Optimal Control} or \emph{the Optimal Local Control}. For any system state $\bmx$ and subset $D\subseteq[N]$, let
$z(s)$ be the number of arms in $D$ currently in state $s\in\mathcal S$.
Define the normalized state-count vector $\bmx(D)\in\mathbb R^{S}$ by
$\bmx(D,s)=\frac{z(s)}{N}$, and $m(D)=\sum_{s\in\mathcal S}\bmx(D,s)=\frac{|D|}{N}$. Then define
\begin{equation}\label{def:hu-hw}
    \begin{aligned}
        h_W(\bmx,D)=\|\bmx(D)-m(D)\mu^\star\|_W,
        \qquad
        h_U(\bmx,D)=\|\bmx(D)-m(D)\mu^\star\|_U.
    \end{aligned}
\end{equation}
\subsection{Feasibility-ensuring pair}
Here we briefly introduce the feasibility-ensuring pair showed in \citet{hong2024achieving}. Recall from
Section~\ref{subsec:two-set-policy} that, for a subset
\(D\subseteq[N]\), the \emph{Optimal Local Control} routine chooses a
target number of active arms
\(B\in\{\lfloor\alpha|D|\rfloor,\lceil\alpha|D|\rceil\}\).
Let
\(z_D(s)\triangleq|\{i\in D:S_t(i)=s\}|\)
be the number of arms in \(D\) that are in state \(s\). Define the set of transient states as $\mathcal S^{\emptyset}\triangleq\{s\in\mathcal S:\; y^\star(s,0)=y^\star(s,1)=0\}.$
A sufficient condition ensuring that the \emph{Optimal Local Control} routine can always realize \emph{exactly} $B$ activations is 
\begin{equation}\label{eq:integer-feasibility}
    \sum_{s\neq\tilde s}
    \pi^\star(1\mid s)z_D(s)
    \le
    \alpha|D|-|\mathcal S^{\emptyset}|-1,
    \qquad
    \sum_{s\neq\tilde s}
    \pi^\star(0\mid s)z_D(s)
    \le
    (1-\alpha)|D|-|\mathcal S^{\emptyset}|-1.
\end{equation}
Intuitively, \eqref{eq:integer-feasibility} guarantees that after rounding decisions for states
$s\neq \tilde s$, the remaining budget is nonnegative and does not exceed the available number of arms in the neutral state, so one can use neutral state $\tilde s$ to meet the budget constraint.

For parameters $(\eta,\enfe)$, define the feasibility slack
$
    \delta(\bmx,D)\triangleq \eta\,m(D)-\enfe-\|\bmx(D)-m(D)\mu^\star\|_U.
$

\begin{definition}[Feasibility-ensuring pair]\label{def:fe-pair}
A pair $(\eta,\enfe)$ is called \emph{feasibility-ensuring} if 
\begin{equation*}
    \delta(\bmx,D)\ge 0
    \quad\Longrightarrow\quad
    \text{$D$ satisfies the integer feasibility condition \eqref{eq:integer-feasibility}.}
\end{equation*}
Equivalently, whenever the subset-level state composition $\bmx(D)$ is sufficiently close 
to the target profile $m(D)\mu^\star$ in $\|\cdot\|_U$ with margin $\eta m(D)-\enfe$, the \emph{Optimal Local Control} routine is guaranteed
to be feasible on $D$.
\end{definition}
In our case, $\States^{\emptyset}=\emptyset$. We use the following  explicit choice 
\begin{equation}\label{eq:explicit-fe-pair}
\eta = S^{-1/2}\min\{y^\star(\tilde s,0),\,y^\star(\tilde s,1)\},
\qquad
\enfe = S^{-1/2}\frac{|\mathcal S^{\emptyset}|+1}{N}=S^{-1/2}\frac{1}{N}.
\end{equation}
Lemma 2 in~\citet{hong2024achieving} shows that this choice of $(\eta,\enfe)$ is a \emph{feasibility-ensuring pair}.

\subsection{Maximal Feasible Set}\label{subsec:max-feasible}

In two-set policy, at each time step it selects a subset of arms on which a local control routine is executed. Rather than requiring an exact maximal feasible set, we work with an approximate notion of maximality, parameterized by a small
tolerance $\epsilon_N^{\mathrm{rd}}\ge 0$. This tolerance allows mild slack in the maximality requirement while still ensuring that the chosen set is nearly as large as any maximal feasible set.
\begin{definition}[$\epsilon_N^{\mathrm{rd}}$-maximal feasible set]
\label{def:eps-max-feasible}
Fix a system state $\bmx$ and a tolerance parameter $\epsilon_N^{\mathrm{rd}}\ge 0$.
A set $D\subseteq[N]$ is called $\epsilon_N^{\mathrm{rd}}$-maximal feasible at state $\bmx$ if the following hold:
\textbf{1. Feasibility:} $\delta(\bmx,D)\ge 0$, or $D=\emptyset$;
\textbf{2. Approximate maximality:} for any $D'$ such that $D\subseteq D'\subseteq[N]$ and
  $\delta(\bmx,D')\ge \epsilon_N^{\mathrm{rd}}$, we have
  $
    m(D') \le m(D)+\epsilon_N^{\mathrm{rd}}.
  $
\end{definition}

When $\epsilon_N^{\mathrm{rd}}=0$, Definition~\ref{def:eps-max-feasible} reduces to the usual notion of a
maximal element (under set inclusion) in the family
$\{D'\subseteq[N]:\delta(\bmx,D')\ge 0\}\cup\{\emptyset\}$.
Throughout the theoretical development we fix $\epsilon_N^{\mathrm{rd}}=\epsilon_N^{\mathrm{fe}}=S^{-1/2}/N$.

\subsection{Lyapunov functions}\label{subsec:lyf-RB}

In this subsection we introduce our Lyapunov function, which is constructed  based on the Lyapunov function $\widetilde{V}$ introduced in \citet{hong2024achieving}. For any system state $\sigma=(\bmx,d^{\mathrm{OL}},d^{\pistar})$, define
\begin{equation}\label{eq:def-Vtilde}
    \widetilde{V}(\sigma)=\widetilde{V}(\bmx,d^{\mathrm{OL}},d^{\pistar})\triangleq h_U(\bmx,d^{\mathrm{OL}})+h_W(\bmx,d^{\pistar})+L(1-m(d^{\mathrm{OL}})),
\end{equation}
where $h_U, h_W$ are defined in \eqref{def:hu-hw}, \(L = 2\lambda_U^{1/2}+4\lwhalf-2\lwhalf\min(\alpha,1-\alpha)\).
It has been shown in \citet{hong2024achieving} that 
\begin{equation}\label{eq:def-Vtilde_max}
    0\le \widetilde{V}(\sigma)\le \widetilde{V}_{\max}\triangleq\lambda_U^{1/2}+\lwhalf+L, \quad\forall \sigma.
\end{equation}
Define our Lyapunov function $V$ as
\begin{equation}\label{equ:defi-V}
V(\sigma)=V(\bmx,d^{\mathrm{OL}},d^{\pistar})\triangleq\left(\widetilde{V}(\sigma)-\frac{\bar\eta}{2}\right)^++\frac{\gamma}{4(\beta+r_{\max})}(\mu^\star-\bmx([N]))Qb,
\end{equation}
where 
\begin{equation}\label{eq:def-beta-Q-b-bareta-gamma}
    \begin{gathered}
        \beta=8r_{\max}\sqrt S\lambda_U, 
        \quad 
        b=\left(\sum_{a\in\Actions}(r(s,a)-r(\tilde s,a))\pistar(a\mid s)\right)_{s\in\States}\in\mathbb{R}^S,\quad Q=(I-\Phi)^{-1},\\
        \gamma=\frac{3\eta}{4K_{Vh}}, \quad K_{Vh}=2\max\left\{\lambda_U,\lambda_W\left(1+\frac{L\lambda_U^{1/2}}{\eta\omega}\right)\right\} 
        ,\quad \bar\eta=\eta-\enfe, \quad \omega=\min\{\alpha,1-\alpha\}.
    \end{gathered}
\end{equation}
Recall that $\Phi$ is defined in \eqref{def:phi} and $\eta,\enfe$ are defined in \eqref{eq:explicit-fe-pair}. For a short notation, define
\begin{equation}
    y_{\mathrm{neu}}\triangleq\min\{y^\star(\tilde s,0),y^\star(\tilde s,1)\}.
\end{equation}
Notice that when $N\ge 4/\yneu$, we have $\bar\eta\ge\frac{3}{4}\eta>0$. Since $\rho(\Phi)<1$, $U\succeq I$, and
$U-\Phi U\Phi^\top=I$,
We apply Lemma~\ref{lem:gramian-inverse-bound} to
$(A,G)=(\Phi,U)$ and gives
$\infinfnorm{Q}=\infinfnorm{(I-\Phi)^{-1}}\le2\sqrt S\lambda_U$.
Since $\infnorm{b}\le r_{\max}$ and
$\infinfnorm{\mu^\star-\bmx([N])}\le 2$,
\begin{equation}\label{eq:sup-sigma-le-beta}
        \sup_\sigma\left|(\mu^\star-\bmx([N]))Qb-\E^{\pits}\left[(\mu^\star-\bm{X}_{t+1}([N]))Qb\mid \Sigma_t=\sigma\right]\right|
    \le 4\infinfnorm{Q}\infnorm{b}\le \beta.
\end{equation}
Following the similar argument as \eqref{eq:sup-sigma-le-beta} we have $\forall \sigma$,
\begin{equation}\label{eq:Vmax-def}
    \begin{aligned}
        \left|V(\sigma)\right|&=\left|\left(\widetilde{V}(\sigma)-\frac{\bar\eta}{2}\right)^++\frac{\gamma}{4(\beta+r_{\max})}(\mu^\star-\bmx([N]))Qb\right| \\
        &\le \widetilde{V}_{\max}+\frac{\gamma}{4r_{\max}}\infinfnorm{\mu^\star-\bmx([N])}\infinfnorm{Q}\infnorm{b}  \\
        &\le\widetilde{V}_{\max}+2\gamma\sqrt S\lambda_U\triangleq V_{\max}.
    \end{aligned}
\end{equation}

\section{Proof details of the lemmas for Theorem~\ref{thm-homoRB}}
\label{app:proof-thm-homoRB}

\subsection{Proof of Lemma~\ref{lem:rhorel-rhohpi-homo}}\label{subsec:pr-C1-homo}

\begin{proof}{Proof.}
We use the Lyapunov function $V$ defined in~\eqref{equ:defi-V}.
The following two lemmas provide the Conditions \ref{cond:drift} and \ref{cond:dominance} in our Lyapunov framework. 
\begin{lemma}\label{lem:DC-RB}
    When $N\ge\max\left\{\frac{4}{\yneu},\left(\frac{4M}{\gamma}\right)^2\right\}$, the Lyapunov function $V$ defined in \eqref{equ:defi-V} satisfies 
    \[
        \E^{\pits}[V(\Sigma_{t+1})\mid \Sigma_t=\sigma]-V(\sigma)\le -\frac{\gamma}{4(\beta+r_{\max})}g(\sigma)+6\exp(8)S\widetilde{V}_{\max}\exp\left(-CN\right), \qquad \forall \sigma,
    \]
    where $\gamma,\beta$ are defined in \eqref{eq:def-beta-Q-b-bareta-gamma}, $\widetilde{V}_{\max}$ is defined in~\eqref{eq:def-Vtilde_max},  and
    \begin{equation}\label{def:g-sigma}
        g(\sigma)=r_{\max}\mathbb{I}[d^{\mathrm{OL}}\neq [N]]
+ (\mu^\star-\bmx([N]))b\cdot\mathbb{I}[d^{\mathrm{OL}}= [N]],
    \end{equation}
    \begin{equation}\label{def:g-C-M}
        \begin{aligned}
    M&=\frac{L}{2K_{Vh}\omega}\left(1+ \frac{2S^{-\frac{1}{2}}}
         {\eta}\right)+5\luhalf S+3\lwhalf S 
         +(\lwhalf+\luhalf)\exp\left(\frac{1}{2\lambda_U^3S^3}+32\right)\frac{16\lambda_U^{3/2}S^{2}}{\eta} +2\lwhalf,\\
         C&=\frac{\min\{\gamma^2/16,1\}}{4\max\{\lambda_W,\lambda_U\}S^2}.
        \end{aligned}
    \end{equation}
\end{lemma}
\begin{proof}{Proof of Lemma~\ref{lem:DC-RB}.}
We have 
    \begin{align*}
        & \E^{\pits}[V(\Sigma_{t+1})\mid \Sigma_t=\sigma]-V(\sigma)\\
        = & \E^{\pits}\left[\left(\widetilde{V}(\Sigma_{t+1})-\frac{\bar\eta}{2}\right)^+\mid \Sigma_t 
        = \sigma\right]-\left(\widetilde{V}(\sigma)-\frac{\bar\eta}{2}\right)^+ \\
        &+\frac{\gamma}{4(\beta+r_{\max})}\left(\E^{\pits}\left[(\mu^\star-\bm{X}_{t+1}([N]))Qb\mid \Sigma_t=\sigma\right]-(\mu^\star-\bmx([N]))Qb\right) \\
        \le & -\frac{\gamma}{4}\mathbb{I}[d^{\mathrm{OL}}\neq [N]]+6\exp(8)S\widetilde{V}_{\max}\exp\left(-CN\right) \\
        &+\frac{\gamma}{4(\beta+r_{\max})}\left(\E^{\pits}\left[(\mu^\star-\bm{X}_{t+1}([N]))Qb\mid \Sigma_t=\sigma\right]-(\mu^\star-\bmx([N]))Qb\right) \\
        = & -\frac{\gamma}{4}\mathbb{I}[d^{\mathrm{OL}}\neq [N]] 
        +6\exp(8)S\widetilde{V}_{\max}\exp\left(-CN\right) \\
        & -\frac{\gamma}{4(\beta+r_{\max})}\left((\mu^\star-\bmx([N]))Qb-\E^{\pits}\left[(\mu^\star-\bm{X}_{t+1}([N]))Qb\mid \Sigma_t=\sigma\right]\right)\mathbb{I}[d^{\mathrm{OL}}\neq [N]] \\
        & -\frac{\gamma}{4(\beta+r_{\max})}\left((\mu^\star-\bmx([N]))Qb-\E^{\pits}\left[(\mu^\star-\bm{X}_{t+1}([N]))Qb\mid \Sigma_t=\sigma\right]\right)\mathbb{I}[d^{\mathrm{OL}}=[N]] \\
        \le &  6\exp(8)S\widetilde{V}_{\max}\exp\left(-CN\right)
        -\frac{\gamma}{4(\beta+r_{\max})}\biggl(r_{\max}\mathbb{I}[d^{\mathrm{OL}}\neq[N]]\\
        &\qquad +\left((\mu^\star-\bmx([N]))Qb  -\E^{\pits}\left[(\mu^\star-\bm{X}_{t+1}([N]))Qb\mid \Sigma_t=\sigma\right]\right)\mathbb{I}[d^{\mathrm{OL}}=[N]]\biggr),
    \end{align*}
    where the first inequality follows from Lemma~\ref{lem:drift-fV} and the second inequality from \eqref{eq:sup-sigma-le-beta}.
    Note that $Q=(I-\Phi)^{-1}$, thus
    $
        \mu^\star-\bmx([N])=(\mu^\star-\bmx([N]))Q-(\mu^\star-\bmx([N]))\Phi Q.
    $
    Hence,
    \begin{align*}
        &\left((\mu^\star-\bmx([N]))Qb-\E^{\pits}[(\mu^\star-\bm{X}_{t+1}([N]))Qb\mid \Sigma_t=\sigma]\right)\cdot\mathbb{I}[d^{\mathrm{OL}}= [N]] \\
        =& (\mu^\star-\bmx([N]))Qb\cdot\mathbb{I}[d^{\mathrm{OL}}= [N]]-(\mu^\star-\bmx([N]))\Phi Q b\cdot\mathbb{I}[d^{\mathrm{OL}}= [N]] \\
        =& (\mu^\star-\bmx([N]))b\cdot\mathbb{I}[d^{\mathrm{OL}}= [N]], 
    \end{align*}
    where the first equality follows from Lemma~\ref{lem:lemma1-hong}.
    Therefore, we have the desired drift bound
    $
         \E^{\pits}[V(\Sigma_{t+1})\mid \Sigma_t=\sigma]-V(\sigma) 
         \le -\frac{\gamma}{4(\beta+r_{\max})}g(\sigma)+6\exp(8)S\widetilde{V}_{\max}\exp\left(-CN\right).
    $ \Halmos
\end{proof}
\begin{lemma}\label{lem:C2-RB}
Consider an $N$-armed RB problem. For each \( \sigma=(\bmx,d^{\mathrm{OL}},d^{\pistar})\), define \(r^{\pits}(\sigma)\triangleq\frac{1}{N}\E_{\bma\sim\pits(\cdot\mid \sigma)}[\sum_{i\in[N]} r(s_i,a_i)]\).
Then
    $
        \rho^{\mathrm{rel}}-r^{\pits}(\sigma)\le g(\sigma),
    $
    where $g(\sigma)$ is defined in \eqref{def:g-sigma}. 
\end{lemma}
\begin{proof}{Proof of Lemma~\ref{lem:C2-RB}.}
    Consider the affine function introduced in \citet{hong2024achieving}: 
    \begin{align*}
        \bar r^{\pits}(\bmx)=r^{\pits}(\bmx,[N],\emptyset) & = \sum_{a\in\Actions}\langle \bmx([N]),
(r(\cdot,a)-r(\tilde s,a))\pistar(a\mid \cdot)\rangle+\alpha r(\tilde s,1)+(1-\alpha)r(\tilde s,0)\\
        &=\langle\bmx([N]),b\rangle+\alpha r(\tilde s,1)+(1-\alpha)r(\tilde s,0).
    \end{align*}
    Notice that $\rho^{\mathrm{rel}}=\bar r^{\pits}(\mu^\star)$, hence
    \begin{align*}
        \rho^{\mathrm{rel}}-r^{\pits}(\sigma) & =\bar r^{\pits}(\mu^\star)-r^{\pits}(\sigma) \\
         & =(\bar r^{\pits}(\mu^\star)-r^{\pits}(\sigma))\mathbb{I}\big[d^{\mathrm{OL}}\ne [N]\big] + (\bar r^{\pits}(\mu^\star)-r^{\pits}(\sigma))\mathbb{I}\big[d^{\mathrm{OL}}= [N]\big] \\
         & \le r_{\max}\mathbb{I}\big[d^{\mathrm{OL}}\ne [N]\big]+ (\mu^\star-\bmx([N]))b\mathbb{I}\big[d^{\mathrm{OL}}= [N]\big]
          =g(\sigma).  \Halmos
    \end{align*} 
\end{proof}

We now apply Lemma~\ref{lem:generic-drift-transfer} with source kernel $Q=\bp^{\pits}$ and target kernel $P=\bphat^{\pits}$, benchmark $R=\rhorel$, function $F=g$, one-period reward $\bar r=r^{\pits}$, and
\(
a=\frac{\gamma}{4(\beta+r_{\max})}, b=6\exp(8)S\widetilde{V}_{\max}\exp(-CN), c=1, d=0.
\)
Lemmas~\ref{lem:DC-RB} and~\ref{lem:C2-RB} provide the required drift and dominance inequalities. Therefore,
\[
    \rhorel-\widehat\rho^{\pits}(\bms_0)
    \le
    \frac{4(\beta+r_{\max})}{\gamma}\sup_\sigma |(\bphat^{\pits}-\bp^{\pits})V(\sigma)|
    +\frac{4(\beta+r_{\max})}{\gamma}\,6\exp(8)S\widetilde{V}_{\max}\exp(-CN).
\]
By Lemma~\ref{lem:phat-p-V}, the model-error term is at most $V_{\max}N\delta$, where $V_{\max}$ is defined in \eqref{eq:Vmax-def}. Hence
\[
    \rhorel-\widehat\rho^{\pits}(\bms_0)
    \le
    \frac{4V_{\max}(\beta+r_{\max})}{\gamma}N\delta
    +\frac{4(\beta+r_{\max})}{\gamma}6\exp(8)S\widetilde{V}_{\max}\exp(-CN).
\]
Defining 
\begin{equation}\label{eq:def-C6}
    C_6=\frac{4V_{\max}(\beta+r_{\max})}{\gamma},\; C_6^{\mathrm{exp}}=\frac{24\exp(8)(\beta+r_{\max})S\widetilde{V}_{\max}}{\gamma},\;
    C_6^{\mathrm{rate}}=\frac{\min\{\gamma^2/16,1\}}{4\max\{\lambda_W,\lambda_U\}S^2}
\end{equation}
completes the proof of Lemma~\ref{lem:rhorel-rhohpi-homo}. \Halmos
\end{proof}

\subsection{Proof of Lemma~\ref{lem:rhohrel-rhopihat-homo}}\label{subsec:pr-C2-homo}

Lemma~\ref{lem:rhohrel-rhopihat-homo} bounds $\rhohrel-\rho^{\pihts}(\bms_0)$, which is the second part in the decomposition~\eqref{eq:decomp}. As a result, the analysis in this part is carried out under the \emph{empirical} model. 
To streamline the discussion, we first introduce the key objects and notation for the empirical system in Appendix~\ref{sssec:defs-hat-homo}; their well-definedness and basic properties are established in Appendix~\ref{subsec:wd-pro}.

\subsubsection{Definitions}\label{sssec:defs-hat-homo}

As in the proof of Lemma~\ref{lem:rhohrel-rhopihid-hete}, to apply our Lyapunov framework we first need to ensure that the Markov chain induced by the policy $\pihstar$ under the empirical kernel $\Phat$ remains ergodic, which is shown in Theorem~\ref{thm:LP_perturbation_main_complete}. 

\paragraph{Weighted $\ell_2$ norms}
Similar to Definition~\ref{def:WU}, we define following matrices for the empirical system:
\begin{definition} 
\label{def:WhatUhat}
Define $\widehat W,\widehat U\in\mathbb{R}^{S\times S}$ by
\[
    \widehat W \triangleq \sum_{k=0}^{\infty} (\widehat P^{\pihstar}-\one\muhst)^k\bigl((\widehat P^{\pihstar}-\one\muhst)^{\top}\bigr)^k,
    \quad\widehat U \triangleq \sum_{k=0}^{\infty} \widehat \Phi^k(\phihat^{\top})^k.
\]

\paragraph{Subset Lyapunov functions}
For the empirical system, we can naturally define 
\begin{equation}\label{def:huh-hwh}
    \begin{aligned}
        h_{\widehat W}(\bmx,D) =\|\bmx(D)-m(D)\muhst\|_{\widehat{W}},
        \qquad
        h_{\uhat}(\bmx,D) =\|\bmx(D)-m(D)\muhst\|_{\uhat}
    \end{aligned}
\end{equation}
Furthermore, denote $\lwhat$ and $\luhat$ as the largest eigenvalues of $\widehat W, \uhat$.
\end{definition}

\paragraph{Feasibility-ensuring pair} Under the condition of Theorem~\ref{thm:LP_perturbation_main_complete} and according to the construction of the empirical primal optimal solution, we know
$
    \{s\in\mathcal S:\; y^\star(s,0)=y^\star(s,1)=0\}=\{s\in\mathcal S:\widehat y^\star(s,0)=\widehat y^\star(s,1)=0\}.
$
So in terms of the empirical model, we can choose the feasibility-ensuring pair $(\widehat\eta,\widehat\epsilon_N^{\mathrm{fe}})$ as:
$
\widehat\eta = S^{-1/2}\min\{\widehat y^\star(\tilde s,0),\widehat y^\star(\tilde s,1)\}
$
and
$
\widehat \epsilon_N^{\mathrm{fe}}=\enfe = S^{-1/2}\frac{|\mathcal S^{\emptyset}|+1}{N}=S^{-1/2}\frac{1}{N}.
$

\paragraph{Lyapunov functions}
For the empirical model, we define the matrix $\phihat$ as
$
    \phihat \triangleq \Phat_{\pihstar}-\one\muhst
    -\bigl(c_{\pihstar}-\alpha\mathbf 1\bigr)\bigl(\Phat_{\tilde s 1}-\Phat_{\tilde s 0}\bigr).
$
For any system state $\sigma=(\bmx,d^{\mathrm{OL}},d^{\pistar})$, we make use of the following empirical version of the Lyapunov function introduced in~\citet{hong2024achieving}:
\begin{equation*}
    \widehat{\widetilde{V}}(\sigma)=\widehat{\widetilde{V}}(\bmx,d^{\mathrm{OL}},d^{\pistar})=h_{\uhat}(\bmx,d^{\mathrm{OL}})+h_{\widehat{W}}(\bmx,d^{\pistar})+\widehat L(1-m(d^{\mathrm{OL}})),
\end{equation*}
where $h_{\uhat}, h_{\widehat W}$ are defined in \eqref{def:huh-hwh}, \(\widehat L = 2\luhhalf+4\lwhhalf-2\lwhhalf\min(\alpha,1-\alpha)\le 2\luhhalf+4\lwhhalf\).
Now we can define the empirical Lyapunov function $\Vhat$ as:
\begin{equation}\label{equ:defi-Vhat}
    \Vhat(\sigma)=\Vhat(\bmx,d^{\mathrm{OL}},d^{\pistar})=\left(\widehat {\widetilde{V}}(\sigma)-\frac{\widehat{\bar\eta}}{2}\right)^++\frac{\widehat\gamma}{4(\widehat \beta+r_{\max})}(\muhst-\bmx([N]))\Qhat \widehat b
\end{equation}
where 
\begin{equation}\label{eq:def-hat-beta-Q-b-bareta-gamma}
    \begin{gathered}
        \widehat \beta=8r_{\max}\sqrt S\lambda_{\uhat},\quad
        \widehat b=\left(\sum_{a\in\Actions}(r(s,a)-r(\tilde s,a))\pihstar(a\mid s)\right)_{s\in\States}\in\mathbb{R}^S,\quad \Qhat=(I-\phihat)^{-1},\\
        \widehat\gamma=\frac{3\widehat{\eta}}{4\widehat K_{Vh}},
        \quad \widehat K_{Vh}=2\max\left\{\luhat,\lwhat\left(1+\frac{\widehat L\luhhalf}{\widehat\eta\omega}\right)\right\},\quad \widehat{\bar\eta}=\widehat{\eta}-\enfe, 
        \quad \omega=\min\{\alpha,1-\alpha\}.
    \end{gathered}
\end{equation}
For a short notation, define
\begin{equation}
    \yneuhat\triangleq\min\{\yhat^\star(\tilde s,0),\yhat^\star(\tilde s,1)\}.
\end{equation}
Theorem \ref{thm:LP_perturbation_main_complete} gives $\yneuhat\ge\yneu/2$. Hence, when $N\ge 8/\yneu$, $\widehat{\bar\eta}\ge\frac{3}{4}\widehat \eta>0$.
\subsubsection{Well-definedness and properties}\label{subsec:wd-pro}

Lemma~\ref{lem:wd-what-uhat} guarantees that $\widehat W, \uhat$ are well-defined and their eigenvalues are in the range of $[1,\lwhat]$ and $[1,\luhat]$.
\begin{lemma}\label{lem:wd-what-uhat}
    Under the conditions in Lemma~\ref{lem:rhohrel-rhopihat-homo}, the matrices $\widehat W, \uhat$ given in Definition~\ref{def:WhatUhat} are well-defined. Moreover, their eigenvalues are lower bounded by $1$.
\end{lemma}

We can establish the well-definedness of $\widehat{W}$ and lower bound its eigenvalues, following the same argument as Lemma~4 of \citet{hong2024achieving}. The key is that by Theorem~\ref{thm:LP_perturbation_main_complete}, all eigenvalues of $\Phat^{\pihstar}-\one\muhst$ have modulus strictly smaller than $1$, and the same property also holds for $\phihat$. Therefore, the proof for $\uhat$ proceeds analogously, with $\Phat^{\pihstar}-\one\muhst$ replaced by $\phihat$. 

Part 3 of Theorem~\ref{thm:LP_perturbation_main_complete} directly proves that
\begin{equation}\label{def:ubluhat}
\lambda_{\widehat U}\le\ubluhat\triangleq2\lambda_U, \quad\widehat{\beta}\le\widehat{\beta}^{\mathrm{up}}\triangleq16r_{\max}\sqrt S\lambda_U.
\end{equation}
The following lemma provides an upper bound, denoted as $\ublwhat$, for $\lwhat$. 
\begin{lemma}\label{lem:bound_of_lwhat}
   Under the conditions in Lemma~\ref{lem:rhohrel-rhopihat-homo}, \(\widehat{W}\) satisfies
    \begin{equation}\label{def:ublwhat}
        \lwhat \leq \ublwhat \triangleq 5S(2+\log_2S)\tau,
    \end{equation}
    where \( \tau\) is the mixing time of the single-armed MDP under the optimal single-armed policy $\pi^\star$.
\end{lemma}

\begin{proof}{Proof.}
Note that
$
    \widehat{W} = I + \sum_{t=1}^\infty (\widehat{P}^{\pihstar,t} - \one\muhst) (\widehat{P}^{\pihstar,t} - \one\muhst)^\top
$
    since $(\widehat{P}^{\pihstar} - \one\muhst)^t = \widehat{P}^{\pihstar,t} - \one\muhst$ if $t \geq 1$, and if $t = 0$ then $(\widehat{P}_{\pihstar} - \one\muhst)^t = I$.
Therefore for any $v$ such that $\twonorm{v} = 1$, we have
\begin{align}
    v^\top \widehat{W} v &= 1 + \sum_{t=1}^\infty v^\top(\widehat{P}^{\pihstar,t} - \one\muhst) (\widehat{P}^{\pihstar,t} - \one\muhst)^\top v 
    \leq 1 + \sum_{t=1}^\infty \twonorm{(\widehat{P}^{\pihstar,t} - \one\muhst)^\top v}^2. \label{eq:W_eval_bd_1}
\end{align}
Next we can bound
\begin{align*}
    \twonorm{(\widehat{P}^{\pihstar,t} - \one\muhst)^\top v} &= \sup_{w:\twonorm{w}\leq 1} v^\top (\widehat{P}^{\pihstar,t} - \one\muhst) w \\
    & \leq \sup_{w:\twonorm{w}\leq 1}\infinfnorm{v^\top} \infinfnorm{\widehat{P}^{\pihstar,t} - \one\muhst} \infnorm{w} 
     \leq \sqrt{S} \infinfnorm{\widehat{P}^{\pihstar,t} - \one\muhst},
\end{align*} 
where we use that $\infnorm{w} \leq \twonorm{w}\leq 1$ and that $\infinfnorm{v^\top} = \onenorm{v} \leq \sqrt{S} \twonorm{v}$. According to Theorem~\ref{thm:LP_perturbation_main_complete}, we know that the Markov chain induced by $\pihstar$ under $\Phat$ has mixing time $\tau'$ which is upper bounded by $(2+\log_2 S)\tau$. For \(1\le t<\tau'\), we use
\(\infinfnorm{
\widehat P^{\pihstar,t}-\one\muhst}\le2\). For \(t\ge\tau'\), Lemma~\ref{lem:Qt-Qinfty} gives
\(\infinfnorm{
\widehat P^{\pihstar,t}-\one\muhst}
\le2^{-\lfloor t/\tau'\rfloor}\).
Plugging these into \eqref{eq:W_eval_bd_1} we get 
\begin{align*}
    v^\top\widehat Wv\le
1+4S(\tau'-1)
+
S\sum_{t=\tau'}^{\infty}
2^{-2\lfloor t/\tau'\rfloor}=
1+4S(\tau'-1)
+S\tau'
\sum_{k=1}^{\infty}4^{-k}\le
5S(2+\log_2S)\tau.
\end{align*}
Thus \(\lwhat \leq 5S(2+\log_2S)\tau\). \Halmos
\end{proof}

The following lemma offers a strictly positive upper and lower bound for $\widehat \gamma$.
\begin{lemma}\label{lem:bound-gammahat}
Under the conditions in Lemma~\ref{lem:rhohrel-rhopihat-homo}, we have \(0<\lbgmhat\le\widehat{\gamma}\le\ubgmhat\), where
    \begin{equation}\label{eq:def-lbg-upg-ubKvh}
            \lbgmhat \triangleq \frac{3S^{-1/2}\yneu}{8\ubkhatvh}, 
            \; 
            \ubgmhat \triangleq \frac{9}{16}S^{-1/2}\yneu,
            \; 
            \ubkhatvh \triangleq 2\max\left\{\ubluhat,\ublwhat\left(1+\frac{4(\ublwhat+\ubluhat)\ubluhat}{S^{-1/2}\yneu\omega}\right)\right\},
    \end{equation}
    and $\ublwhat,\ubluhat$ are defined in \eqref{def:ublwhat} and \eqref{def:ubluhat}.
\end{lemma}
\begin{proof}{Proof.}
Firstly using \eqref{def:ubluhat} and \eqref{def:ublwhat} we have 
\[
     \widehat K_{Vh}=2\max\left\{\luhat,\lwhat\left(1+\frac{\widehat L\luhhalf}{\widehat\eta\omega}\right)\right\}
    \le 2\max\left\{\ubluhat,\ublwhat\left(1+\frac{4(\ublwhat+\ubluhat)\ubluhat}{S^{-1/2}\yneu\omega}\right)\right\}\triangleq \ubkhatvh.
\]

It is clear that $\widehat K_{Vh}\ge 2$. According to Theorem~\ref{thm:LP_perturbation_main_complete}, when $\delta\le\delta_{\min}$, we have
\begin{align*}
    \ubgmhat\triangleq\frac{9}{16}S^{-1/2}\yneu
    \ge\widehat \gamma
    =\frac{3S^{-1/2}}{4\widehat K_{Vh}}\min\{\yhat^\star(\tilde{s},0),\yhat^\star(\tilde{s},1)\}
    \ge \frac{3S^{-1/2}\yneu}{8\ubkhatvh}\triangleq\lbgmhat.  \Halmos
    \end{align*}

\end{proof}

In terms of the properties of the Lyapunov function we use, 
\begin{equation}\label{def:hat-tilde-V-max}
    0\le \widehat{\widetilde{V}}(\sigma)\le \lambda_{\widehat U}^{1/2}+\lwhhalf+\widehat L\le 3\lambda_{\widehat U}^{1/2}+5\lwhhalf\le 3\sqrt{\ubluhat}+5\sqrt{\ublwhat}\triangleq \widehat{\widetilde{V}}_{\max}, \qquad\forall \sigma, 
\end{equation}
where $\ubluhat$ is defined in \eqref{def:ubluhat} and $\ublwhat$ is defined in \eqref{def:ublwhat}.
Part 3 of Theorem~\ref{thm:LP_perturbation_main_complete} gives $\rho(\phihat)<1$, so $\widehat Q=(I-\phihat)^{-1}$ is well-defined and $\infinfnorm{\widehat Q}\le4\sqrt S\lambda_U$. 

\subsubsection{Final proof of Lemma~\ref{lem:rhohrel-rhopihat-homo}}

\begin{proof}{Proof.}
We work with the empirical Lyapunov function $\Vhat$ defined in~\eqref{equ:defi-Vhat} (see Appendix~\ref{subsec:lyf-RB}).
The next two lemmas provide the drift bound and the gap-dominance bound in our Lyapunov framework (Conditions~\ref{cond:drift} and~\ref{cond:dominance}), both stated for the empirical system under the policy $\pihts$.
Their proofs follow the same steps as those for Lemma~\ref{lem:DC-RB} and Lemma~\ref{lem:C2-RB}, replacing all quantities by their hatted versions. We omit the details here.

\begin{lemma}\label{lem:DC-Vhat-homo}
    When $\delta\le\delta_{\min}, N\ge\max\left\{\frac{4}{\yneuhat},\left(\frac{4\widehat M}{\widehat \gamma}\right)^2\right\}$, the Lyapunov function defined in \eqref{equ:defi-Vhat} satisfies 
    \[
    \widehat \E^{\pihts}[\Vhat(\Sigma_{t+1})\mid \Sigma_t=\sigma]-\Vhat(\sigma)\le -\frac{\widehat\gamma}{4(\widehat\beta+r_{\max})} \widehat g(\sigma)+6\exp(8)S\widehat{\widetilde{V}}_{\max}\exp\left(-\widehat C N\right),
    \]
    where $\widehat\gamma, \widehat\beta$ are defined in~\eqref{eq:def-hat-beta-Q-b-bareta-gamma}, $\widehat{\widetilde{V}}_{\max}$ is defined in~\eqref{def:hat-tilde-V-max}, and
    \begin{equation}\label{def:hat-g-gamma-C-K-M}
        \begin{gathered}
        \widehat M  =\frac{\widehat L}{2\widehat K_{Vh}\omega}\left(1+ \frac{2S^{-\frac{1}{2}}}
         {\widehat \eta}\right)+5\luhhalf S+3\lwhhalf S 
         +(\lwhhalf+\luhhalf)\exp\left(\frac{1}{2\luhat^3S^3}+32\right)\frac{16\luhat^{3/2}S^{2}}{\widehat \eta} +2\lwhhalf,\\
          \widehat g(\sigma) =r_{\max}\mathbb{I}[d^{\mathrm{OL}}\neq [N]]
             +(\muhst-\bmx([N]))\widehat b\cdot\mathbb{I}[d^{\mathrm{OL}}= [N]],
             \qquad
          \widehat C =\frac{\min\{\widehat \gamma^2/16,1\}}{4\max\{\lwhat,\luhat\}S^2}.
        \end{gathered}
    \end{equation}
\end{lemma}
\begin{lemma}\label{lem:C2-hat-RB}
Suppose the conditions of Lemma~\ref{lem:rhohrel-rhopihat-homo} hold. \(\forall \sigma=(\bmx,d^{\mathrm{OL}},d^{\pistar})\). Define \(r^{\pihts}(\sigma)\triangleq\frac{1}{N}\E_{\bma\sim\pihts(\cdot\mid \sigma)}[\sum_{i\in[N]} r(s_i,a_i)]\). 
Then we have
$
     \widehat \rho^{\mathrm{rel}}-r^{\pihts}(\sigma)\le \widehat g(\sigma),
$
where $\widehat{g}(\sigma)$ is defined in \eqref{def:hat-g-gamma-C-K-M}. 
\end{lemma}

By Lemma~\ref{lem:bound-gammahat}, Lemma~\ref{lem:bound-of-Mhat} and Theorem~\ref{thm:LP_perturbation_main_complete},
\[
    N\ge\max\left\{\frac{8}{\yneu},\left(\frac{4\ubmhat}{\lbgmhat}\right)^2\right\},\delta\le\delta_{\min}
    \quad \Longrightarrow \quad
    N\ge\max\left\{\frac{4}{\yneuhat},\left(\frac{4\widehat M}{\widehat \gamma}\right)^2\right\},
\]
so the condition required by Lemma~\ref{lem:DC-Vhat-homo} holds.

\paragraph{From the empirical drift bound to a true-system bound.}
We apply Lemma~\ref{lem:generic-drift-transfer} with source kernel $Q=\bphat^{\pihts}$ and target kernel $P=\bp^{\pihts}$, benchmark $R=\rhohrel$, function $F=\widehat g$, one-period reward $\bar r=r^{\pihts}$, and \(a=\frac{\widehat\gamma}{4(\widehat\beta+r_{\max})}, b=6\exp(8)S\widehat{\widetilde{V}}_{\max}\exp(-\widehat C N), c=1, d=0.\)

The drift and gap-dominance conditions are Lemmas~\ref{lem:DC-Vhat-homo} and~\ref{lem:C2-hat-RB}; the preceding bounds ensure that their assumptions hold. Lemma~\ref{lem:generic-drift-transfer} gives
\[
    \rhohrel-\rho^{\pihts}(\bms_0)
    \le
    \frac{4(\widehat\beta+r_{\max})}{\widehat\gamma}\sup_\sigma |(\bp^{\pihts}-\bphat^{\pihts})\Vhat(\sigma)|
    +\frac{4(\widehat\beta+r_{\max})}{\widehat\gamma}\,6\exp(8)S\widehat{\widetilde{V}}_{\max}\exp(-\widehat C N).
\]
Notice that here 
\[
\widehat C =\frac{\min\{\widehat \gamma^2/16,1\}}{4\max\{\lwhat,\luhat\}S^2}\ge \frac{\min\{(\lbgmhat)^2/16,1\}}{4\max\{\ublwhat,\ubluhat\}S^2}.
\]
Then using Lemma~\ref{lem:p-phat-vhat}, Lemma~\ref{lem:bound-gammahat}, ~\eqref{def:ubluhat} and defining 
\begin{equation}\label{eq:def-C7}
    C_7=\frac{4\Vhat_{\max}(\widehat\beta^{\mathrm{up}}+r_{\max})}{\lbgmhat},\; C_7^{\mathrm{exp}}=\frac{24\exp(8)(\widehat\beta^{\mathrm{up}}+r_{\max})S\widehat{\widetilde{V}}_{\max}}{\lbgmhat},\;
    C_7^{\mathrm{rate}}=\frac{\min\{(\lbgmhat)^2/16,1\}}{4\max\{\ublwhat,\ubluhat\}S^2},
\end{equation}
completes the proof of Lemma~\ref{lem:rhohrel-rhopihat-homo}. \Halmos
\end{proof}

\subsection{Other Lemmas and Proofs}\label{subsec:other-lem-pf-homo}
This subsection collects several lemmas used in the proofs of
Lemma~\ref{lem:rhorel-rhohpi-homo} and~\ref{lem:rhohrel-rhopihat-homo}. Throughout this subsection, we consider an $N$-armed RB problem with an initial state $\bms_0$ satisfying Assumptions~\ref{ass:RB},~\ref{ass:strict_action_gaps}. 
\begin{lemma}\label{lem:drift-fV}
    When $N\ge\max\left\{\frac{4}{\yneu},\left(\frac{4M}{\gamma}\right)^2\right\}$, the Lyapunov function $\widetilde{V}$ defined in \eqref{eq:def-Vtilde} satisfies:
    \[
    \E^{\pits}\left[\left(\widetilde{V}(\Sigma_{t+1})-\frac{\bar\eta}{2}\right)^+\mid \Sigma_t 
        \right]-\left(\widetilde{V}(\Sigma_t)-\frac{\bar\eta}{2}\right)^+\le-\frac{\gamma}{4}\mathbb{I}[D_t^{\mathrm{OL}}\neq [N]]+6\exp(8)S\widetilde{V}_{\max}\exp(-CN),
    \]
    where $\gamma$ is defined in \eqref{eq:def-beta-Q-b-bareta-gamma} and $M,C$ are defined in~\eqref{def:g-C-M}.
\end{lemma}

\proof{Proof.}
The argument follows the proof of Theorem~1 and Lemma~8 in
\citet{hong2024achieving}, with two changes needed to make all constants and remainder terms explicit and uniform in \(N\).

First, the drift constant in \citet{hong2024achieving} depends on $N$ through
$\bar\eta=\eta-\epsilon_N^{\mathrm{fe}}$. Denote the corresponding
$N$-dependent constant by
$\gamma_N^{\mathrm H}=\bar\eta/K_{Vh}$. Since
$\eta=S^{-1/2}y_{\mathrm{neu}}$ and
$\epsilon_N^{\mathrm{fe}}=S^{-1/2}/N$, the condition
$N\ge 4/y_{\mathrm{neu}}$ implies
$\bar\eta\ge 3\eta/4$. Hence
$\gamma_N^{\mathrm H}\ge 3\eta/(4K_{Vh})=\gamma$, where $\gamma$
is independent of $N$.

Second, we make explicit the $O(N^{-1/2})$ remainder appearing in the
final drift calculation of \citet{hong2024achieving}. In our notation, this
remainder is bounded by $M/\sqrt N$, where $M$ is defined in~\eqref{def:g-C-M}. Therefore, under
$N\ge (4M/\gamma)^2$, we have $M/\sqrt N\le \gamma/4$, and consequently
$-\gamma_N^{\mathrm H}/2+M/\sqrt N
\le -\gamma/2+\gamma/4=-\gamma/4$.

Finally, we make the exponentially small remainder terms explicit. In the middle- and low-region parts of the proof of \citet[Lemma~8]{hong2024achieving}, the desired drift bounds hold outside concentration events whose probabilities decay exponentially in $N$. Using the $N$-independent lower bound $\gamma$ established above, the same concentration inequalities give, for both regions, an exponential rate at least
$
C=\frac{\min\{\gamma^2/16,1\}}
        {4\max\{\lambda_W,\lambda_U\}S^2},
$
as defined in~\eqref{def:g-C-M}. The contributions of the corresponding failure events to the conditional drift are therefore bounded, in total, by
$6e^8S\widetilde V_{\max}\exp(-CN)$.

Applying the same three-region truncation argument as in the proof of
\citet[Lemma~8]{hong2024achieving}, with $\gamma$ in place of
$\gamma_N^{\mathrm H}$, $M/\sqrt N$ in place of the
$O(N^{-1/2})$ remainder, and the above explicit exponential bound, gives
the claimed result. \Halmos
\endproof

\begin{lemma}\label{lem:bound-of-Mhat}
    Suppose we have the single-armed model accuracy bound:
    $
        d_S(\Phat,P) \le \delta\le\delta_{\min},
    $ where $\delta_{\min}$ is defined in~\eqref{eq:def-delta-min}.
    Then the $\widehat{M}$ defined in \eqref{def:hat-g-gamma-C-K-M} satisfies $\widehat{M}\le\ubmhat$, where
    \begin{equation}\label{def:ubmhat}
    \begin{aligned}
        \ubmhat\triangleq&\frac{2\ublwhat+2\ubluhat}{4\omega}\left(1+\frac{4}{\min\{y^\star(\tilde{s},0),y^\star(\tilde{s},1)\}}\right)+5\sqrt{\ubluhat}S+3\sqrt{\ublwhat}S\\
         & +\exp(33)\left(\sqrt{\ublwhat}+\sqrt{\ubluhat}\right)\frac{32(\ubluhat)^{\frac{3}{2}}S^{\frac{5}{2}}}{\min\{y^\star(\tilde{s},0),y^\star(\tilde{s},1)\}}+2\sqrt{\ublwhat}.
    \end{aligned}
    \end{equation}
\end{lemma}

\begin{proof}{Proof.}
According to Theorem~\ref{thm:LP_perturbation_main_complete}, when $\delta\le\delta_{\min}$, we have
\begin{align*}
    \frac{3}{2}S^{-1/2}\min\{y^\star(\tilde{s},0),y^\star(\tilde{s},1)\}
    \ge\widehat \eta
    &=S^{-1/2}\min\{\yhat^\star(\tilde{s},0),\yhat^\star(\tilde{s},1)\} 
    \ge \frac{1}{2}S^{-1/2}\min\{y^\star(\tilde{s},0),y^\star(\tilde{s},1)\}.
    \end{align*}
It is clear that $\widehat K_{Vh}\ge 2$, so
    \begin{align*}
        \widehat M=&\frac{\widehat L}{2\widehat K_{Vh}\omega}\left(1
  + \frac{2S^{-\frac{1}{2}}}
         {\widehat \eta}\right)+5\luhhalf S+3\lwhhalf S 
          +(\lwhhalf+\luhhalf)\exp\left(\frac{1}{2\luhat^3S^3}+32\right)\frac{16\luhat^{3/2}S^{2}}{\widehat \eta} +2\lwhhalf \\
         \le &\frac{2\ublwhat+2\ubluhat}{4\omega}\left(1+\frac{4}{\min\{y^\star(\tilde{s},0),y^\star(\tilde{s},1)\}}\right)+5\sqrt{\ubluhat}S+3\sqrt{\ublwhat}S\\
         & +\exp(33)\left(\sqrt{\ublwhat}+\sqrt{\ubluhat}\right)\frac{32(\ubluhat)^{\frac{3}{2}}S^{\frac{5}{2}}}{\min\{y^\star(\tilde{s},0),y^\star(\tilde{s},1)\}}+2\sqrt{\ublwhat}\triangleq\ubmhat 
    \end{align*}
     where $\ublwhat,\ubluhat$ are defined in \eqref{def:ublwhat} and \eqref{def:ubluhat}, which only depend on the parameters of the original system. \Halmos
\end{proof}

The following lemma bounds $\infnorm{\Vhat}$ under a sufficiently small single-arm estimation error.
\begin{lemma}
\label{lem:vhat-upperbound}
Suppose that $d_S(\Phat,P)\le\delta\le\delta_{\min}$
and $N\ge8/\yneu$, where $\delta_{\min}$ is defined
in~\eqref{eq:def-delta-min}.
Then the Lyapunov function $\widehat V$ defined
in~\eqref{equ:defi-Vhat} satisfies
\[
    \infnorm{\widehat V}
    \le \widehat V_{\max}
    \triangleq
    3\sqrt{2\lambda_U}
    +12\sqrt{S(2+\log_2S)\tau}
    +4\ubgmhat\sqrt S\lambda_U.
\]
\end{lemma}

\begin{proof}{Proof.}
    We first notice that:
    \begin{align*}
        \left(\widehat{\widetilde{V}}-\frac{\widehat{\bar\eta}}{2}\right)^+& \le \infnorm{\widehat{\widetilde{V}}}\le \widehat{\widetilde{V}}_{\max}\le 3\lambda_{\widehat{U}}^{1/2}+5\lwhhalf 
         \le3\sqrt{2\lambda_U}+12\sqrt{S(2+\log_2S)\tau},
    \end{align*}
    where the last inequality uses \eqref{def:ublwhat} and ~\eqref{def:ubluhat}.
    Then following the same argument as \eqref{eq:Vmax-def}, 
    \begin{align*}
        \left|\frac{\widehat\gamma}{4(\widehat \beta+r_{\max})}(\muhst-\bmx([N]))\Qhat \widehat b\right|
        \le \widehat\gamma\infinfnorm{\widehat Q}
        \le4\widehat\gamma\sqrt S\lambda_U
        \le4\ubgmhat\sqrt S\lambda_U,
    \end{align*}
    where the last inequality uses Lemma~\ref{lem:bound-gammahat} and Theorem~\ref{thm:LP_perturbation_main_complete}. Combining the above results we get the final bound. \Halmos
\end{proof}

\begin{lemma}\label{lem:phat-p-V}
Suppose that $d_S(\Phat,P) \le \delta$. Then $|(\widehat \bp^{\pits}-\bm{P}^{\pits} )V(\sigma)| \le V_{\max}N\delta$ for all $\sigma=(\bmx,d^{\mathrm{OL}},d^{\pistar})$.
\end{lemma}

\begin{proof}{Proof.}
Note that for all $t,$ we have
\begin{align*}
    \bp^{\pits}V(\sigma)
    =\E^{\pits}\left[V(\Sigma_{t+1})\mid \Sigma_t=\sigma\right]
    &=\E^{\pits}\left[V(\bm{X}_{t+1},D_{t+1}^{\mathrm{OL}},D_{t+1}^{\pistar})\mid \Sigma_t=\sigma\right] \\
    & = \E^{\pits}\left[\E^{\pits}\left[V(\bm{X}_{t+1},D_{t+1}^{\mathrm{OL}},D_{t+1}^{\pistar})\mid \bm{X}_{t+1}, \Sigma_t=\sigma\right]\mid \Sigma_t=\sigma \right]
\end{align*}
Introduce the shorthand
$
    f(\bm{X}_{t+1},\sigma) = \E^{\pits}\left[V(\bm{X}_{t+1},D_{t+1}^{\mathrm{OL}},D_{t+1}^{\pistar})\mid \bm{X}_{t+1}, \Sigma_t=\sigma\right]
$
It is clear that $\left|f(\bm{X}_{t+1},\sigma)\right|\le V_{\max}, \forall t$. Then we have
\begin{align*}
    \bp^{\pits}V(\sigma)& =\E^{\pits}\left[f(\bm{X}_{t+1},\sigma)\mid \Sigma_t=\sigma\right] \\
    & = \E^{\pits}\left[\E^{\pits}[f(\bm{X}_{t+1},\sigma)\mid \Sigma_t=\sigma, \bm{A}_t\sim \pits(\sigma)]\mid \Sigma_t=\sigma\right] \\
    & = \E^{\pits}\left[\sum_{\bm{x}_{t+1}}f(\bm{x}_{t+1},\sigma)\cdot \bp(\bm{x}_{t+1}\mid \bmx, \bm{A}_t)\mid \Sigma_t=\sigma\right] \\
    & =\sum_{\bma_t}\sum_{\bm{x}_{t+1}}f(\bm{x}_{t+1},\sigma)\cdot \bp(\bm{x}_{t+1}\mid \bmx, \bma_t)\cdot\pits(\bma_t\mid \sigma) .
\end{align*}
Following the same argument, we have:
\begin{align*}
    \widehat \bp^{\pits}V(\sigma)
    =\widehat \E^{\pits}\left[V(\Sigma_{t+1})\mid \Sigma_t=\sigma\right]
    &=\widehat \E^{\pits}\left[V(\bm{X}_{t+1},D_{t+1}^{\mathrm{OL}},D_{t+1}^{\pistar})\mid \Sigma_t=\sigma\right] \\
    & = \widehat \E^{\pits}\left[\widehat \E^{\pits}\left[V(\bm{X}_{t+1},D_{t+1}^{\mathrm{OL}},D_{t+1}^{\pistar})\mid \bm{X}_{t+1}, \Sigma_t=\sigma\right]\mid \Sigma_t=\sigma \right] \\
    & = \widehat \E^{\pits}\left[\E^{\pits}\left[V(\bm{X}_{t+1},D_{t+1}^{\mathrm{OL}},D_{t+1}^{\pistar})\mid \bm{X}_{t+1}, \Sigma_t=\sigma\right]\mid \Sigma_t=\sigma \right],
\end{align*}
where the last equality is owing to the following key observation: given $\bm{X}_{t+1}$, the choice of the two sets $D_{t+1}^{\mathrm{OL}}, D_{t+1}^{\pistar}$ has no relationship with the transition kernel $P$ and $\widehat P$. Hence,
$
    \bphat^{\pits}V(\sigma) = \sum_{\bma_t}\sum_{\bm{x}_{t+1}}f(\bm{x}_{t+1},\sigma)\cdot \bphat(\bm{x}_{t+1}\mid \bmx, \bma_t)\cdot\pits(\bma_t\mid \sigma) 
$
Then
\begin{align*}
    |(\bphat^{\pits}-\bp^{\pits})V(\sigma)| & =\left|\sum_{\bma_t}\sum_{\bm{x}_{t+1}}f(\bm{x}_{t+1},\sigma)\left(\bphat(\bm{x}_{t+1}\mid \bmx,\bma_t)-\bp(\bm{x}_{t+1}\mid \bmx,\bma_t)\right)\pits(\bma_t\mid \sigma)\right|\\
     & \le V_{\max}\sum_{\bma_t}\pits(\bma_t\mid \sigma)\sum_{\bm{x}_{t+1}}\left|\bphat(\bm{x}_{t+1}\mid \bmx,\bma_t)-\bp(\bm{x}_{t+1}\mid \bmx,\bma_t)\right|\\
     & \le V_{\max}\sum_{\bma_t}\pits(\bma_t\mid \sigma)\onenorm{\bphat(\cdot\mid \bmx,\bma_t)-\bp(\cdot\mid \bmx,\bma_t)} \\
     & \le V_{\max}N\delta \sum_{\bma_t}\pits(\bma_t\mid \sigma) = V_{\max}N\delta,
\end{align*}
where in the last inequality we use Lemma~\ref{lem:samplesize-singlearm_accu_rb}. \Halmos
\end{proof}

Applying the same argument to the empirical system, we obtain the following lemma:
\begin{lemma}\label{lem:p-phat-vhat}
    Suppose we have the single-armed model accuracy bound:
    $
        d_S(\Phat,P) \le \delta\le\delta_{\min},
    $
    where $\delta_{\min}$ is defined in~\eqref{eq:def-delta-min}. Then for all $\sigma=(\bmx,d^{\mathrm{OL}},d^{\pihstar})$, we have
    $
        |(\bp^{\pihts}-\widehat \bp^{\pihts} )\Vhat(\sigma)| \le \Vhat_{\max}N\delta.
    $
\end{lemma}

The following lemma shows that if all arms follow the \emph{Optimal Local Control} subroutine, the expected scaled state-count vector $\E^{\pits}[\bm{X}_t([N])]$ has a linear dynamics.
\begin{lemma}[{\citealt[Lemma 1]{hong2024achieving}}]\label{lem:lemma1-hong}
Under Assumption~\ref{ass:RB}, suppose \eqref{eq:integer-feasibility} holds for the $N$-armed system at time step $t$. 
We have
\begin{equation}\label{eq:8}
\E^{\pits}\left[\bm{X}_{t+1}\bigl([N]\bigr)-\mu^{\star}\mid
\bm{X}_t,\text{all arms follow Optimal Local Control} \right]
= \bigl(\bm{X}_t\bigl([N]\bigr)-\mu^{\star} \bigr)\Phi .
\end{equation}
where $\Phi$ is the $S\times S$ matrix defined in \eqref{def:phi}.
\end{lemma}

\end{document}